\documentclass[twoside,11pt]{article}

\usepackage[hidelinks]{hyperref}
\usepackage{amsmath}
\usepackage{amssymb}
\usepackage{amsthm}
\usepackage{indentfirst}
\usepackage{paralist}
\usepackage{geometry}
\usepackage{setspace,lipsum}
\usepackage[title]{appendix}
\usepackage{color}
\usepackage{multimedia}
\usepackage{multicol}
\usepackage{bm}
\usepackage{graphics}
\usepackage{graphicx}
\usepackage{booktabs}
\usepackage{romannum}
\usepackage{grffile}
\usepackage{float}
\usepackage{url}
\usepackage{mathabx}
\usepackage{supertabular}
\usepackage{rotating}
\usepackage{pdflscape}
\usepackage{verbatim}
\usepackage{listings}
\usepackage{xcolor}
\usepackage{multicol}
\usepackage[shortlabels]{enumitem} 
\usepackage{caption,subcaption}
\usepackage{epstopdf}
\usepackage{algorithm}
\usepackage{algorithmic}
\usepackage[round]{natbib}
\usepackage{authblk}
\usepackage{etoolbox}
\graphicspath{{figure/}}
\newcommand{\tabincell}[2]{\begin{tabular}{@{}#1@{}}#2\end{tabular}}

\newcommand{\dd}{\mathrm{d}}

\newcommand{\E}{\mathbb{E}}
\newcommand{\p}{\mathbb{P}}

\newcommand{\f}{\mathcal{F}}

\DeclareMathOperator*{\argmin}{arg\,min}
\newtheorem{assumption}{Assumption}
\newtheorem{eg}{Example}
\newtheorem{proposition}{Proposition}
\newtheorem{theorem}{Theorem}
\newtheorem{lemma}{Lemma}

\begin{document}
\pagenumbering{arabic}
\providecommand{\keywords}[1]
{
\small	
\textbf{\textit{Keywords:}} #1
}
\providecommand{\jel}[1]
{
\small	
\textbf{\textit{JEL Classification:}} #1
}

\title{Policy Evaluation and Temporal-Difference Learning in Continuous Time and Space: A Martingale Approach}

\author{Yanwei Jia\thanks{Department of Industrial Engineering and Operations Research, Columbia University, New York, NY 10027, USA. Email: yj2650@columbia.edu.} ~ ~ ~ Xun Yu Zhou\thanks{Department of Industrial Engineering and Operations Research \& Data Science Institute, Columbia University, New York, NY 10027, USA. Email: xz2574@columbia.edu.}}

\maketitle
\begin{abstract}
\singlespacing
We propose a unified framework to study policy evaluation (PE) and the associated temporal difference (TD) methods for reinforcement learning in continuous time and space. We show that PE is equivalent to maintaining the martingale condition of a process. From this perspective, we find that the mean--square TD error approximates the quadratic variation of the martingale and thus is not a suitable objective for PE. We present two methods to use the martingale characterization for designing PE algorithms. The first one minimizes a ``martingale loss function", whose solution is proved to be the best approximation of the true value function in the mean--square sense. This method interprets the classical gradient Monte-Carlo algorithm. The second method is based on a system of equations called the ``martingale orthogonality conditions" with test functions. Solving these equations in different ways recovers various classical TD algorithms, such as TD($\lambda$), LSTD, and GTD. Different choices of test functions determine in what sense the resulting solutions approximate the true value function. Moreover, we prove that any convergent time-discretized algorithm converges to its continuous-time counterpart as the mesh size goes to zero, and we provide the convergence rate. We demonstrate the theoretical results and corresponding algorithms with numerical experiments and applications.
\end{abstract}

\keywords{Continuous time and space, reinforcement learning, policy evaluation, temporal difference, martingale.}

\section{Introduction}
Policy evaluation (PE) is a crucial step in most critic-related reinforcement learning (RL) algorithms such as actor-critic algorithms and policy iteration. Its objective is to estimate/predict the value function of a given policy using samples, generally without knowledge about the environment. Existing PE methods have predominantly been limited to discrete-time problems with finite-state Markov decision processes (MDPs). For instance,
Monte Carlo methods use samples to estimate expectations assuming the whole sample trajectories can be repeatedly presented for training; hence they are compatible with offline learning.
The most popular PE methods are based on the temporal difference (TD) error. These are  incremental learning procedures driven by the error between temporally
successive predictions. \cite{sutton1988learning} argues that predictions of the TD methods are both more
accurate and easier to compute than other methods. More importantly, these methods
can learn the value  in real-time before a task terminates; hence it can be used both online and offline \citep{sutton2011reinforcement}.

Despite the fast development and vast applications, there are two major limitations in the current study on RL in general and on PE in particular.
First, most algorithms are developed for MDPs, and little attention has been  paid to problems with continuous time and space. The  few existing studies in the  continuous setting  have been largely restricted to deterministic systems; see for example  \citet{baird1993advantage,doya2000reinforcement,fremaux2013reinforcement,vamvoudakis2010online} and \citet{lee2021policy}, where the state processes follow ordinary differential equations (ODEs) and there are no environmental noises. In particular, \citet{baird1993advantage} and \citet{doya2000reinforcement} are the first to propose some continuous-time versions of the TD methods.
In real life, however, there are abundant examples in which an agent can or indeed needs to interact with a {\it random} environment at ultra-high frequency,
e.g., high-frequency stock trading, autonomous driving, and robots navigation.
Second, while there have been numerous PE algorithms proposed using function approximation such as residual
gradient, gradient
Monte Carlo, and TD($\lambda$), they were usually devised in heuristic and ad hoc manners and their underlying objectives were not always clearly stated.\footnote{See Appendix A, Table \ref{tab:summary algorithms}, for a list of names of existing PE algorithms for MDPs.}
Although many of them are proved to be convergent, the limiting functions are not always interpreted properly especially if the function approximators do not
contain the true solutions. In short, there seems a lack of a {\it unified}  framework to study PE and there is
need for a continuous time and
space perspective, from which many well-known algorithms appear as
discretizations.

The goal of this paper is to bridge these  gaps by providing a unified theoretical
underpinning of PE in continuous time and space with general Markov diffusion processes. Instead of discretizing time, state, and action from the start and then applying the existing discrete techniques and results, we carry out all our theoretical analysis for the continuous setting and discretize time only at the final, algorithmic stage. The advantage of doing so is two-fold. On one hand, as
\citet{doya2000reinforcement} argues, the control performance with this approach will be smoother and the right granularity for discretization will be guided by the function approximation. On the other hand, and indeed more importantly, for analyses in a continuous setting, we have plenty of well-developed tools at our disposal including those of stochastic calculus, differential equations, and stochastic control, which, in turn, will  provide better interpretability/explainability  to the underlying learning technologies. 

Stochastic optimization in continuous time and space, also known as stochastic control, has a long history starting from the 1960s. However, its theory is model-based, namely, the system dynamics and the objective functions are assumed to be given and known. The problem can then be solved by well-established approaches such as Pontryagin's maximum principle and Bellman's dynamic programming. For full accounts of the stochastic control theory see, e.g., \citet{YZbook} and \citet{fleming2006controlled}.
On the other hand, to our best knowledge, the study on model-free RL for diffusion processes started only recently.
\citet{wang2020reinforcement} propose an entropy-regularized, stochastic relaxed control formulation for trading off exploration and exploitation in continuous time and space, and derive the continuous version of the Boltzmann distribution  (Gibbs measure) as the optimal exploratory policy. When the problem is linear--quadratic (LQ), namely the dynamic is linear and the payoff is quadratic in state and action, the optimal strategy specializes to Gaussian exploration. \citet{wang2020continuous} apply this general theory to a mean--variance financial portfolio selection problem, which is inherently of an LQ structure, and design an algorithm for extensive simulation and empirical experiments. \citet{DDJ2020} further consider the equilibrium mean--variance strategies addressing the time-inconsistent issue of the problem. \citet{GuoXZ2020}
extend the formulation and results of \citet{wang2020reinforcement}  to mean-field games. \citet{GXZ2020} use the idea of \citet{wang2020reinforcement} to a non-learning problem -- simulated annealing for nonconvex optimization formulated as controlling the temperature of a Langevin diffusion.

For PE, there are generally two aspects one should address. First and more fundamentally,
one specifies a mathematical objective against which a learning task is evaluated. Usually, such an objective is described by either an optimization problem (to minimize a loss/error function) or a system of equations. Second and on the implementation front, one designs an algorithm to achieve the objective. Many papers have contributed to the second aspect, namely, to develop more efficient numerical solvers to accelerate convergence, reduce variance, or  save computational cost; see, e.g., \citet{xu2002efficient,liu2016proximal,du2017stochastic}. In contrast to that line of research, the present paper focuses on the first aspect aiming at building a unified theoretical framework for PE. We propose and analyze several common objectives in the continuous setting, and demonstrate that they generate continuous counterparts of
some of the best-known PE algorithms for MDPs. This not only leads to PE algorithms for the continuous problems but also provides additional foundations for the discrete ones. As our algorithms designed for our continuous setting are discretized in time for implementation, their convergence with a fixed discretization mesh size
has been already established by existing results. Moreover, we show that, as the discretization gets finer, the limiting point of a convergent discrete-time algorithm also converges to the corresponding solution to the continuous problem, and we further provide the convergence rate.

The entire theoretical analysis of the current paper is premised upon the fact that
the value function along the state process combined with the accumulated running payoff is a martingale. 
This martingality naturally gives rise to a target for offline learning: the value of the martingale at any given time is the least square estimate of that at the terminal time. On the other hand, the martingality leads to orthogonality conditions that in turn generate algorithms corresponding to many existing well-known TD algorithms for MDPs.

A similar martingale condition can also be derived for discrete-time MDPs, which is equivalent to the so-called Bellman equation or the Bellman consistency. In Appendix C we provide such a derivation.
However, to our best knowledge, in the existing RL literature such a condition has not been explicitly presented  -- even if it is rather straightforward to deduce -- nor has it been employed to study PE. Instead, the Bellman equation has been the predominant tool to devise PE algorithms. We demonstrate that the change of perspective from the Bellman equation to the martingality is crucial
in our analysis.

Specifically, our main contributions can be summarized as follows:
\begin{enumerate}[(i)]
\item We show that the continuous analogue of the na\"ive residual gradient method, which minimizes the mean-square TD error \citep{barnard1993temporal, baird1995residual,doya2000reinforcement,wang2020continuous}, converges to the minimizer of the {\it quadratic variation} of the aforementioned martingale. It is, therefore,  {\it inconsistent} with the learning objective. This in turn provides a theoretical explanation why the method is not a desired approach for PE when the environment is stochastic.
\item We propose a martingale loss function based on the total mean-square error between the said martingale process and its terminal value. We prove that minimizing such a loss function is equivalent to minimizing the mean-square error between the approximate value function and the true one. This loss function is implementable on samples, and justifies the  Monte Carlo PE with function approximation \citep{sutton2011reinforcement} in the classical MDP and RL literature.
\item 
We provide a unified perspective  to  interpret TD errors and the related algorithms, including TD($\lambda$), least square TD (LSTD), and gradient TD (GTD and its variants), based on the martingale orthogonality conditions. Specifically, by introducing a finite number of suitable test functions to these conditions, the learning problem is transformed into a system of equations called moment conditions. From this vantage point, we realize that TD($\lambda$) is nothing but to directly apply stochastic approximation to solve such equations, LSTD is  to solve them explicitly when they form a linear system, and GTD methods are to solve various quadratic forms of the moment conditions. In addition, different choices of the test functions determine in what sense the true value function is approximated. For example, TD($\lambda$) essentially correspond to different test functions for different values of $\lambda$, and hence may converge to {\it different} limits.
\end{enumerate}

For reader's easy reference, we present Table 1 in Appendix A summarizing popular PE methods and algorithms, and the interpretations
we will have discovered in this paper in terms of the  objectives  and the convergent limits of the algorithms.

As the conditional expectation in the expression of the value function is connected to both a partial differential equation (PDE) through the Feynman--Kac formula \citep{karatzas2014brownian} and to a backward stochastic differential equation (BSDE) through a martingale representation theorem \citep{el1997backward}, the results of the current paper have natural implications on applying  machine learning methods to numerically solve (high-dimensional) PDEs in search of breaking the  ``curse of dimensionality". The latter has been a hotly pursued topic lately; see for example \citet{raissi2018deep} and \citet{hure2019some}. \citet{han2018solving} propose a deep learning approach to solving PDEs by solving the associated BSDEs via simulation. All these works need to assume that the coefficients of the PDEs  are known. The results of our paper shed light on solving PDEs
by PE methodologies in a data-driven way, in view of the intimate connection among PDEs, BSDEs and PE for Markov diffusion process.\footnote{It is interesting to note that there seems to be less research on solving recursive Bellman-like equations using MDPs, even though the same curse of dimensionality exists for discrete-time
equations.}

The rest of the paper proceeds as follows. In Section \ref{sec:setup}, we formulate the PE problem in continuous time and space and present the martingale characterization of the value function. In Section \ref{sec:td error}, we extend the classical mean-square TD error to the continuous setting and show why it is not a proper objective when the environment is stochastic through simple simulated counter-examples and theoretical analysis. In Section \ref{sec:martingale perspective}, we propose several objectives for PE from  the martingale perspective, based on which we recover and interpret some well-studied PE algorithms. We also present numerical experiments for demonstration. Section \ref{sec:applications} is devoted to some extensions of
our problem formulation along with applications to option-like payoffs and linear-quadratic problems. In Section \ref{algorithmic} we discuss the choice of test functions and the way to do function approximation from the algorithmic perspective.  We conclude in Section \ref{sec:conclusion}. Appendix contains some supplementary materials and all the proofs.

\section{Problem Formulation and Preliminaries}
\label{sec:setup}

Throughout this paper, by convention all vectors are {\it column} vectors unless otherwise specified, and $\mathbb{R}^k$ is the space of all $k$-dimensional  vectors (hence $k\times 1$ matrices). Let $A$ and $B$ be two matrices of the same size.
We denote by $A \circ B$ the inner product between $A$ and $B$, by $|A|$ the Euclidean/Frobenius norm of $A$, and write $A^2: = AA^\top$, where $A^\top$ is  $A$'s transpose. 

A general continuous-time RL problem can be formulated under the stochastic control framework with controlled It\^o's stochastic differential equations (SDEs), analogous to MDPs in discrete time.  However, since this paper concerns only a part (though a crucial part) of the RL problem, namely policy evaluation (PE)  under a {\it fixed} control policy,
we will formulate the problem {\it without}  the control variable, which is
the continuous-time counterpart of the Markov reward process (MRP) in discrete time.\footnote{PE sometimes is also referred to as the \textit{prediction problem}. A general stochastic control formulation of RL can be founded in \cite{wang2020reinforcement}, which will also be reviewed in Appendix B.} 

Let $d,m$ be given positive integers, $T>0$,  and $b: [0,T]\times \mathbb{R}^d \mapsto \mathbb{R}^d$ and $\sigma:
[0,T]\times \mathbb{R}^d\mapsto \mathbb{R}^{d\times m}$ be  given functions. 
The {\it state} (or {\it feature}) dynamic follows a Markov diffusion process  governed by an SDE:
\begin{equation}
\label{eq:model classical}
\dd X_s = b(s,X_s)\dd s + \sigma(s,X_s) \dd W_s,
\end{equation}
such that for any given initial time--state pair $(t,x)\in [0,T]\times \mathbb{R}^d$, the SDE (\ref{eq:model classical}) with $X_t=x$ admits a solution $X=\{X_s,t\leq s\leq T\}$ on a certain filtered probability space $\left( \Omega ,\mathcal{F},\mathbb{P};\{\mathcal{F}_s\}_{s\geq t}\right) $ along with a standard $\{\mathcal{F}_s\}_{s\geq t}$-adapted $m$-dimensional  Brownian motion $W=\{W_{s},$ $s\geq t\}$. Note here we are concerned with the {\it weak} solution which includes the filtered probability space and the Brownian motion as part of the solution. See \cite{karatzas2014brownian} for various notions of solutions to an SDE.

Assuming the weak solution of (\ref{eq:model classical}) is unique (i.e. all the solutions
have identical probability distribution, even if possibly with different sample paths), we define the \textit{value function}
\begin{equation}
\label{J}
J(t,x) =\E\left[\int_t^T r(s,X_s)\dd s + h(X_T)\Big|X_t = x\right],
\end{equation}
where $r$ is an (instantaneous) reward (cost) function (i.e. rate of reward/cost conditioned on time and state)
and $h$ the lump-sum reward (cost) function applied at the end of the planning period, $T$.

Unlike most RL problems that are formulated in an infinite planning horizon (known as {\it continuing tasks}), the current paper mainly focuses on a finite horizon setting (known as {\it episodic tasks}).\footnote{We will briefly discuss the infinite horizon case with exponentially discounted payoff  in Subsection \ref{sec:discount}.} Finite horizons reflect limited lifespans of real-life tasks, 
e.g., a trader sells a financial contract with a maturity date, a robot finishes a task before a deadline, and a video gamer strives to pass a checkpoint given a time limit.

The PE task is, for a fixed given policy (which is suppressed in the formulation above due to the reason we stated earlier), to devise a numerical procedure to find $J(t,x)$ as a {\it function} of $(t,x)$ using multiple sample trajectories of the process $\{ s,X_s,r( s,X_s)\}_{t\leq s\leq T}$,
where $\{X_s,\;t\leq s\leq T\}$ is the solution to
(\ref{eq:model classical}), {\it without} the knowledge of the model parameters (the functional forms of $b,\sigma,r,h$).
Hence we cover the settings of on-policy (i.e., the samples are generated under a {\it target} policy)\footnote{\cite{sutton2008convergent} uses ``behavioral policy'' to describe the policy to generate observations and ``target policy'' to describe the policy we want to evaluate. Off-policy means  training on data from a behavioral policy in order to learn the value of a target policy, and on-policy means that the behavioral policy coincides with the target policy in learning.}, episodic (i.e., the same learning task is repeated for many episodes/multiple trajectories), offline (i.e., the approximated function is updated after a full episode/trajectory has been run) and  online (i.e., the approximated function is updated in real time as we go). We emphasize that for a finite-horizon problem, it is generally  too ambitious to expect an effective algorithm that learns from a {\it single} trajectory with no resets, due to the limited sample size. Learning with a single trajectory is usually done in an infinite horizon setting.

We make the following standard regularity assumptions on the coefficients of \eqref{eq:model classical} and the reward function \eqref{J} to ensure  the theoretical  well-posedness of the problem:
\begin{assumption}
\label{ass:sde regularity}
The following conditions hold true:
\begin{enumerate}[(i)]
\item $b,\sigma,r,h$ are all continuous functions in their respective arguments;
\item $b,\sigma$ are uniformly Lipschitz in $x$, i.e., for $\varphi = b,\ \sigma$, there exists a constant $C>0$ such that
\[ |\varphi(t,x) - \varphi(t,x')| \leq C|x-x'|,\;\;\forall t\in [0,T],\ x,x'\in \mathbb{R}^d; \]
\item $b,\sigma$ have linear growth in $x$, i.e., for $\varphi = b,\ \sigma$, there exists a constant $C>0$ such that
\[|\varphi(t,x)| \leq C(1+|x|) ,\;\;\forall (t,x)\in [0,T] \times \mathbb{R}^d;\]
\item $r$ and $h$ both have polynomial growth  in  $x$,  i.e., 
there exist constants $C>0$ and $\mu\geq 1$ such that
\[
|r(t,x)| \leq C(1+|x|^{\mu}) ,\;\;|h(x)| \leq C(1+|x|^{\mu}),\; \forall (t,x)\in [0,T] \times \mathbb{R}^d .\]
\end{enumerate}
\end{assumption}

Under Assumption \ref{ass:sde regularity}(i)-(iii), the SDE \eqref{eq:model classical} admits a unique strong solution (and hence a unique weak solution) whose moments of any given order are uniformly bounded; see, e.g.,  \cite{karatzas2014brownian}. The unique existence of a weak solution alone requires much weaker assumptions; see e.g. \cite{SV79}, but we will not pursue along that line. On the other hand, Assumption \ref{ass:sde regularity}(iv) is to ensure that
$J(t,x)$ is finite for any $(t,x)$.

We now recall some existing results on Markov diffusion processes underpinning  the theoretical analysis in this paper.
First, $J$ can be characterized by a PDE based on the celebrated Feynman--Kac formula (\citealp{karatzas2014brownian}):\footnote{This PDE is a spacial case of the (nonlinear) Hamilton-Jacobi-Bellman (HJB) equation in  continuous-time stochastic control  when the control variable is fixed.}
\begin{equation}
\label{eq:pde characterization}
\left\{\begin{array}{l}
\mathcal{L}J (t,x) + r\big(t,x\big) = 0,\; (t,x)\in [0,T)\times \mathbb{R}^d,\\
J(T,x) = h(x),
\end{array}\right.
\end{equation}
where
\[ \mathcal{L}J (t,x): = \frac{\partial J}{\partial t}(t,x) + b\big( t,x\big) \circ \frac{\partial J}{\partial x}(t,x) + \frac{1}{2}\sigma^2\big( t,x\big) \circ \frac{\partial^2 J}{\partial x^2}(t,x)\]
is known as the {\it infinitesimal generator} associated with the diffusion process \eqref{eq:model classical}. Here, $\frac{\partial J}{\partial x} \in \mathbb{R}^d$ is the gradient, and $\frac{\partial^2 J}{\partial x^2}\in \mathbb{R}^{d\times d}$ is the Hessian.

The above PDE would be fully specified had the model been completely known.\footnote{Some of this PDE's theoretical properties, such as  existence, uniqueness, and regularity, have been well studied in terms of viscosity solution; see, e.g., \cite{crandall1992user,beck2021nonlinear}.} If the state space has a dimension up to 4 (i.e. $d\leq 4$), the equations can be efficiently solved numerically by methods such as Monte-Carlo and finite element algorithms. Unfortunately, in many practical applications the model parameters are not known, nor is the dimension small. Here,
to avoid unnecessary technicality, we assume
\begin{assumption}
\label{ass:fk pde}
The PDE \eqref{eq:pde characterization} admits a classical solution $J\in C^{1,2}([0,T) \times \mathbb{R}^d)$ satisfying the polynomial growth condition, i.e.,
there exist constants $C>0$ and $\mu\geq 1$ such that
\[
|J(t,x)| \leq C(1+|x|^{\mu}), \  \forall (t,x)\in [0,T] \times \mathbb{R}^d .\]
\end{assumption}

Second, the PDE (\ref{eq:pde characterization}) is related to the following forward--backward stochastic differential equation (FBSDE):
\begin{equation}
\label{eq:fbsde}\left\{\begin{array}{l}
\dd X_s = b(s,X_s)\dd s + \sigma(s,X_s) \dd W_s,\
s\in[t,T]; \;\;X_{t} = x,\\

\dd Y_s = -r\big(s,X_s\big)\dd s + Z_s\dd W_s,s\in[t,T]; \ Y_T = h(X_T).
\end{array}\right.
\end{equation}
Its solution, $\{(X_s, Y_{s}, Z_{s}),\;t\leq s\leq T\}$, has the following representations in terms of $J$:
\begin{equation}
\label{eq:link}
Y_s=J(s,X_s),\;\;Z_{s} =  \frac{\partial J}{\partial x}(s,X_s)^\top\sigma\big(s,X_s\big),\;\;s\in[t,T].
\end{equation}
The above relationship can be easily seen by applying It\^o's formula to $J$; for details see \citet{el1997backward}.

For any fixed $(t,x)\in [0,T]\times \mathbb{R}^d$ and $\{X_s,\;t\leq s\leq T\}$ solving the first equation of (\ref{eq:fbsde}), define
\begin{equation}
\label{eq:m} M_s: = J(s,X_{s}) + \int_{t}^s r\big(s',X_{s'}\big)\dd s'
\equiv Y_s+\int_{t}^s r\big(s',X_{s'}\big)\dd s',\;\;s\in[t,T].
\end{equation}


The following result is the theoretical foundation of this paper, which characterizes the value function $J$ by the martingality of $M$.

\begin{proposition}
\label{proposition:martingale bsde}
Suppose Assumptions \ref{ass:sde regularity} and \ref{ass:fk pde} hold. For any fixed $(t,x)\in [0,T]\times \mathbb{R}^d$ and $\{X_s,\;t\leq s\leq T\}$ solving the first equation of (\ref{eq:fbsde}), the process $M = \{M_s,t\leq s\leq T \}$ is a square-integrable martingale. Conversely, if there is a continuous function $\tilde{J}$ such that for all $(t,x)\in [0,T]\times \mathbb{R}^d$, $\tilde M = \{\tilde M_s,t\leq s\leq T \}$ is a square-integrable martingale, where $\tilde{M}_s:= \tilde{J}(s,X_{s}) + \int_{t}^s r\big(s',X_{s'}\big)\dd s'$, and $\tilde{J}(T,x) = h(x)$, then $\tilde{J}\equiv J$ on $[0,T]\times \mathbb{R}^d$.
\end{proposition}

This proposition inspires a martingale approach to PE in continuous-time RL, which will be developed in this paper. Essentially, the approach exploits  the equivalence between PE (Feynman--Kac formula) and the martingality.

Finally, for a square-integrable semi-martingale $M = \{M_t,0\leq t\leq T \}$, its quadratic variation process, denoted by $\langle M \rangle = \{\langle M \rangle_t,0\leq t\leq T \}$, is defined to be \citep{karatzas2014brownian}
\[\langle M \rangle_t = \lim_{||\Delta|| \to 0}  \sum_{i=0}^{K-1} (M_{\tau_i} - M_{\tau_{i-1}})^2<\infty, \]
where $\Delta : 0 = \tau_0 < \cdots < \tau_K = t$ is an arbitrary partition of the interval $[0,t]$, and $||\Delta|| = \max_{1\leq i \leq K}\{ \tau_{i} - \tau_{i-1}\}$ is the largest mesh size. 
For $M$ defined by (\ref{eq:m}), we have
\begin{equation}
\label{eq:qv}
\langle M \rangle_t=\langle Y \rangle_t=\int_0^t|Z_s|^2ds,\;\;t\in[0,T].
\end{equation}


Introduce
\[ L^2_{\f}([0,T]) = \bigg\{ \kappa = \{\kappa_t, 0\leq t\leq T \} \text{ is real-valued and $\f_t$-progressively measurable} : \E\int_0^T \kappa_t^2 \dd t < \infty \bigg\}.\]
It is a Hilbert space with $L^2$-norm $||\kappa||_{L^2} = (\E\int_0^T \kappa_t^2 \dd t)^{\frac{1}{2}}$.
More generally, for any semi-martingale $Y=\{Y_{s},$ $s\geq 0\}$, we denote
\[
L^2_{\f}([0,T];Y) = \bigg\{ \kappa = \{\kappa_t, 0\leq t\leq T \}:  \mbox{ $\kappa$ is $\f_t$-progressively measurable and } \E\int_0^T |\kappa_t|^2 \dd \langle Y \rangle_t < \infty \bigg\}. \]

\section{Temporal Difference Error in Continuous Time}
\label{sec:td error}
In this section, we first review \cite{doya2000reinforcement}'s TD error approach  for deterministic dynamics and then explain why we can {\it not} extend that approach
to the stochastic setting.

\subsection{Doya's TD algorithm for deterministic dynamics}
Many RL algorithms for discrete-time MDPs use TD error to evaluate policies. \cite{doya2000reinforcement} extends the TD approach to RL with continuous time and space, albeit only for {\it deterministic} dynamics. For readers' convenience and for highlighting the key differences between deterministic and stochastic settings, we briefly review \cite{doya2000reinforcement}'s approach here.

In our setting with $\sigma=0$ (and hence all the expectations are dropped), Doya's approach starts with the obvious identity
\begin{equation}\label{learning_value}
J(t,X_t)=\int_{t}^{t' }r(s,X_s)\dd s+J(t', X_{t'}),\;\;t'\in(t,T].
\end{equation}
Rearranging this equation and dividing both sides by $t'-t$, we obtain
\begin{equation}\label{delta}
\frac{J(t', X_{t'})-J(t,X_t)}{t'-t}+\frac{1}{t'-t}\int_t^{t'}r(s,X_s)\dd s=0.
\end{equation}
Letting $t'\rightarrow t$ on the left hand side motivates the definition of the {\it TD error rate}:\footnote{\cite{doya2000reinforcement} still refers this term as ``TD error'', while we add ``rate'' in its definition to reflects that it is indeed the {\it instantaneous} temporal difference  at a given time $t$. However, we will use both terms interchangeably in this paper.}
\begin{equation}
\label{eq:td error}
\delta_t:=\dot{J}_t+r_t,
\end{equation}
where $\dot{J}_t:=\frac{d}{dt}J(t,X_t)$ is the total derivative of $J$ along $(t,X_t)$, and $r_t:=r(t,X_t)$.

The {\it function approximation} approach widely employed for PE first approximates $J$ by a parametric family of functions $J^{\theta}$
(upon using linear spans of basis functions or neural networks, or taking advantage of any known or plausible structure of the underlying problem), with $\theta \in \Theta \subseteq \mathbb{R}^L$. Henceforth, we always use $\theta$-superscripted functions to denote
those corresponding to the parameterized functions. For instance, $\delta_t^\theta:=\dot{J}_t^\theta+r_t$.

\citet{doya2000reinforcement} determines $\theta$ by minimizing the \textit{mean--square TD error} (MSTDE)
\begin{equation}
\label{eq:mean squared td error}
\operatorname{MSTDE}(\theta):=\frac{1}{2}\int_0^T |\delta_t^{\theta}|^2\dd t\equiv \frac{1}{2}\int_0^T \big|\dot{J}^{\theta}_t+r_t\big|^2\dd t,
\end{equation}
in view of the fact that this error {\it ought} to be zero theoretically.

To approximate and compute $\operatorname{MSTDE}(\theta)$, we
discretize $[0,T]$ into small intervals $[t_i,t_{i+1}]$, $i=0,1,\cdots,K-1$, with
an equal length $\Delta t$, where $t_0=0$ and $t_{K}=T$. This leads to
\begin{equation}
\label{eq:MSTDE}
\operatorname{MSTDE}(\theta)\approx \frac{1}{2}\sum_{i=0}^{K-1} \bigg( \frac{J^{\theta}(t_{i+1}, X_{t_{i+1}}) - J^{\theta}(t_{i}, X_{t_{i}})}{t_{i+1} - t_{i}} + r_{t_{i}} \bigg)^2\Delta t =:\operatorname{MSTDE}_{\Delta t}(\theta).
\end{equation}
A gradient descent algorithm is then applied to obtain the minimizer $\theta^*$ of $\operatorname{MSTDE}_{\Delta t}$ which in turn determines $J(t,x)=J^{\theta^*}(t,x)$. Namely,
\begin{equation}
\label{eq:gradient squared td error}
\theta \leftarrow \theta - \alpha \sum_{i=0}^{K-1} \bigg( \frac{J^{\theta}(t_{i+1}, X_{t_{i+1}}) - J^{\theta}(t_{i}, X_{t_{i}})}{t_{i+1} - t_{i}} + r_{t_{i}} \bigg) \bigg( \frac{\partial J^{\theta}}{\partial \theta}(t_{i+1}, X_{t_{i+1}})  - \frac{\partial J^{\theta}}{\partial \theta}(t_{i}, X_{t_{i}})  \bigg),
\end{equation}
where $\alpha$ is the learning rate (step size). This updating rule is also referred to as the \textit{na\"ive
residual gradient} method \citep{barnard1993temporal, baird1995residual}. 

The above algorithm is stated in  the offline setting; namely, one uses the {\it whole} sample trajectory to update $\theta$. However, TD-learning is often advocated for {\it online} learning: instead of observing the full sample path over $[0,T]$, one updates the estimate of the value function at each discrete time point using all available historical information.
Take the most popular one-step method for example. With the time discretization described above,
this method updates $\theta$ by
\[\theta \leftarrow  \theta - \alpha \bigg( \frac{J^{\theta}(t_{i+1}, X_{t_{i+1}}) - J^{\theta}(t_{i}, X_{t_{i}})}{t_{i+1} - t_{i}} + r_{t_{i}} \bigg) \bigg( \frac{\partial J^{\theta}}{\partial \theta}(t_{i+1}, X_{t_{i+1}})  - \frac{\partial J^{\theta}}{\partial \theta}(t_{i}, X_{t_{i}})  \bigg). \]
Notice that this increment is just one term in that of  \eqref{eq:gradient squared td error}.

The most important feature of these TD-based algorithms that makes them implementable for learning is that one can {\it observe}
the payoffs $r_{t_i}$ and the states $X_{t_{i}}$, and hence can compute $J^{\theta}(t_i,X_{t_i})$, $i=0,1,\cdots,K-1$, through samples,
{\it without} having to know the model parameters.

\subsection{Mean-square TD error for stochastic dynamics}
\label{sec:td toy example}
If we are to extend the MSTDE approach {\it na\"ively} from \cite{doya2000reinforcement}'s deterministic setting to the current stochastic (diffusion) setting, then we first note that the following equation, which is similar to (\ref{learning_value}), holds
\begin{equation}\label{learning_value_st}
J(t,X_t)=\E\left[\int_{t}^{t' }r(s,X_s)\dd s+J(t', X_{t'})\Big|\mathcal{F}_t\right].\;\;t'\in(t,T].
\end{equation}
This equation is called {\it Bellman's consistency}. 
Then
\begin{equation}\label{deltas}
\E\left[\frac{J(t', X_{t'})-J(t,X_t)}{t'-t}+\frac{1}{t'-t}\int_t^{t'}r(s,X_s)\dd s\right]=0.
\end{equation}

We may then be tempted to define a stochastic version of the TD error rate as in (\ref{eq:td error}). Unfortunately, the
path-wise total derivative $\dot{J}_t=\frac{d}{dt}J(t,X_t)$ no longer exists in the current diffusion case; hence, the TD error rate $\delta_t$ is not well defined now. This issue stems from the intrinsic non-differentiability of (non-degenerate) diffusion processes. For instance, it is well-known that with probability one, the sample trajectory of a Brownian motion is nowhere differentiable.

To facilitate our analysis without getting overly technical, we make the following regularity assumption on the value function approximators $J^{\theta}$ we use in this paper:
\begin{assumption}
\label{ass:regularity}
$J^{\theta}(t,x)$ is a sufficiently smooth function of $(t,x,\theta)$ so that all the derivatives required exist in the classical sense. Moreover, for all $\theta \in \Theta$, $J^{\theta}(\cdot,X_{\cdot}),\  \mathcal{L} J^{\theta}(\cdot,X_{\cdot}) + r_{\cdot},\  \big| \frac{\partial J^{\theta}}{\partial x}(\cdot,X_{\cdot})^\top \sigma(\cdot,X_{\cdot}) \big| \in L^2_{\f}([0,T])$, and their $L^2$-norms are continuous functions of $\theta$.
\end{assumption}

Given an approximator $J^{\theta}$, a theoretically well-defined error based on (\ref{learning_value_st}) in continuous time is the so-called \textit{Bellman's error rate}:
\begin{equation}
\label{eq:bellman error pdf}
\lim_{t' \to t+} \frac{1}{t'-t} \E\left[\int_{t}^{t' }r(s,X_s)\dd s+J^{\theta}(t', X_{t'}) - J^{\theta}(t,X_t)\Big|\mathcal{F}_t\right] = \mathcal{L} J^{\theta}(t,X_t) + r\big(t,X_t\big).
\end{equation}
This can be derived by applying It\^o's formula to $J^{\theta}(t,X_t)$.

If there is no randomness in the environment, the conditional expectation in \eqref{eq:bellman error pdf} vanishes and hence Bellman's error coincides with TD error \eqref{eq:td error} in the deterministic case. In a non-degenerate stochastic environment, however, only
Bellman's error $\{  \mathcal{L} J^{\theta}(t,X_t) + r\big(t,X_t\big),0\leq t\leq T\}$ is  well defined on  sample trajectories. 
Note that this error is zero everywhere for the true value function, according to (\ref{eq:pde characterization}). So it seems natural to set a PE objective to minimize  Bellman's error.
Unfortunately, this error accounts for the conditional expectation and thus
represents the {\it average} of temporal differences over infinitely many sample trajectories that are distributed according to the SDE \eqref{eq:model classical}. Therefore, the knowledge about the state dynamics is {\it required} in computing the
conditional expectation or, equivalently, in applying the operator $\mathcal{L}$.
This knowledge is nevertheless unknown to the agent in our RL setting. In other words,
in sharp contrast to the TD error, Bellman's error and its discretization version cannot be computed with only samples in a black-box environment.

On the other hand, even though MSTDE does not exist theoretically in the continuous-time stochastic setting, we can still define and compute its {\it discretization} version, in a way analogous to (\ref{eq:MSTDE}):

\begin{equation}\label{TD-error}
\begin{aligned}
\operatorname{MSTDE}_{\Delta t}(\theta):=	& \frac{1}{2}\E\left[\sum_{i=0}^{K-1} \bigg( \frac{J^{\theta}(t_{i+1}, X_{t_{i+1}}) - J^{\theta}(t_{i}, X_{t_{i}})}{t_{i+1} - t_{i}} + r_{t_{i}} \bigg)^2\Delta t\right].
\end{aligned}
\end{equation}
Indeed, \citet{wang2020continuous} use this version to develop a PE algorithm for the mean--variance problem.
A natural question is whether minimizing  $\operatorname{MSTDE}_{\Delta t}(\theta)$ (or equivalently applying the stochastic version of the residual gradient algorithm \eqref{eq:gradient squared td error}) will lead to the correct solution in the stochastic environment. The answer is unfortunately negative, as illustrated by the following example.
\begin{eg}
\label{eg:toy 1}
{\rm
Let us find a function that represents the conditional expectation $J(t,x) = \E[X_1|X_t = x]$, where $X_t = W_t$ is a Brownian motion. This is probably the simplest example possible. Because the Brownian motion is a martingale, we know the {\it ground truth} solution $J(t,x) = x$. Pretending we do not know this solution, we proceed to minimize $\operatorname{MSTDE}_{\Delta t}(\theta)$ to learn $J$ based on the simulated sample paths of
the state process $X\equiv W=\{W_t,\;0\leq t\leq 1\}$, which is a standard Brownian motion starting from $W_0=0$.

We first use a parameterized family $J^{\theta}(t,x) = [\theta(1-t) + 1] x$ to approximate $J$. This family contains the true function when $\theta=0$. The discretized MSTDE is
\[\operatorname{MSTDE}_{\Delta t}(\theta) = \frac{1}{2}
\E\left[ \sum_{i=0}^{K-1} \left( \frac{\big(\theta (1-t_{i+1}) + 1\big) X_{t_{i+1}} - \big(\theta (1-t_{i}) + 1\big) X_{t_{i}}}{t_{i+1} - t_{i}} \right)^2\Delta t \right].\]
We apply the stochastic gradient decent (SGD) with the updating rule
\[ \theta \leftarrow \theta - \alpha \sum_{i=0}^{K-1} \bigg( \frac{\big(\theta (1-t_{i+1}) + 1\big) X_{t_{i+1}} - \big(\theta (1-t_{i}) + 1\big) X_{t_{i}}}{t_{i+1} - t_{i}} \bigg) \left[(1-t_{i+1})X_{t_{i+1}} - (1 - t_i)X_{t_i}\right] . \]

In our simulation  we use multiple independent episodes for training. We take the time grid size as $\Delta t = 0.01$, 
initialize
the parameter to be $\theta^{(0)} = -1$, and apply the above updating rule with the learning rate $\alpha=0.01$.

Figure \ref{fig:td error example 1} illustrates the convergence of $\theta$ to $\theta^*_{\operatorname{MSTDE}} = -\frac{3}{2}$ which is {\it not} the true solution $\theta_{\text{true}} = 0$. In other words, the value function is not correctly learned by MSTDE. Equivalently, it does not solve the PDE \eqref{eq:pde characterization} or the FBSDE \eqref{eq:fbsde} correctly.}
\end{eg}

\begin{figure}[h]
\centering
\includegraphics[width = 0.5\textwidth]{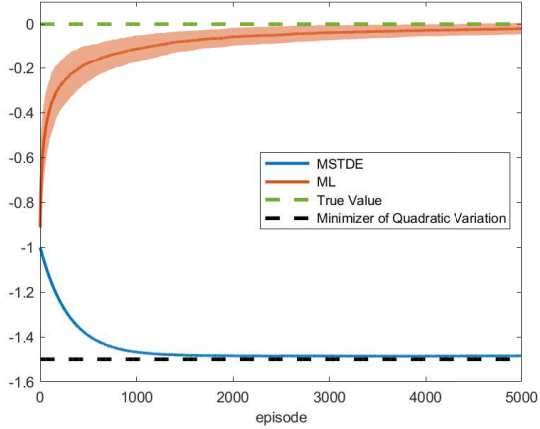}
\caption{\textbf{The paths of parameters over episodes with different objectives for Example \ref{eg:toy 1}.} The true solution is $\theta_{\text{true}} = 0$. Applying SGD to minimize $\operatorname{MSTDE}_{\Delta t}$ leads to $\theta^*_{\operatorname{MSTDE}} = -\frac{3}{2}$. Applying SGD to minimize the martingale loss function generates  the correct solution. We repeat the experiment for 100 times to calculate the standard deviations, which are represented as the shaded areas. The width of each shaded area is twice the corresponding standard deviation.}
\label{fig:td error example 1}
\end{figure}

\subsection{Theoretical characterization of mean-square TD error}
To understand {\it theoretically} why taking the objective of minimizing MSTDE does not work for stochastic problems, recall the processes $(X,Y,Z)$ and the martingale $M$ defined
through (\ref{eq:fbsde})--(\ref{eq:m}) in which we take $t=0$. Then
\[
\begin{aligned}
& \sum_{i=0}^{K-1} \bigg( \frac{J(t_{t+1}, X_{t_{i+1}}) - J(t_{i}, X_{t_{i}})}{t_{i+1} - t_{i}} + r_{t_{i}} \bigg)^2 \Delta t\\
= & \frac{1}{ \Delta t} \sum_{i=0}^{K-1} \bigg( J(t_{i+1}, X_{t_{i+1}}) - J(t_{i}, X_{t_{i}}) + \int_{t_{i}}^{t_{i+1}} r_{t_s} \dd s + O\big( (\Delta t)^2\big) \bigg)^2 \\
\approx & \frac{1}{ \Delta t} \langle M \rangle_T =\frac{1}{( \Delta t)^2} \langle Y \rangle_T
= \frac{1}{ \Delta t} \int_0^T |Z_t|^2 \dd t,
\end{aligned}
\]
which is {\it not} zero, unlike in the deterministic setting. Hence, minimizing the MSTDE is wrong, because it is equivalent to minimizing the expected {\it quadratic variation} of the martingale $M$, which should {\it not} be minimized as the objective for estimating the value function.\footnote{A related notion in financial econometrics is the {\it realized variance} of a time series, which is proved to be an unbiased estimate of the integrated variance or the quadratic variation, see, e.g., \citet{barndorff2002estimating}.}

Take Example \ref{eg:toy 1} again:
\setcounter{eg}{0}
\begin{eg}[Continued]
{\rm
Let us use the previously taken parameterized family $Y^{\theta}_t = J^{\theta}(t,X_t) = [\theta(1-t) + 1] X_t$ to approximate $J$. Then
\[ \dd M^{\theta}_t\equiv \dd Y^{\theta}_t =  [\theta(1-t) + 1] \dd W_t, \]
leading to   $$\langle M^{\theta} \rangle_1 = \int_0^1 \big[1 + \theta(1-t)\big]^2\dd t = \frac{1}{3}\theta^2 + \theta+1,$$
which is minimized at $\theta^* = -\frac{3}{2}$, instead of the desired value $\theta_{\text{true}} = 0$. This theoretical value matches the simulation result reported in Figure \ref{fig:td error example 1}.}
\end{eg}

We now present a slightly more involved example, one that includes a running reward term.
\begin{eg}
\label{eg:toy running reward}
{\rm 	We seek a function representing the conditional expectation $J(t,x) = \E[X_1^2 - \int_t^1 \dd s|X_t = x]$ where $X_t=W_t$ is a Brownian motion. Theoretically, the problem is equivalent to solving the following BSDE:
\[ \dd Y_t = \dd t+  Z_t\dd W_t,\ Y_1 = X_1^2 .\]
The true solution is $Y_t = X_t^2$, $Z_t = 2X_t$, namely, $J(t,x) = x^2$. If we use a parameterized family $Y^{\theta}_t = J^{\theta}(t,X_t) = [\theta_0(1-t) + 1] X_t^2 + \theta_1(1-t) X_t + \theta_2(1-t)$ to approximate $J$, then the desired parameter values are $\theta_{\text{true}} = (0,0,0)$.

%
Let us compute the quadratic variation of $M_t^\theta:=Y_t^\theta-t$. By It\^o's lemma and replacing $X_t$ by $W_t$, we obtain
\[ \dd M^\theta_t = (-\theta_2 - \theta_1 W_t - \theta_0 W_t^2 -1)\dd t + \big[2 \big(\theta_0(1-t) + 1\big)W_t + \theta_1(1-t) \big]\dd W_t. \]
Then its expected quadratic variation is
\[
\begin{aligned}
	\E[\langle M^{\theta} \rangle_1] = & \E\int_0^1 \big[2 \big(\theta_0(1-t) + 1\big)W_t + \theta_1(1-t) \big]^2 \dd t  \\
	= & \int_0^1 \left[ 4 \big(\theta_0(1-t) + 1\big)^2 t + \theta_1^2(1-t)^2 \right] \dd t \\
	= & 4\left(\frac{1}{12}\theta_0^2 + \frac{1}{3}\theta_0 + \frac{1}{2}\right) + \frac{1}{3}\theta_1^2,
\end{aligned} .\]
which attains minimum at $\theta^*_0 = -2, \theta_1^* = 0$.

Here, the parameter $\theta_2$ is not present in the expected quadratic variation,
and hence remains undetermined. However, due to numerical errors in computing the TD error, we can determine $\theta_2$ by minimizing the high-order small term in the quadratic variation, given the minimizer, $(\theta_0^*,\theta_1^*)$,  of the leading term. To do this, recall we have the following expansion:
\[(\dd M_t^{\theta})^2 = \underbrace{\cdots}_{\text{leading-order term}}(\dd W_t)^2 + \underbrace{\cdots}_{\text{high-order small term}}(\dd t)^2 + \underbrace{\cdots}_{\text{mean-zero term}}\dd W_t\dd t.  \]
So, parameters will be determined first through the leading term in the quadratic variation. Parameters that do not show up in the leading term have much smaller but non-negligible  impact on the TD error, which can be determined through the second term in the above. Finally, the mean-zero term can be ignored because it will be averaged out.

Therefore, in the current example, $\theta_2$ will be determined through minimizing the following:
\[
\E \int_0^1 (-\theta_2 - \theta_1^* W_t - \theta_0^* W_t^2 - 1)^2 \dd t   =\int_0^1\left[ (\theta_2 + 1)^2 + 2(\theta_2+1)\theta_0^* t + 3{\theta_0^*}^2 t^2\right] \dd t.
\]
The minimizer is $\theta_2^* = 0$. So, optimal parameters to minimize the MSTDE are $\theta^*_{\operatorname{MSTDE}} =(-2,0,0)$ and hence the resulting learned function  is $J(t,x)=(2t-1)x^2$. However, the true function is $J(t,x)=x^2$.\footnote{Even though the two parameters $\theta_1^*,\theta_2^*$ agree with the correct ones, this seems to be a coincidence and the final learned  value function still deviates from the true one.}
}
\end{eg}

We now verify this analysis by simulation.
The discretized mean-square TD error is
\[
\operatorname{MSTDE}_{\Delta t}(\theta) = \frac{1}{2} \E\left[ \sum_{i=0}^{K-1} \left( \frac{J^{\theta}(t_{i+1}, X_{t_{i+1}}) - J^{\theta}(t_{i}, X_{t_{i}})}{t_{i+1}- t_{i}} - 1 \right)^2 \Delta t \right]  .\]
We initialize the parameter to be $\theta^{(0)} = (-1,-1,-1)$, and use the SGD algorithm. The learning rate is taken as $0.01$. The result, shown in Figure \ref{fig:td error example 2}, is consistent with the above theoretical analysis, which incidentally justifies our scheme of
determining some of the parameters through the high-order term.


\begin{figure}[!htbp]
\centering
\begin{subfigure}{0.32\textwidth}
\centering
\includegraphics[width = 1\textwidth]{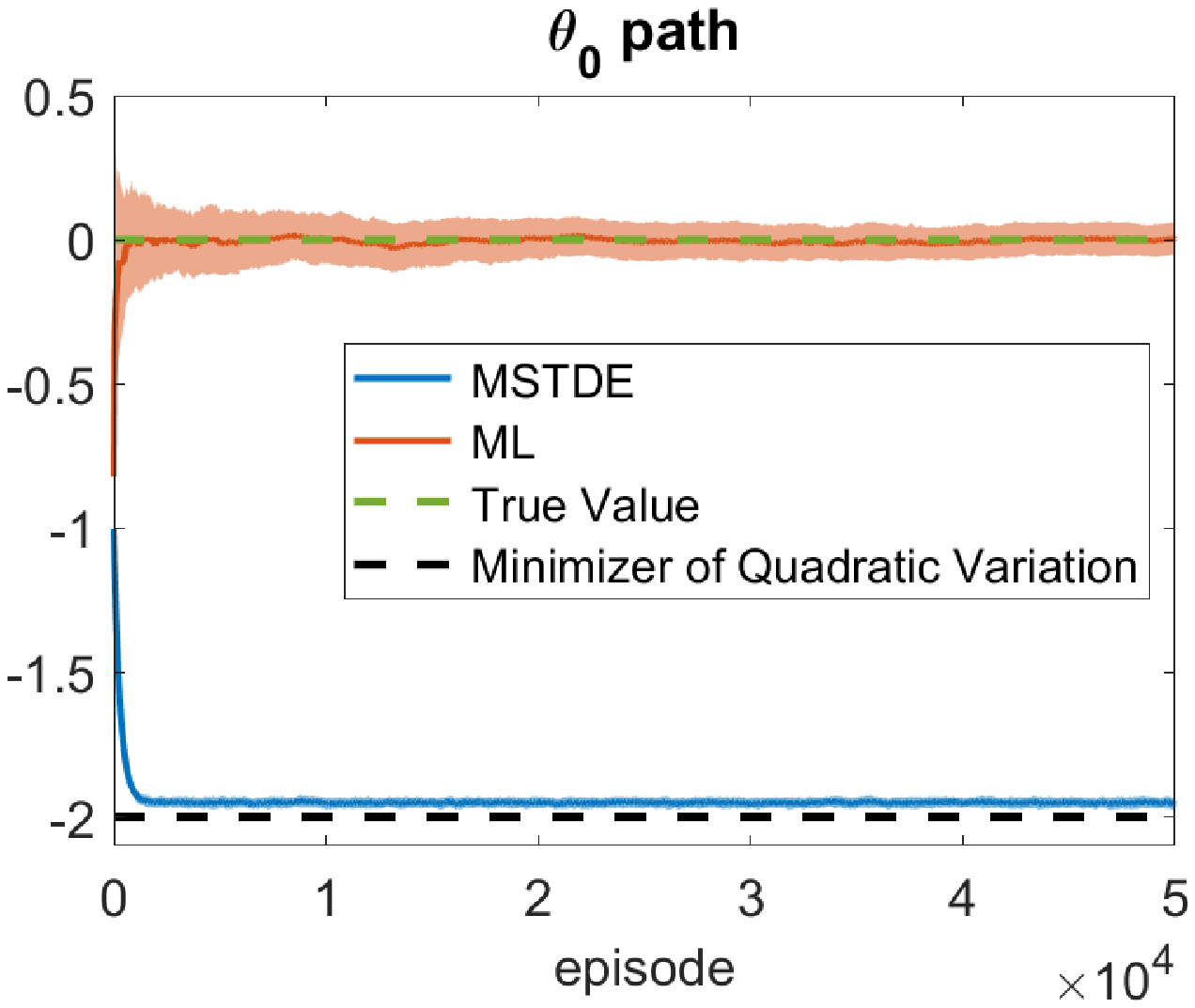}
\end{subfigure}
\begin{subfigure}{0.32\textwidth}
\centering
\includegraphics[width = 1\textwidth]{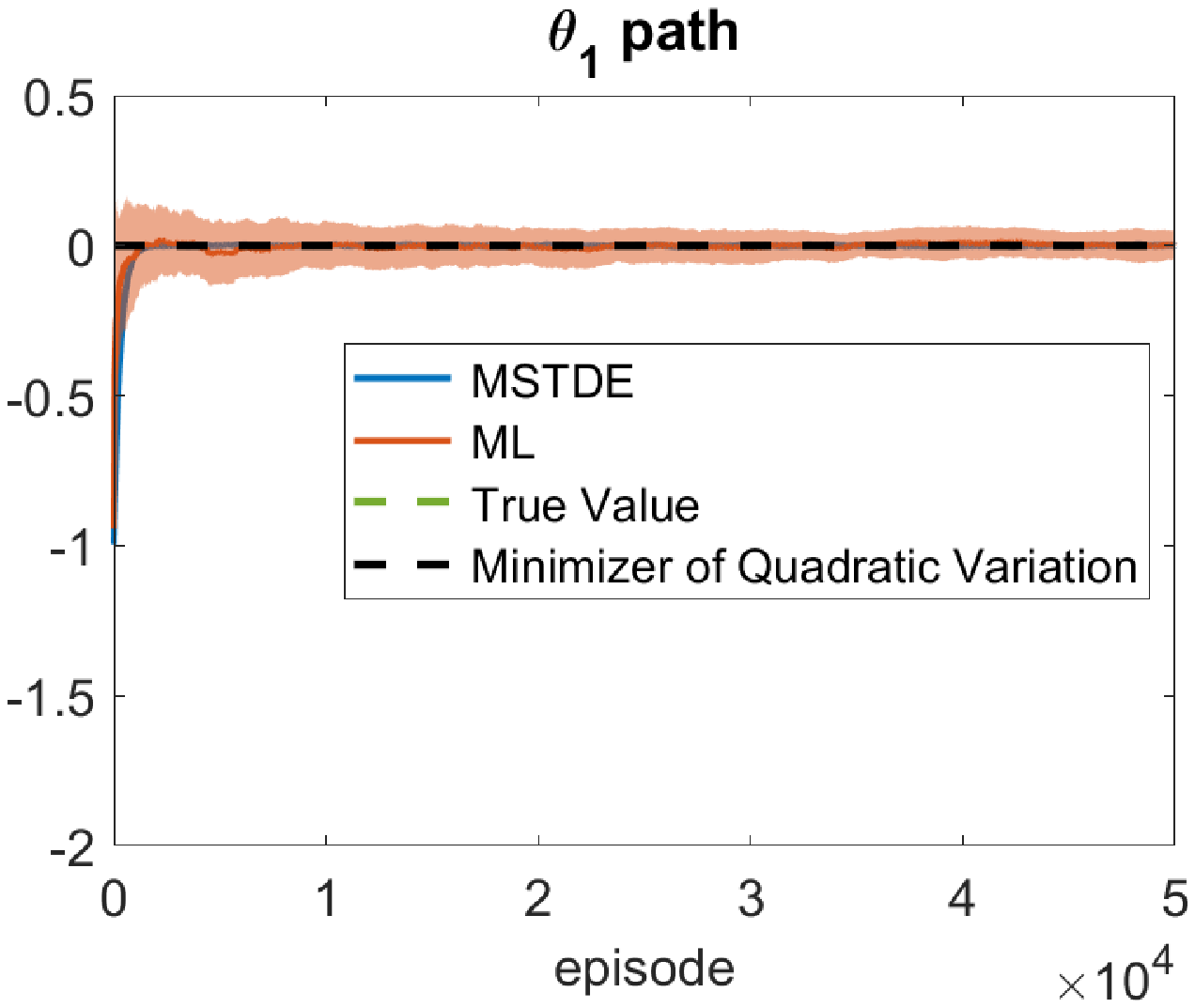}
\end{subfigure}
\begin{subfigure}{0.32\textwidth}
\centering
\includegraphics[width = 1\textwidth]{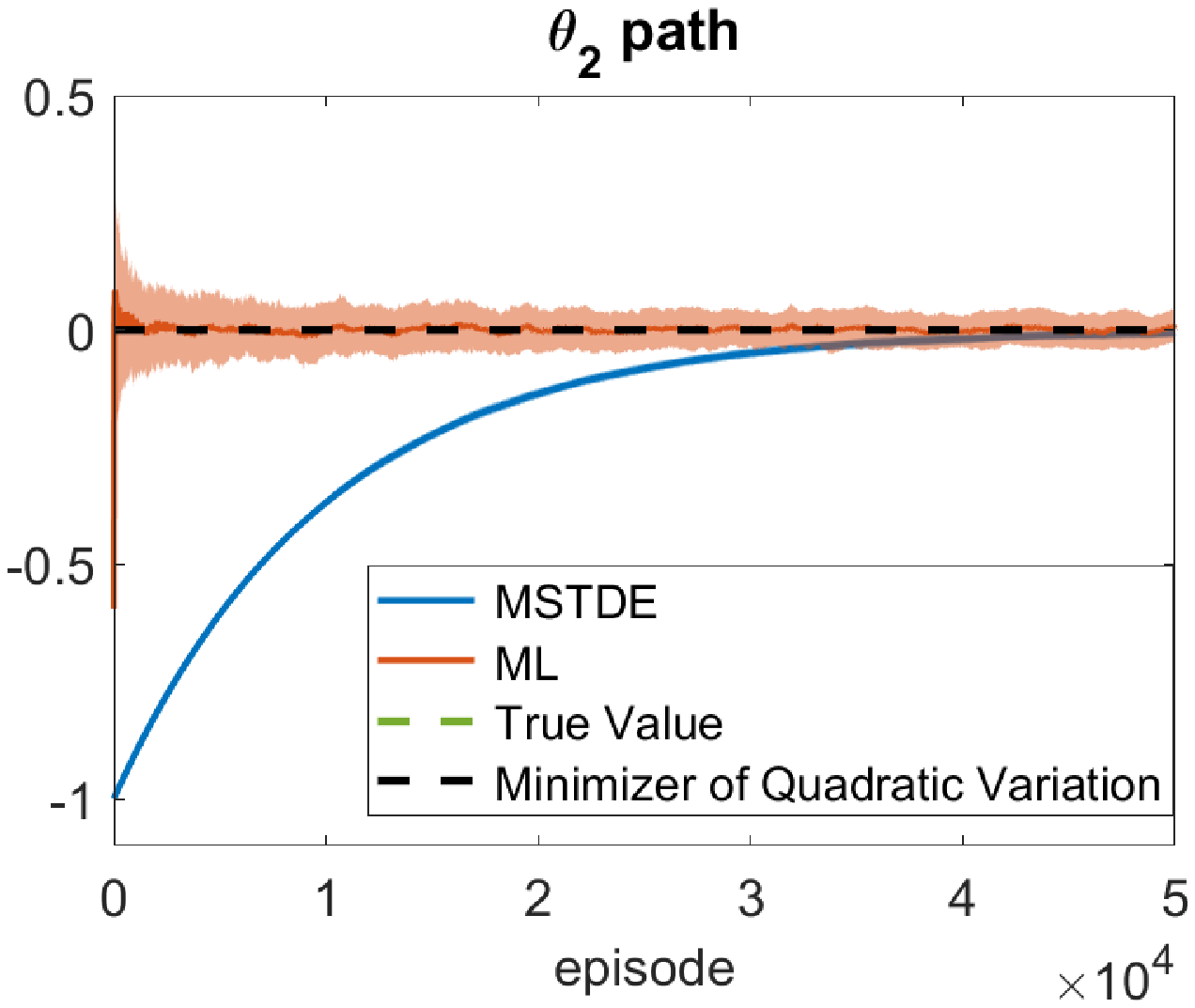}
\end{subfigure}
\caption{\textbf{The paths of parameters over episodes with different objectives for Example \ref{eg:toy running reward}.} The true solution is $\theta_{\text{true}} = (0,0,0)$. Applying SGD to minimize mean-square TD error leads to $\theta^*_{\operatorname{MSTDE}} = (-2,0,0)$. Applying SGD to minimize the martingale loss function leads to the desired solution. We repeat the experiment for 100 times to calculate the standard deviations of the predicted parameters, which are represented as the shaded areas. The width of each shaded area is twice the corresponding standard deviation.}
\label{fig:td error example 2}
\end{figure}

\medskip

Next, we present a general result stipulating that any algorithm minimizing $\operatorname{MSTDE}_{\Delta t}$ indeed converges to the minimizer of the quadratic variation of $M^\theta$.


First, it follows from It\^o's lemma that
\[ dM_t^\theta = \left[  \mathcal{L} J^{\theta}(t, X_t) + r_t\right]\dd t
+\left(\frac{\partial J^{\theta}(t, X_t)}{\partial x}\right)^\top \sigma (t,X_t) \dd W_t.
\]
Hence
\begin{equation}
\label{eq:de}  d\langle M^{\theta} \rangle_t\equiv \big( \dd M^{\theta}_t \big)^2 = \underbrace{\left[  \left(\frac{\partial J^{\theta}(t, X_t)}{\partial x}\right)^\top \sigma (t,X_t)\right]^2 (\dd W_t )^2}_{\text{leading-order term: quadratic variation}} + \underbrace{\left[  \mathcal{L} J^{\theta}(t, X_t) + r_t\right]^2 (\dd t)^2 + (\cdots) \dd W_t\dd t}_{\text{high-order small term}}.
\end{equation}
The first term on the right hand side is the leading order term since $(\dd W_t)^2 = \dd t$. Therefore, minimizing the left hand side is essentially to minimize this
leading term.

Before presenting the theorem that formalizes this result, let us note that in the time-discretization throughout this paper, $X_{t_i}$, $i = 0,1,\cdots,K$, are discrete {\it observations} of the continuous-time process $X$ which is the {\it exact} solution to \eqref{eq:model classical}, instead of  its approximation resulting from any numerical approximation scheme (such as the ones in \citealp{kloeden1992stochastic}). So, in our paper the only approximation happens in evaluating the cumulative reward between two consecutive observations -- we use the instantaneous reward $r_{t_i}$ observed at time $t_i$ multiplied by the length of the time window to approximate the total reward: $r_{t_i}\Delta t \approx \int_{t_{i}}^{t_{i+1}} r_{t_s} \dd s$ -- and in calculating the integral on $[0,T]$ by a discrete sum with the forward Euler scheme.

Clearly, the error of approximating cumulative reward is 0 if the running reward is a constant.
When it is not a constant, the convergence rate of the approximation requires some conditions, which we put forward as an assumption.

\begin{assumption}
\label{ass:growth in reward}
There exist constants $C > 0$ and $\mu_1,\mu_2,\mu_3\geq 0$, such that
\[ |r(t',x') - r(t,x)| \leq C|t'-t|^{\mu_1}|x'-x|^{\mu_2}(|x'|^{\mu_3}+|x|^{\mu_3}),\forall t',t\in [0,T], \ x',x\in \mathbb{R}^d .\]
\end{assumption}

\begin{theorem}
\label{thm:squared td minimizer}
Suppose Assumptions \ref{ass:sde regularity}, \ref{ass:fk pde}, and \ref{ass:regularity} hold. Let
\[\theta^*_{\operatorname{MSTDE}}(\Delta t)\in \arg\min_{\theta \in \Theta} \operatorname{MSTDE}_{\Delta t}(\theta)\]
and assume that $\theta^*_{\operatorname{MSTDE}}:=\lim_{\Delta t\to 0}\theta^*_{\operatorname{MSTDE}}(\Delta t)$ exists. Then
\[ \theta^*_{\operatorname{MSTDE}} \in \arg\min_{\theta \in \Theta}\E\int_{0}^T \left|  \left(\frac{\partial J^{\theta}(t, X_t)}{\partial x}\right)^\top \sigma (t,X_t)  \right|^2 \dd t.\]
Moreover, if Assumption \ref{ass:growth in reward} also holds true, then
\[ \E\int_{0}^T \left|  \left(\frac{\partial J^{\theta^*_{\operatorname{MSTDE}}(\Delta t)}(t, X_t)}{\partial x}\right)^\top \sigma (t,X_t)  \right|^2 \dd t - \min_{\theta \in \Theta}\E\int_{0}^T \left|  \left(\frac{\partial J^{\theta}(t, X_t)}{\partial x}\right)^\top \sigma (t,X_t)  \right|^2 \dd t \leq C \Delta t, \]
for some constant $C$.
\end{theorem}

In contrast, the true value function $J$ solves the PDE \eqref{eq:pde characterization} which  corresponds to the coefficient of the $(\dd t)^2$ term in (\ref{eq:de}). So once again the parameters should minimize the mean-square
Bellman's error (which as discussed earlier depends on the model parameters and hence any algorithm trying to accomplish it is not implementable).
This shows a fundamental discrepancy between the objective of  MSTDE and that of PE in the stochastic diffusion environment.

The undesirability of the na\"ive residual  gradient or MSTDE has actually been noticed in the discrete-time MDP literature. For example, \citet{sutton2009fast} point out the similar problem of MSTDE and present a simple counterexample in an adsorbing three-state Markov chain. \citet[p.272]{sutton2011reinforcement}
comment that ``by penalizing all TD errors it (MSTDE) achieves something more like temporal smoothing than accurate prediction,'' although the authors stop short of explaining
{\it why} it achieves temporal smoothing.
Our theory confirms this intuition by a rigorous analysis showing that, in the diffusion setting, the MSTDE minimizer is primarily determined through minimizing  quadratic variation. As quadratic variation measures the smoothness of a diffusion process, the value function process $\{ J^{\theta}(t, X_t),0\leq t\leq T\}$ has the smoothest trajectory at $\theta=\theta^*_{\operatorname{MSTDE}}$.  

\subsection{Online mean-square TD error algorithms}
So far our discussions have been focused on the offline setting. However, TD is often used for online learning. The question now is whether an online algorithm may correct the errors arising from MSTDE.

Let us take the one-step online method for illustration. Precisely, suppose the time discretization is $0=t_0 < t_1<\cdots < t_K = T$. At each time $t_{i+1}$, $i = 0,\cdots, K-1$,
this method updates $\theta$ by the following SGD algorithm:
\[\theta \leftarrow  \theta - \alpha \bigg( \frac{J^{\theta}(t_{i+1}, X_{t_{i+1}}) - J^{\theta}(t_{i}, X_{t_{i}})}{t_{i+1} - t_{i}} + r_{t_{i}} \bigg) \bigg( \frac{\partial J^{\theta}}{\partial \theta}(t_{i+1}, X_{t_{i+1}})  - \frac{\partial J^{\theta}}{\partial \theta}(t_{i}, X_{t_{i}})  \bigg), \]
and then loop over all episodes. 

Since multiple episodes are used, this procedure, by and large, can be viewed as drawing samples in the time direction uniformly (since the learning rate is constant, one does not differentiate among different times). Therefore,
such a one-step updating rule is equivalent to solving
$$\begin{aligned}
&\min_{\theta} \E_{t \sim \mathcal{U}(0,T)}\left[\bigg( \frac{J^{\theta}({t+\Delta t}, X_{{t+\Delta t}}) - J^{\theta}(t, X_t)}{\Delta t} + r_{t} \bigg)^2\right]\\ \approx &\min_{\theta}\frac{1}{\Delta t}\int_0^T \E[(\dd M_t^{\theta})^2] \approx \min_{\theta} \E[\int_0^T \dd \langle M^{\theta} \rangle_t ] = \min_{\theta} \E[ \langle M^{\theta} \rangle_T ],
\end{aligned}$$
where $\mathcal{U}(0,T)$ is the uniform distribution on $[0,T]$.
This is the same optimization problem as the offline learning when we use the whole trajectory; hence, theoretically, it will lead to the same  undesired solution that is determined by Theorem \ref{thm:squared td minimizer}.

%

We revisit Examples \ref{eg:toy 1} and \ref{eg:toy running reward} and implement  the above online algorithm to minimize the mean-square TD error. Figures \ref{fig:td error example 1 online} and \ref{fig:td error example 2 online} show the results of the learned parameters respectively. As predicted by our analysis, they again converge to the same wrong solutions that are determined by minimizing quadratic variation.
\begin{figure}[!htbp]
\centering
\includegraphics[width = 0.5\textwidth]{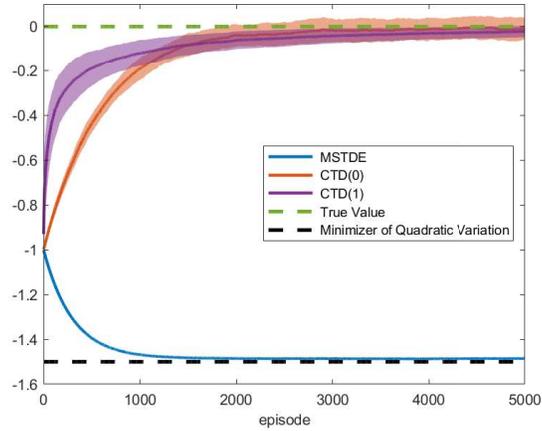}
\caption{\textbf{The paths of parameters over episodes with different objectives under the online setting for Example \ref{eg:toy 1}.} The true solution is $\theta_{\text{true}} = 0$. Applying SGD to minimize one-step $\operatorname{MSTDE}_{\Delta t}$ online leads to  $\theta^*_{\operatorname{MSTDE}} = -\frac{3}{2}$. CTD(0) and CTD(1) lead to the desired solution. We repeat the experiment for 100 times to calculate the standard deviations, which are represented as the shaded areas. The width of each shaded area is twice the corresponding standard deviation.}
\label{fig:td error example 1 online}
\end{figure}

\begin{figure}[!htbp]
\centering
\begin{subfigure}{0.32\textwidth}
\centering
\includegraphics[width = 1\textwidth]{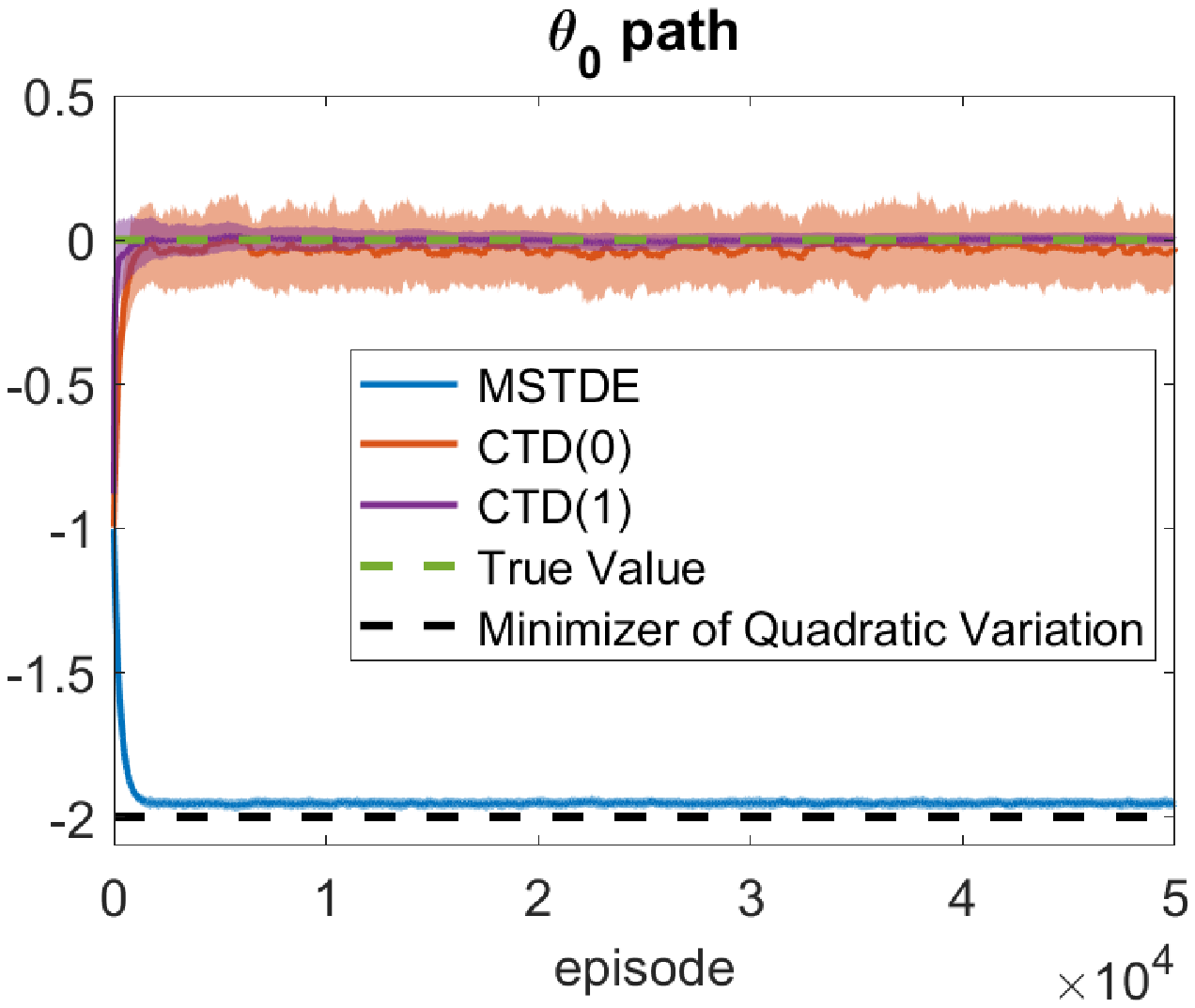}
\end{subfigure}
\begin{subfigure}{0.32\textwidth}
\centering
\includegraphics[width = 1\textwidth]{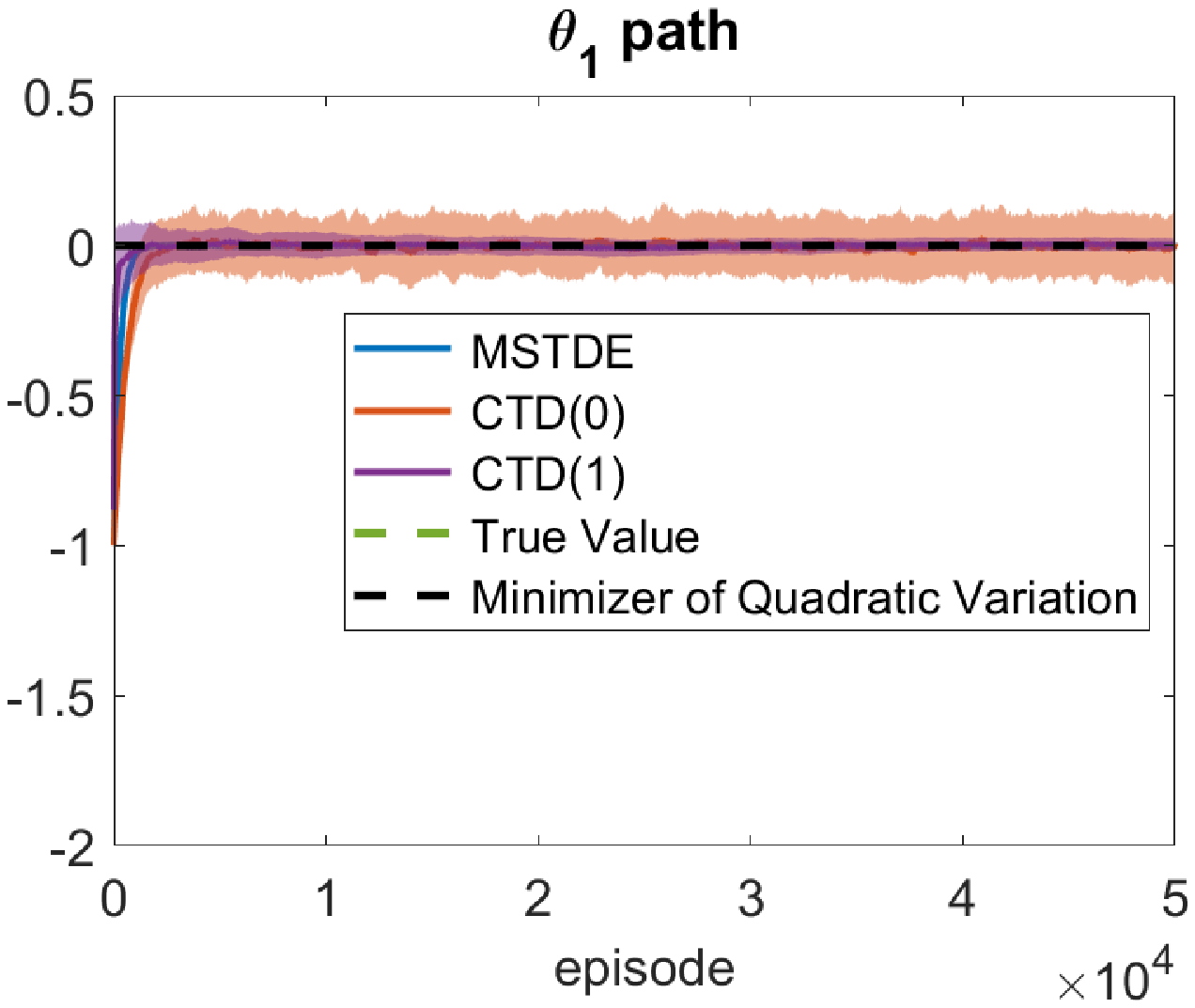}
\end{subfigure}
\begin{subfigure}{0.32\textwidth}
\centering
\includegraphics[width = 1\textwidth]{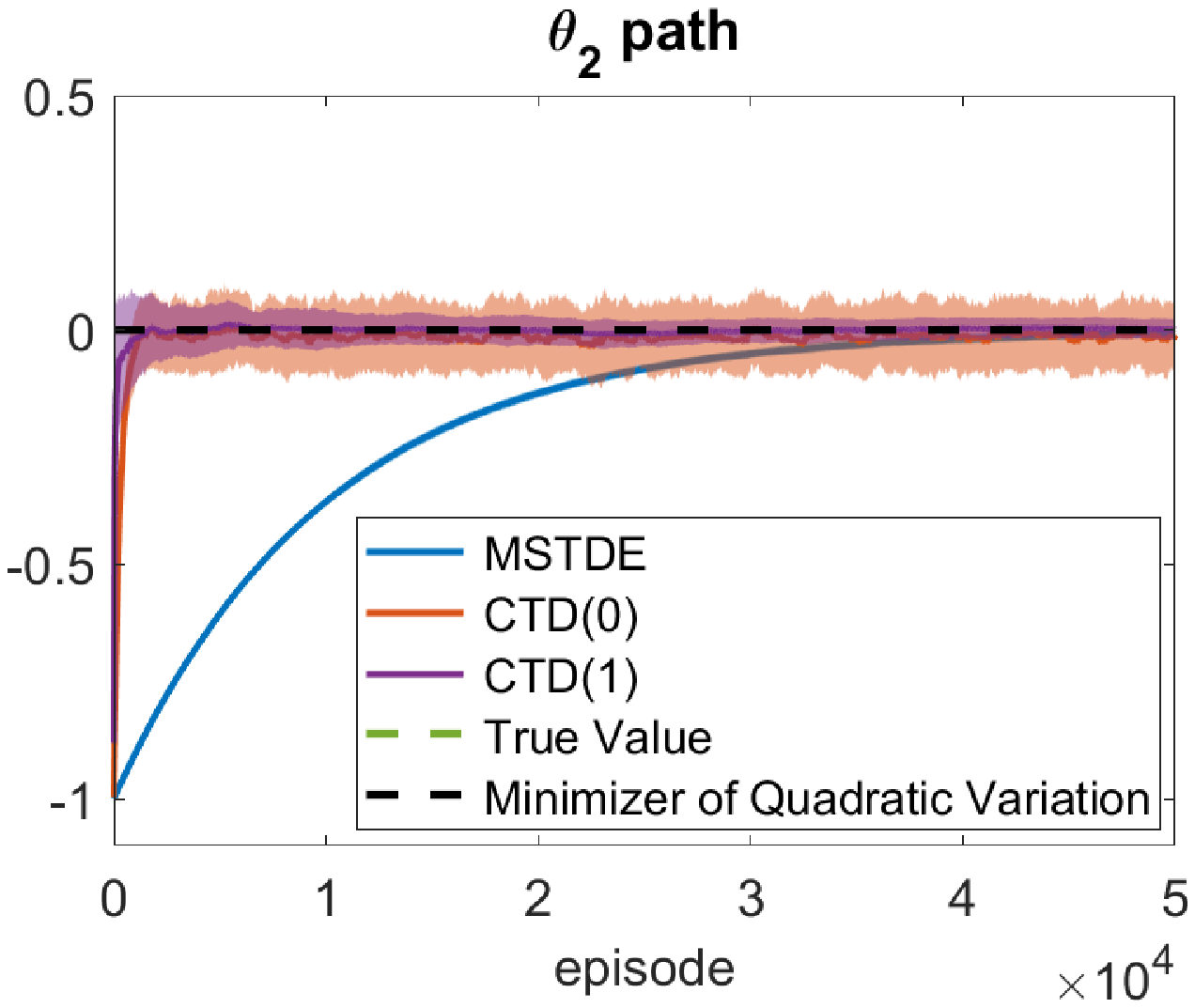}
\end{subfigure}
\caption{\textbf{The path of parameters over episodes for different objectives  under the online setting for Example \ref{eg:toy running reward}.} The true solution is $\theta_{\text{true}} = (0,0,0)$. Based on our analysis of quadratic variation, the minimizer is $\theta^*_{\operatorname{MSTDE}} = (-2,0,0)$. CTD(0) and CTD(1) lead to the desired solution. We repeat the experiment for 100 times to calculate the standard deviations, which are represented as the shaded areas. The width of each shaded area is twice the corresponding standard deviation.}
\label{fig:td error example 2 online}
\end{figure}

\section{Martingale Perspective and Approach}
\label{sec:martingale perspective}

It follows from the previous section that MSTDE is not the right objective/loss function for PE in continuous-time {\it stochastic} RL. In this section we propose and analyze several other objective functions or criteria all based on the martingality of the process $M$, and connect some of them to well-studied alternative TD algorithms for MDPs.

\subsection{Offline learning: Martingale loss function}
\label{sec:ml}
In this subsection, we propose a loss function that uses full sample trajectories  and is therefore applicable for offline learning, and test the corresponding algorithm's  performance.

Let the state process be
$\{X_t,0\leq t\leq T\}$.
Recall that the square-integrable martingality of $M_t = J(t,X_t) + \int_0^t r(s,X_s)\dd s$ characterizes the correct value function.  The martingale condition is further equivalent to
\[ M_t=\E[M_T|\mathcal{F}_t],\;\text{ for all } t\in[0,T],\]
which in turn stipulates that the process at any given time prior to the terminal time $T$ is the expectation of the terminal value conditional on all the information
available at that time. However, a fundamental property of the conditional expectation yields that $M_t$ minimizes the $L^2$-error between $M_T$ and any $\mathcal{F}_t$-measurable random variables, namely,
\[ M_t\equiv \E[M_T|\mathcal{F}_t]=\argmin_{\xi\mbox{ is $\mathcal{F}_t$-measurable}} \E|M_T-\xi|^2,\;\text{ for all } t\in[0,T].\]

This property inspires the following loss function, termed the {\it martingale loss function}:
\begin{equation}
\label{eq:martingale loss function}
\begin{aligned}
\operatorname{ML}(\theta) := & \frac{1}{2}|| M_T - M^{\theta}_{\cdot}||_{L^2}^2 = \frac{1}{2}	\E \int_0^T |M_T - M_t^{\theta}|^2 \dd t\\
\approx &\frac{1}{2} \E\bigg[\sum_{i=0}^{K-1} \bigg( h(X_T) - J^{\theta}(t_i,X_{t_i}) + \int_{t_i}^T r(s, X_s)\dd s \bigg)^2 \Delta t \bigg]  \\
\approx &\frac{1}{2} \E\bigg[  \sum_{i=0}^{K-1} \bigg( h(X_{t_K}) - J^{\theta}(t_i,X_{t_i}) + \sum_{j=i}^{K-1} r(t_{j}, X_{t_j}) \Delta t \bigg)^2 \Delta t \bigg]\\
= & \frac{1}{2}\E\bigg[  \sum_{i=0}^{K-1} \bigg( h(X_{t_K}) + \sum_{j=0}^{K-1} r(t_{j}, X_{t_j}) \Delta t   - J^{\theta}(t_i,X_{t_i}) - \sum_{j=0}^{i-1} r(t_{j}, X_{t_j}) \Delta t \bigg)^2 \Delta t \bigg]\\
=&: \operatorname{ML}_{\Delta t}(\theta),
\end{aligned}
\end{equation}
where $0=t_0 < t_1 <\cdots < t_K = T$ is a mesh grid in time.
Note that this loss function does not rely on the knowledge of the functional forms of $b,\sigma,r$ or $h$.\footnote{In particular, $M_T = h(X_T) + \int_0^T r_s\dd s$ does not depend on the parameter $\theta$ and can be directly observed from samples as the total reward obtained over $[0,T]$. } As long as we can observe the accumulated reward $\sum_{j=0}^{i-1} r(t_{j}, X_{t_j})$ along with the final reward $h(X_{T})$, the loss function can be implemented with the expectation replaced by the average over sample trajectories.  

This loss function uses the whole trajectory to calculate the difference between the predicted value function and the realized reward-to-go, minimizing which naturally leads to an unbiased estimation. This approach is the continuous-time analogue of the so-called  {\it Monte Carlo policy evaluation} with function approximation \citep{sutton2011reinforcement} in the classical MDP and RL literature. It is primarily for offline learning where one observes multiple trajectories offline and updates  estimate after observing one full trajectory.

The martingale loss objective is {\it not} of a TD type; it does not compare the approximate function values at two consecutive times.
To explain the difference between the martingale loss function and the mean-square TD error,
let us assume that the running reward $r\equiv 0$ for simplicity. In this case,
$M_t=J(t,X_t)$ is a martingale. The martingale loss  considers the difference values of $J$ between any time instance and the final time, $ J(X_T)-J(t_i,X_{t_i})=h(X_T) - J(t_i,X_{t_i})$, while the TD concerns the difference between two consecutive time instances, $J(t_{i+1},X_{t_{i+1}}) - J(t_i,X_{t_i})$.
The intuition why the former leads to the right solution is that it always compares
the current value of $J$ with that of the final time, $h(X_T)$, which is observable and thus can serve as an ultimate and correct target. In fact, instead of thinking of $J(t,x)$ as a bivariate function of time $t$ and space $x$, in any numerical procedure one is essentially looking for $K$ functions of $x$, denoted by $J_i(\cdot) = J(t_i, \cdot)$, where $i=0,\cdots,K-1$, with $J_K(x) = h(x)$  known and given.
Therefore, the martingale loss is the aggregated error between $J_i$ and $J_K=h$, minimizing which also minimizes the error between $J_i$ and $J_K$ for {\it each} $i$.
As a result, each $J_i$ converges to the correct value. In contrast, the {\it mean-square} TD error represents the aggregated intertemporal $L^2$ error between $J_i$ and $J_{i+1}$.
When computing this error, each $J_{i}$ except $J_0$ shows up twice in $|J_i - J_{i+1}|^2$ and $|J_i - J_{i-1}|^2$; so the resulting function $J_{i}$ will be twisted away from the true value, leading to a wrong solution.

Finally, we can apply SGD to minimize the proposed martingale loss function and the updating rule is given by
\begin{equation}
\label{eq:gradient martingale loss function}
\theta \leftarrow \theta + \alpha  \sum_{i=0}^{K-1} \bigg( h(X_{t_K}) + \sum_{j=i}^{K-1} r(t_{j}, X_{t_j}) \Delta t   - J^{\theta}(t_i,X_{t_i})\bigg) \frac{\partial J^{\theta}}{\partial \theta}(t_i, X_{t_i})  \Delta t .
\end{equation}
Let us call this the $\operatorname{ML}$ {\it (martingale loss) algorithm}, which is the counterpart of the \textit{gradient Monte Carlo} in classical RL with MDP, when $G(t_i):= h(X_{t_K}) + \sum_{j=i}^{K-1} r(t_{j}, X_{t_j}) \Delta t$ is taken as the Monte Carlo target at each $t_i$ \citep{sutton2011reinforcement}. 
We apply this algorithm to numerically solve Examples \ref{eg:toy 1} and \ref{eg:toy running reward}, and find that it leads to the true solution; see Figures \ref{fig:td error example 1} and \ref{fig:td error example 2}.
In our implementation, the initial value is the same as before and the learning rate is tuned for smoother convergence. In particular, the initial learning rate is set to be 0.01 and decays according to $(\sharp \text{episode})^{-0.67}$ where $\sharp \text{episode}$ denotes
the number of episode.\footnote{This decay schedule satisfies the usual requirement on the decay rate of the learning rate for the convergence of SGD algorithms. Note here our purpose is not to optimize convergence rates, but to confirm the limiting point for a convergent algorithm. Tuning the learning rate is not crucial to our results, as long as the algorithm does converge.}

The next theorem states that minimizing the martingale loss function is equivalent to minimizing the mean-square error between the estimated value function $J^{\theta}$ and the true value function $J$. This error is known as the \textit{mean-square value error} (MSVE):
\begin{equation}
\label{eq:value l2 error}
\operatorname{MSVE}(\theta) = \E \int_0^T |J(t,X_t) - J^{\theta}(t,X_t)|^2 \dd t.
\end{equation}

\begin{theorem}
\label{thm:martingale loss function}
It holds that
\[\arg\min_{\theta \in \Theta}\operatorname{ML}(\theta)  = \arg\min_{\theta \in \Theta}\operatorname{MSVE}(\theta) . \]
Moreover, under Assumptions \ref{ass:sde regularity}, \ref{ass:fk pde}, and \ref{ass:regularity}, as $\Delta t\to 0$, any convergent subsequence of the minimizer of the discretized martingale loss function $\theta^*_{\operatorname{ML}}(\Delta t) \in \arg\min_{\theta \in \Theta} \operatorname{ML}_{\Delta t}(\theta)$ converges to the minimizer of martingale loss function; that is
\[ \lim_{\Delta t \to 0}\theta^*_{\operatorname{ML}}(\Delta t) = \theta^*_{\operatorname{ML}} \in \arg\min_{\theta \in \Theta}\operatorname{ML}(\theta) =  \arg\min_{\theta \in \Theta}\operatorname{MSVE}(\theta).  \]
Furthermore, if in addition Assumption \ref{ass:growth in reward} holds, then
\[ \operatorname{ML}\left(\theta^*_{\operatorname{ML}}(\Delta t) \right) - \min_{\theta \in \Theta}\operatorname{ML}(\theta) \leq C (\Delta t)^{\min\{1,\mu_1 + \frac{\mu_2}{2}\}} \]
for some constant $C>0$.

\end{theorem}

Clearly, MSVE is a theoretically sound loss function for learning. However, by itself it does not lead to an {\it implementable} algorithm because the true value $J$ is not observable from data. Theorem \ref{thm:martingale loss function} strengthens the  theoretical foundation of the martingale loss function that is implementable. Moreover, this theorem establishes the convergence together with the convergence rate of applying any convergent algorithm developed for minimizing discrete-time martingale loss as the time step tends to zero. Therefore, it also provides a theoretical foundation for implementing the discretization procedure.

We illustrate this result with the following examples.

\begin{eg}
\label{eg:toy 1 continued wrong para}
{\rm 	Consider the same learning problem in Example \ref{eg:toy 1},  but with a different  parameterized value function given by $J^{\theta}(t,x) = \theta x^3$. Recall $X_t=W_t$ is a Brownian motion. The main difference between this example and the previous ones is that now the parametric family does {\it not} contain the true solution. Indeed, it does not even satisfy the correct terminal condition that $J^{\theta}(1,w)=x$, which could happen in more complex problems when the terminal payoff functions are unknown. Recall the true value function is $J(t,x) = x$. Let us compute the MSVE:
\[ \E\int_0^1 | J(t,X_t) - J^{\theta}(t,X_t) |^2 \dd t  = \E\int_0^1 \big( W_t - \theta W_t^3 \big)^2 \dd t  = \int_0^1 (t - 6\theta t^2 + 15\theta^2 t^3) \dd t. \]
The minimizer is $\theta^* = \frac{4}{15}$. According to Theorem \ref{thm:martingale loss function}, minimizing the martingale loss function should converge to this solution.}
\end{eg}

\begin{eg}
\label{eg:toy 1 continued wrong para 2}
{\rm	Consider the same learning problem in Example \ref{eg:toy 1}, with the parameterized value function $J^{\theta}(t,x) = x + (1-t)e^{\theta x- \frac{1}{2}\theta^2 t + \theta}$. Recall $X_t=W_t$ is a Brownian motion. This time it satisfies the terminal condition, but still does not contain the true solution. Let us compute the MSVE:
\[\begin{aligned}
	& \E\int_0^1 | J(t,X_t) - J^{\theta}(t,X_t) |^2 \dd t  = \E\int_0^1 (1-t)^2e^{2\theta W_t - \theta^2 t + 2\theta} \dd t  \\
	= & \int_0^1 (1-t)^2e^{\theta^2 t + 2\theta} \dd t = -\frac{e^{2\theta}(2-2e^{\theta^2}+2\theta^2+\theta^4)}{\theta^6}.
\end{aligned} \]
The minimizer is $\theta^* \approx -2.12568$. According to Theorem \ref{thm:martingale loss function}, minimizing the proposed martingale loss function should converge to this solution.}
\end{eg}

We test the numerical solutions to Examples \ref{eg:toy 1 continued wrong para} and \ref{eg:toy 1 continued wrong para 2} by applying our ML algorithms with SGD. The initial learning rate is taken as $0.001$ and decays as $(\sharp \text{episode})^{-0.67}$. Figures \ref{fig:toy 1 wrong para} and \ref{fig:toy 1 wrong para 2} confirm the result of Theorem \ref{thm:martingale loss function}. 

\begin{figure}[h]
\centering
\includegraphics[width=0.5\textwidth]{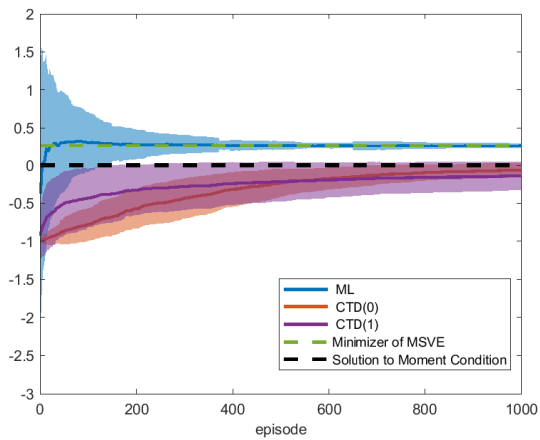}
\caption{\textbf{ML and CTD($\lambda$) methods converge to different points for Example \ref{eg:toy 1 continued wrong para}.} Applying ML algorithm leads to $\theta^*_{\operatorname{ML}} = \frac{4}{15}$, which is the minimizer of MSVE. CTD methods converge to $\theta^*_{\text{moment}} = 0$, which is the solution to the moment condition. In this case, the moment conditions associated with CTD(0) and CTD(1) have the same solution so the two algorithms converge to the same point. We repeat the experiment for 100 times to calculate the standard deviations, which are represented as the shaded areas. The width of each shaded area is twice the corresponding standard deviation.}
\label{fig:toy 1 wrong para}
\end{figure}

\begin{figure}[h]
\centering
\includegraphics[width=0.5\textwidth]{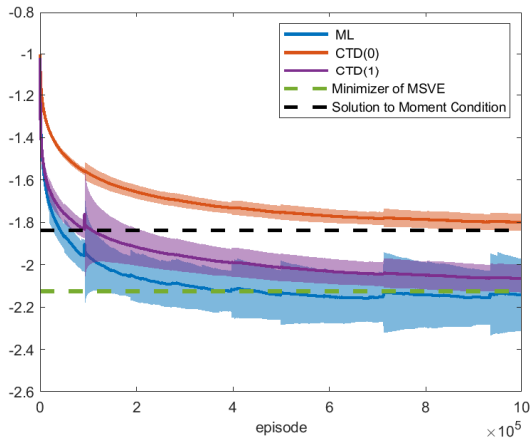}
\caption{\textbf{ML and CTD($\lambda$) methods converge to different points with different values of $\lambda$ for Example \ref{eg:toy 1 continued wrong para 2}.} Applying ML algorithm  leads to  $\theta^*_{ML} \approx -2.12568$, which is the minimizer of MSVE. CTD(0) converges to $\theta^*_{\text{moment},\operatorname{CTD}(0)} \approx -1.83923$, which is the solution to the moment condition associated with the choice of  test function for CTD(0). CTD(1) converges to $\theta^*_{\text{moment},\operatorname{CTD}(1)} \approx -2.12568$, which is the solution to the moment condition associated with the choice of  test function for CTD(1). Because of the different choices of test functions, the two CTD algorithms converge to different points. It is a coincidence that ML and CTD(1) converge to the same point. We repeat the experiment for 100 times to calculate the standard deviations, which are represented as the shaded areas. The width of each shaded area is twice the corresponding standard deviation.}
\label{fig:toy 1 wrong para 2}
\end{figure}

\subsection{Online and offline learning: TD based on martingale orthogonality conditions}
\label{sec:orthogonality}
%
We have proposed a martingale loss function to interpret the Monte Carlo PE. This approach requires the whole sample trajectory over $[0,T]$; so it is inherently offline and is  difficult to extend to the online setting where only historical samples are available when one updates the approximated function in real time. Classically, TD learning was introduced as  a remedy to enable online learning. However, based on our previous discussion, the mean-square TD error is not the correct objective function to learn the value function even though it can indeed be implemented online. 
In this section, we propose a different approach, again based on the martingality of the process $M$, that generates the continuous-time counterparts of several
well-studied TD algorithms and that works both online and offline.

This approach starts with noting that $M$ being a square-integrable martingale implies
\begin{equation}
\label{eq:martingale difference}
\E\int_{0}^{T} \xi_t \dd M_t = 0,
\end{equation}
for any $\xi \in L^2_{{\cal F}}([0,T], M)$ (called a \textit{test function}).\footnote{It would be more appropriate to call it a {\it test process} because $\xi$ needs to be generally an adapted stochastic process. However, we use the more common term ``test function".} 
In fact, the following result shows that this is a necessary and sufficient condition for the parameterized process $M^{\theta}_t$  to be a martingale.

\begin{proposition}
\label{lemma:martingale characterization test function}
In general, a diffusion process $M^{\theta}\in L^2_{{\cal F}}([0,T])$ is a martingale if and only if $\E\int_{0}^{T} \xi_t \dd M^{\theta}_t = 0$ for any $\xi \in L^2_{{\cal F}}([0,T],M^{\theta})$. In the current setting, $\E\int_{0}^{T} \xi_t \dd M^{\theta}_t = \E\int_{0}^{T} \xi_t\big[ \mathcal{L} J^{\theta}(t, X_t) + r_t  \big] \dd t$.
\end{proposition}

We call (\ref{eq:martingale difference}) a family of {\it martingale orthogonality conditions}.
In theory, one should vary all possible test functions and thus this family has {\it infinitely} many equations. For numerical approximation methods, however, we can choose finitely many test functions in special forms. Notice that, for a parametric family $\{J^{\theta}:\theta\in \Theta\subset \mathbb{R}^{L}\}$, in principle, we need at least $L$ equations as our martingale orthogonality conditions in order to fully determine
$\theta$. For example, we can take $\xi_t = \frac{\partial J^{\theta}}{\partial \theta}(t,X_t)\in \mathbb{R}^L$. (Here, and henceforth, $\xi_t$ may be vector-valued and \eqref{eq:martingale difference} is accordingly a vector-valued equation or a system of equations.) In statistics and econometrics, a problem of the type  \eqref{eq:martingale difference} with a finite number of equations is also referred to as \textit{moment conditions}, and a systematic way to analyze and solve it is known as the \textit{generalized methods of moments} (GMM); see, e.g., \cite{hansen1982large}.

To devise algorithms based on this theory, we need to answer the following questions: How to choose these finite number of test functions? And how to solve the resulting moment conditions in an effective and efficient way?
It turns out that answering these two questions suitably in our continuous setting gives rise to algorithms that correspond to several well-known conventional TD learning algorithms in discrete setting.

\begin{itemize}
\item
Choose $\xi_t = \frac{\partial J^{\theta}}{\partial \theta}(t,X_t)$, and use stochastic approximation \citep{robbins1951stochastic} to update parameters after a whole  episode (offline):
\[\theta \leftarrow  \theta + \alpha \int_0^T \frac{\partial J^{\theta}}{\partial \theta}(t,X_t) \dd M^{\theta}_t \approx \theta + \alpha \sum_{i=0}^{K-1}\frac{\partial J^{\theta}}{\partial \theta}(t_i,X_{t_i})\big( J^{\theta}(t_{i+1},X_{t_{i+1}}) - J^{\theta}(t_i,X_{t_i}) + r_{t_i}\Delta t \big) ,  \]
or to update parameters at every time step (online):
\[\theta \leftarrow  \theta + \alpha  \frac{\partial J^{\theta}}{\partial \theta}(t,X_t) \dd M^{\theta}_t \approx \theta + \alpha \frac{\partial J^{\theta}}{\partial \theta}(t_i,X_{t_i})\big( J^{\theta}(t_{i+1},X_{t_{i+1}}) - J^{\theta}(t_i,X_{t_i}) + r_{t_i}\Delta t \big). \]

These algorithms correspond to the TD(0) algorithm \citep{sutton1988learning}.

\item Choose $\xi_t = \int_0^{t}\lambda^{t-s}\frac{\partial J^{\theta}}{\partial \theta}(s,X_s)\dd s$, where $0<\lambda\leq 1$, and use stochastic approximation to update parameters after one episode:
\[\begin{aligned}
\theta \leftarrow  & \theta+ \alpha \int_0^T \int_0^{t}\lambda^{t-s}\frac{\partial J^{\theta}}{\partial \theta}(s,X_s)\dd s \dd M^{\theta}_t \\
& \approx \theta + \alpha \sum_{i=0}^{K-1}\sum_{j=0}^{i}\Delta t \lambda^{(i-j)\Delta t}\frac{\partial J^{\theta}}{\partial \theta}(t_j,X_{t_j})\big( J^{\theta}(t_{i+1},X_{t_{i+1}}) - J^{\theta}(t_i,X_{t_i}) + r_{t_i}\Delta t \big),
\end{aligned} \]
or to update parameters at every time step:
\[\begin{aligned}
\theta \leftarrow  & \theta+ \alpha \int_0^{t}\lambda^{t-s}\frac{\partial J^{\theta}}{\partial \theta}(s,X_s)\dd s \dd M^{\theta}_t \\
& \approx \theta + \alpha \sum_{j=0}^{i}\Delta t \lambda^{(i-j)\Delta t}\frac{\partial J^{\theta}}{\partial \theta}(t_j,X_{t_j})\big( J^{\theta}(t_{i+1},X_{t_{i+1}}) - J^{\theta}(t_i,X_{t_i}) + r_{t_i}\Delta t \big).
\end{aligned} \]
%
%

These algorithms correspond to the
TD($\lambda$) algorithm \citep{sutton1988learning}. Here the parameter $\lambda$ dictates how much weight to be put on historical predictions in the procedure.
When $\lambda=1$, it puts equal weight on all the past predictions. The smaller $\lambda$ becomes, the more weight on more recent predictions. When $\lambda=0$,
past predictions do not matter, resulting in the TD(0) algorithm.

It should be noted that TD(0) and TD($\lambda$)  algorithms are {\it not} exactly gradient based; rather, they use {\it stochastic approximation} as a first-order method to solve \eqref{eq:martingale difference}. In the literature they are also referred to as {\it semi-gradient} TD algorithms \citep{sutton2011reinforcement} because they do include a part of the gradient.  

\item Choose $\xi_t = \frac{\partial J^{\theta}}{\partial \theta}(t,X_t)$ and take the approximated value function to be a linear combination  of a series of basis functions: $J^{\theta}(t,x) = \sum_{j=1}^L \theta_j \phi_j(t,x)$. Then $\frac{\partial J^{\theta}}{\partial \theta}(t,x) = \phi(t,x):=(\phi_1(t,x),\cdots,\phi_L(t,x))^\top \in \mathbb{R}^L$. In this case, \eqref{eq:martingale difference} becomes a system of linear equatins in $\theta$ and can be solved explicitly as
$$\theta = -\left[\E\int_0^{T} \phi(t,X_t)\big( \dd \phi(t,X_t)^{\top}\big)  \right]^{-1} \E\int_0^{T} r_t \phi(t,X_t) \dd t,$$
assuming the matrix inverse  exists.  The expectation can be estimated using sample average across multiple trajectories. Hence, if there are $M$ episodes, we have
\[\begin{aligned}
\theta = & -\bigg(\frac{1}{M} \sum_{k=1}^{M} \int_0^{T}  \phi(t,X^{(k)}_t)\big( \dd \phi(t,X^{(k)}_t)^{\top}\big)  \bigg)^{-1} \bigg( \frac{1}{M}\sum_{k=1}^M \int_0^T r_t^{(k)}\phi(t,X^{(k)}_t)\dd t \bigg) \\
\approx & -\bigg(\frac{1}{M} \sum_{k=1}^{M} \sum_{i=0}^{K-1} \phi(t_i,X^{(k)}_{t_i})\big( \phi(t_{i+1},X^{(k)}_{t_{i+1}})^{\top} - \phi(t_i,X^{(k)}_{t_i})^{\top}\big) \bigg)^{-1} \bigg( \frac{1}{M}\sum_{k=1}^M \sum_{i=0}^{K-1} r^{(k)}_{t_i}\phi(t_i,X^{(k)}_{t_i})\Delta t \bigg),
\end{aligned}  \]
where the superscript $(k)$ signifies that the corresponding observations are taken from the $k$-th episode.
If there is only one trajectory up to time $t = t_j$, then we
estimate the parameter  using long-time average (under certain ergodicity condition) to obtain
\[\begin{aligned}
\theta = & -\bigg( \frac{1}{t}\int_0^{t}  \phi(s,X_s)\big( \dd \phi(s,X_s)^{\top} \big) \bigg)^{-1} \bigg( \frac{1}{t}\int_0^t r_s\phi(s,X_s)\dd s \bigg) \\
\approx & -\bigg(\frac{1}{j} \sum_{i=0}^{j-1} \phi(t_i,X_{t_i})\big( \phi(t_{i+1},X_{t_{i+1}})^{\top} - \phi(t_i,X_{t_i})^{\top}\big) \bigg)^{-1} \bigg( \frac{1}{j} \sum_{i=0}^{j-1} r_{t_i}\phi(t_i,X_{t_i})\Delta t \bigg).
\end{aligned} \]

These algorithms correspond to the (linear) least square TD(0), or LSTD(0), algorithms \citep{bradtke1996linear}. LSTD and its variants \citep{boyan2002technical} are often discussed in  the context of linear function approximation. Despite the name of ``least square'', it does not solve any minimization problem per se; instead it uses the linear structure to obtain the exact solution to  \eqref{eq:martingale difference} and then use sample average to estimate the expectation. \citet{xu2002efficient} and \citet{geramifard2006incremental} develop a more efficient way to implement this solution in a recursive  way. The reason why it is called ``least square'' comes from the instrumental variable approach  to regression problems \citep{ljung1983theory}.\footnote{Instrumental variables are widely used in statistics and econometrics to estimate causal relationship when exploratory variables are endogenous. A necessary condition for being a  instrumental variable is that it must be uncorrelated with the residual term. In the context of TD learning, the residual term is the TD error.} \citet{bradtke1996linear} show that the basis functions in LSTD are indeed instrumental variables. 

\item 
We choose $\xi_t = \frac{\partial J^{\theta}}{\partial \theta}(t,X_t)$, and  minimize the GMM objective function
\[ \operatorname{GMM}(\theta) = \frac{1}{2}\E\left[\int_0^{T} \xi_t \dd M^{\theta}_t\right]^\top A \E\left[\int_0^{T} \xi_t \dd M^{\theta}_t\right],\]
where $A$ is a given matrix. Different choices of $A$ lead to a variety of algorithms corresponding to what are broadly  called {\it gradient TD} (GTD) algorithms for MDPs. For example,
taking $A=I$, the identity matrix, corresponds to GTD(0) by \cite{sutton2009fast}.
Another choice is $A= [\E\int_0^{T} \xi_t\xi_t^\top\dd t]^{-1}$. In this case,
the gradient of the objective in $\theta$ is (noting $\xi_t$ also depends on $\theta$)
\[ \E\left[\int_0^{T} \dd \left( \frac{\partial M^{\theta}}{\partial \theta}(t,X_t) \right) \xi_t^\top \right] u +  \E\left[\int_0^{T} \frac{\partial \xi_t}{\partial \theta}^\top \dd M_t^{\theta}\right] u - \E\left[\int_0^{T} u^\top \xi_t  \frac{\partial \xi_t}{\partial \theta}(t,X_t)^\top u\dd t\right], \]
where $u: = [\E\int_0^{T} \xi_t\xi_t^\top\dd t]^{-1}\E[\int_0^{T} \xi_t \dd M^{\theta}_t]$ and  $\frac{\partial \xi_t}{\partial \theta}$ is the Jacobian matrix. In particular, $\frac{\partial \xi_t}{\partial \theta} = \frac{\partial^2 J^{\theta}}{\partial \theta^2}(t,X_t) $ is the Hessian matrix and hence is symmetric. When $J^{\theta}(t,x) = \sum_{j=1}^L \theta_j \phi_j(t,x)$ is the linear span of basis functions,  the last two terms of the gradient will vanish because $\frac{\partial \xi_t}{\partial \theta} = \frac{\partial^2 J^{\theta}}{\partial \theta^2}(t,X_t) = 0$.

Two GTD algorithms, called
GTD2 and TDC \citep{sutton2008convergent,sutton2009fast}, apply stochastic approximation  at two different time scales to update $u$ and $\theta$ respectively. Specifically, in both algorithms,
$u$ is estimated with long-term average:
\[ u \leftarrow u + \alpha_u \left[ \xi_t \dd M^{\theta}_t - \xi_t\xi_t^\top u\Delta t \right]\approx  u + \alpha_u \big[ \xi_{t_i} ( M^{\theta}_{t_{i+1}} - M^{\theta}_{t_{i}}) - \xi_{t_i}\xi_{t_i}^\top u\Delta t \big], \]
and then $\theta$ is updated with two different one-step sampling methods.
The GTD2 algorithm proceeds as follows:
\[
\begin{aligned}
\theta \leftarrow & \theta - \alpha_{\theta} \left[ \dd \left( \frac{\partial M^{\theta}}{\partial \theta}(t,X_t) \right)\xi_t^\top u +  \frac{\partial \xi_t}{\partial \theta} ^\top \dd M_t^{\theta}u - u^\top\xi_t  \frac{\partial \xi_t}{\partial \theta} ^\top u \Delta t \right] \\
\approx &  \theta - \alpha_{\theta} \bigg[ \left( \frac{\partial M^{\theta}}{\partial \theta}(t_{i+1},X_{t_{i+1}}) - \frac{\partial M^{\theta}}{\partial \theta}(t_{i},X_{t_{i}}) \right)\xi_{t_i}^\top u \\
& + \frac{\partial \xi}{\partial \theta}(t_i, X_{t_i})^\top ( M_{t_{i+1}}^{\theta} -M_{t_{i}}^{\theta} )u - u^\top\xi_{t_{i}}  \frac{\partial \xi}{\partial \theta}(t_i,X_{t_i})^\top u \Delta t \bigg] .
\end{aligned}  \]
The TDC algorithm observes that $\frac{\partial J^{\theta}}{\partial \theta}(t,X_t) = \xi_t$ and hence updates $\theta$  by
\[
\begin{aligned}
\theta \leftarrow & \theta - \alpha_{\theta} \left[ \xi_t \dd M^{\theta}_t +\xi_{t'}\xi_t^\top u \Delta t  + \frac{\partial \xi_t}{\partial \theta}^\top \dd M_t^{\theta}u - u^\top\xi_t  \frac{\partial \xi_t}{\partial \theta}^\top  u \Delta t \right] \\
\approx & \theta - \alpha_{\theta} \bigg[ \xi_{t_i} ( M_{t_{i+1}}^{\theta} -M_{t_{i}}^{\theta} ) +\xi_{t_{i+1}}\xi_{t_i}^\top u \Delta t \\
& + \frac{\partial \xi}{\partial \theta}(t_i,X_{t_i})^\top ( M_{t_{i+1}}^{\theta} -M_{t_{i}}^{\theta} ) u - u^\top\xi_{t_i} \frac{\partial \xi}{\partial \theta}(t_i, X_{t_i})^\top  u \Delta t \bigg].
\end{aligned}  \]

GTD(0), GTD2 and TDC are gradient based methods as well as typical GMM methods to minimize a quadratic form of the conditions \eqref{eq:martingale difference}, where expectations are estimated using long term averages  as in \citet{hansen1996finite}. \citet{sutton2008convergent,sutton2009fast} and \citet{maei2009convergent} study stochastic approximation and incremental implementation of the gradient of  quadratic functions for  linear and non-linear function approximation respectively.

\end{itemize}

All the above  methods can be classified into two types. The first type applies stochastic approximation to solve the moment conditions directly, like TD($\lambda$). This is the classical TD learning. The second type follows GMM to minimize a quadratic function of the moment conditions by computing its gradient and approximating  the expectation by either long-term average or one long sample trajectory. We call it the GTD method, following \citet{sutton2008convergent}. LSTD is  limited to  linear approximation only and hence can be considered  as a special case of the first type when the moment conditions can be explicitly solved so the only computation needed is to estimate the expectation using samples.

It should be noted that although the goal of this paper is to devise PE algorithms for the continuous setting, the {\it actual} implementations of the various algorithms described above are all discrete-time with a {\it fixed} mesh size  $\Delta t$.
These algorithms correspond to some discrete-time versions of the moment conditions. So natural and important theoretical questions are whether such an algorithm converges to the solution of the continuous-time version
of the respective moment conditions as $\Delta t\rightarrow 0$ and, if yes,  what the convergence rate is. The next two theorems answer these questions.

Henceforth we impose the following assumption on the test functions used for moment conditions.
\begin{assumption}
\label{ass:test functions regularity}
A test function $\xi = \{\xi_t,0\leq t\leq T\}$ is an $\mathbb{R}^{L'}$-valued adapted process satisfying $|\xi|\in L^2_{\f}([0,T];M^{\theta})$ and $\E[|\xi_{t'} - \xi_t|^2] \leq C(\theta)|t'-t|^{\alpha}$ for any $t,t'\in [0,T]$, where $C(\theta)$ is a continuous function of $\theta$ and $0 < \alpha \leq 2$ is a given constant.
\end{assumption}

The following is about the convergence of the TD type algorithms when $\Delta t\to 0$.

\begin{theorem}
\label{thm:td equation}
Denote by $\theta^*_{\operatorname{moment}}(\Delta t)$ the solution to the discrete-time moment conditions with mesh size $\Delta t$:
\[ \E\sum_{i=0}^{K-1}\xi_{t_i}(M^{\theta}_{t_{i+1}} - M^{\theta}_{t_i}) = 0. \]
Then, under Assumptions \ref{ass:sde regularity}, \ref{ass:fk pde}, \ref{ass:regularity}, and \ref{ass:test functions regularity}, as $\Delta t\to 0$, any convergent subsequence of $\theta^*_{\operatorname{moment}}(\Delta t)$ converges to the solution to the continuous-time moment conditions \eqref{eq:martingale difference};  that is,
\[ \theta^*_{\operatorname{moment}}:=\lim_{\Delta t \to 0}\theta^*_{\operatorname{moment}}(\Delta t) \]
solves \eqref{eq:martingale difference}.
Moreover, if in addition Assumption \ref{ass:growth in reward}  holds, then
\[ | \E\int_0^T \xi_t\dd M_t^{\theta^*_{\operatorname{moment}}(\Delta t)} | \leq C (\Delta t)^{\min\{\frac{\alpha}{2},\mu_1 + \frac{\mu_2}{2}\}} \]
for some constant $C$.
\end{theorem}

The next theorem is on the convergence of the GTD type algorithms when $\Delta t\to 0$.

\begin{theorem}
\label{thm:gmm objective}
Let the discretized GMM objective function be
\[ \operatorname{GMM}_{\Delta t}(\theta):=  \frac{1}{2}\E\left[\sum_{i=0}^{K-1} \xi_{t_i}^{\theta}  (M^{\theta}_{t_{i+1}} - M^{\theta}_{t_i})\right]^\top A_{\Delta t} \E\left[\sum_{i=0}^{K-1} \xi_{t_i}^{\theta}  (M^{\theta}_{t_{i+1}} - M^{\theta}_{t_i})\right], \]
where $A_{\Delta t}$ is a discretized approximation of $A$ satisfying  $|A_{\Delta t} - A| \leq \tilde{C}(\theta)|\Delta t|^{\beta} $,  with $\tilde{C}(\theta)$ being a continuous function of $\theta$ and $\beta > 0$ a constant.\footnote{When $A$ is a constant as in  GTD(0), $A_{\Delta t} = A$. When $A = [\E\int_0^T \xi_t\xi_t^\top\dd t]^{-1}$ as in GTD2 and TDC, $A_{\Delta t}: = [\E\sum_{i=0}^{K-1} \xi_{t_i}\xi_{t_i}^\top\Delta t]^{-1}$ is the discretization approximation of this integral.}
Then, under Assumptions \ref{ass:sde regularity}, \ref{ass:fk pde}, \ref{ass:regularity}, and \ref{ass:test functions regularity}, as $\Delta t\to 0$,
any convergent subsequence of the minimizer of the discretized GMM objective function $\theta^*_{\operatorname{GMM}}(\Delta t) \in \arg\min_{\theta \in \Theta} \operatorname{GMM}_{\Delta t}(\theta)$ converges to the minimizer of the continuous GMM objective function;  that is,
\[ \lim_{\Delta t \to 0}\theta^*_{\operatorname{GMM}}(\Delta t) = \theta^*_{GMM} \in \arg\min_{\theta \in \Theta}\operatorname{GMM}(\theta) .  \]
Moreover, if in addition Assumption \ref{ass:growth in reward} holds, then
\[ \operatorname{GMM}\left(\theta^*_{\operatorname{GMM}}(\Delta t) \right) - \min_{\theta \in \Theta}\operatorname{GMM}(\theta) \leq C (\Delta t)^{\min\{\frac{\alpha}{2},\mu_1 + \frac{\mu_2}{2},\beta\}} \]
for some constant $C$.
\end{theorem}	 	

From now on, to distinguish our algorithms from their existing discrete-time counterparts, we will add a prefix ``C", signifying ``continuous", to the names of the algorithms. For instance, we will call them CTD($\lambda$), CLSTD, and so on.

The next important question is in what sense the aforementioned methods approximate the correct value function. First, a convergent CTD(0) or CTD($\lambda$) algorithm should converge to the solution to the moment conditions \eqref{eq:martingale difference} based on the respective choices of test functions. Intuitively, such an algorithm searches for one particular Bellman's error process $\mathcal{L}J^{\theta}(t, X_t) + r_t$ {\it within the parametric family} such that it is orthogonal to the underlying test functions. These TD learning methods are usually easy to implement and work effectively in many applications.
To demonstrate, we re-compute the problems in Examples \ref{eg:toy 1} and \ref{eg:toy running reward} using online CTD(0) and CTD(1) algorithms with stochastic approximation. The learning rate is chosen as 0.01. Both algorithms  converge to the correct values; see Figures \ref{fig:td error example 1 online} and \ref{fig:td error example 2 online}.

However, a caveat is that these algorithms may not always work. On one hand,
due to possible misspecification of the parametric family, solutions to the moment conditions may not exist, in which case the algorithms will not converge; see Example \ref{eg:toy 1 continued wrong para 3} below where the test function is not properly chosen. 
On the other hand, as the following continuations of Examples \ref{eg:toy 1 continued wrong para} and \ref{eg:toy 1 continued wrong para 2} illustrate, even if the solution to the moment conditions exists uniquely and an algorithm converges, the resulting solution may vary depending on the choice of test functions.


\begin{eg}
\label{eg:toy 1 continued wrong para 3}
{\rm Consider the same learning problem in Example \ref{eg:toy 1}, now with the parameterized value function $J^{\theta}(t,x) = x + (1-t)e^{\theta x - \frac{1}{2}\theta^2 t}[(\theta + 1)^2 + 1]$. Recall $X_t=W_t$ is a Brownian motion. This family does not contain the true solution. If we choose the test function to be the constant 1 and use CTD(0), a convergent algorithm should solve
\[ 0 = \E\int_0^1 e^{\theta W_t - \frac{1}{2}\theta^2 t}[(\theta + 1)^2 + 1] \dd t = (\theta + 1)^2 + 1.\]
However, there is no solution to the above equation. Our implementation of CTD(0)
indeed generates a divergent sequence of iterates; see Figure \ref{fig:toy 1 wrong para 3}.

On the other hand, we can use the martingale loss function to get a solution. Indeed, it follows from Theorem \ref{thm:martingale loss function} that the ML algorithm is equivalent to minimizing
\[  \begin{aligned}
	& \E\int_0^1 | J(t,X_t) - J^{\theta}(t,X_t) |^2 \dd t   =  \E\int_0^1 (1-t)^2e^{2\theta W_t -\theta^2 t}[(\theta + 1)^2 + 1]^2 \dd t  \\
	& = \int_0^1 (1-t)^2e^{\theta^2 t}\dd t[(\theta + 1)^2 + 1]^2 = -\frac{2-2e^{\theta^2} + 2\theta^2+\theta^4}{\theta^6}[(\theta + 1)^2 + 1]^2,
\end{aligned} \]
whose minimizer is around $-0.875301$. The implementation of ML confirms this theoretical prediction; see Figure \ref{fig:toy 1 wrong para 3}.

}
\end{eg}

\setcounter{eg}{2}
\begin{eg}[Continued]
{\rm	We revisit  this example where $J^{\theta}(t,x) = \theta x^3$. 
Recall $X_t=W_t$ is a Brownian motion. There is no running reward so $M^{\theta}_t = \theta X_t^3$ and $\dd M^{\theta}_t = 3\theta W_t^2 \dd W_t + 3\theta W_t\dd t$. Hence, any test function that is not identically 0 leads to the {\it only} solution  $\theta^* = 0$.   As a result, any convergent CTD
algorithm  should converge to 0, yielding a value function $J^{\theta^*}(t,x) =0$ that is significantly deviated from the true one $J(t,x) = x$. See Figure \ref{fig:toy 1 wrong para} for the CTD(0) and CTD(1)  experiment results.}
\end{eg}

\begin{eg}[Continued]
{\rm	Consider the parameterized value function $J^{\theta}(t,x) = x + (1-t)e^{\theta x - \frac{1}{2}\theta^2 t + \theta}$. Recall $X_t=W_t$ is a Brownian motion. In this case,  $\dd J^{\theta}(t,X_t) = \dd W_t + (1-t) \theta e^{\theta W_t - \frac{1}{2}\theta^2 t + \theta}\dd W_t - e^{\theta W_t - \frac{1}{2}\theta^2 t + \theta} \dd t$.

If we use the one-step or one-episode CTD(0) algorithm with $\xi_t = \frac{\partial J^{\theta}}{\partial \theta}(t,X_t) = (1-t)e^{\theta X_t - \frac{1}{2}\theta^2 t + \theta}(X_t - \theta t + 1)$, then the moment condition \eqref{eq:martingale difference} becomes
\[\begin{aligned}
	0 = & \E[\int_0^{1} (1-t)e^{\theta W_t - \frac{1}{2}\theta^2 t + \theta}(W_t - \theta t + 1) e^{\theta W_t - \frac{1}{2}\theta^2 t + \theta} \dd t] \\
	& =  \int_0^{1} (1-t)(1+t\theta)e^{(2+t\theta)\theta}  \dd t = \frac{e^{2\theta}[2-\theta+\theta^2-\theta^3+e^{\theta^2}(-2+\theta+\theta^2)]}{\theta^5} .
\end{aligned}   \]
This equation has a unique solution $\theta \approx -1.83923$. A convergent CTD(0) algorithm should converge to this point, which is however {\it different} from the solution produced by the martingale loss function approach.

If we use the one-step or one-episode CTD(1) algorithm with $\xi_t =\int_0^t \frac{\partial J^{\theta}}{\partial \theta}(s,X_s)\dd s = \int_0^t (1-s)e^{\theta W_s - \frac{1}{2}\theta^2 s + \theta}(W_s - \theta s + 1)\dd s$, then the moment condition \eqref{eq:martingale difference} is
\[\begin{aligned}
	0 = & \E[\int_0^{1} \int_0^t (1-\tau)e^{\theta W_{\tau} - \frac{1}{2}\theta^2 \tau + \theta}(W_{\tau} - \theta \tau + 1)\dd \tau e^{\theta W_t - \frac{1}{2}\theta^2 t + \theta} \dd t] \\
	& = \int_0^1 \int_0^t \E\big[ e^{\theta W_{\tau} -\frac{1}{2}\theta^2\tau}(W_{\tau} - \theta \tau + 1) \E[e^{\theta W_t - \frac{1}{2}\theta^2 t}|\f_{\tau}] \big](1-\tau) e^{2\theta}\dd \tau \dd t \\
	& = \int_0^1 \int_0^t \E[e^{2\theta W_{\tau}}(W_{\tau} - \theta \tau + 1)  ](1-\tau)e^{-\theta^2 \tau + 2\theta}  \dd \tau \dd t \\
	& = \int_0^1 \int_0^t (1+\tau \theta)(1-\tau)e^{\theta^2 \tau + 2\theta}  \dd \tau \dd t \\
	& = \frac{e^{2\theta}[ 6 + 2e^{\theta^2}(-3+\theta+\theta^2) - (-1+\theta)\theta(-2+2\theta+\theta^3) ] }{\theta^7} .
\end{aligned}  \]
There is a unique solution $\theta \approx -2.12568$, to which a convergent CTD(1) algorithm converges. This solution coincides with the one by the martingale loss function approach.

The implementations of the above  algorithms are reported in Figure \ref{fig:toy 1 wrong para 2}, which are consistent with the theoretical analysis.}
\end{eg}

When the parametric family is a linear span of some basis functions, the unique solution that solves the moment conditions is theoretically guaranteed under very mild conditions, which is numerically generated by the CLSTD algorithm. More generally, all the CGTD methods aim to minimize some quadratic forms of the moment conditions regardless of whether the existence and/or uniqueness of the solution to the conditions holds, and hence usually produce more robust results. Moreover, these methods have a clear geometric interpretation. Recall that the true value function minimizes Bellman's error to  zero. The space of approximate linear functions may not contain the true function, but the CGTD algorithms minimize the {\it projection} of Bellman's error onto the linear space. This intuition is formalized in
\citet{sutton2009fast} and \citet{maei2009convergent}, who show that the discrete-time GTD minimizers, instead of directly approximating the value function, minimize  the so-called \textit{mean-square projected Bellman's error} (MSPBE). Here, we present a continuous-time version of the result with a more general choice of the test functions.

\begin{theorem}
\label{thm:projected msbe}
For $L'$ linearly independent test functions $\xi^{\theta,(1)}, \cdots, \xi^{\theta,(L')} \in L^2_{\f}([0,T])$, 
denote by $\Pi_{\theta}$  the projection operator onto the linear space spanned by $\{ \xi^{\theta,(1)}, \cdots, \xi^{\theta,(L')}  \}$. Then
\[
\begin{aligned}
& \frac{1}{2}\E\left[\int_0^T \xi_t^{\theta} \dd M^{\theta}_t\right]^\top\left[\E\int_0^T \xi_t^{\theta} (\xi_t^{\theta})^\top \dd t\right]^{-1}\E\left[\int_0^T \xi_t^{\theta} \dd M^{\theta}_t\right] \\
=& \frac{1}{2}\E\left[\int_0^T \big(\mathcal{L} J^{\theta}(t, X_{t}) + r_{t}\big)\xi_t^{\theta}\dd t \right]^\top \left[\E\int_0^T \xi_t^{\theta} (\xi_t^{\theta})^\top \dd t\right]^{-1} \E\left[\int_0^T \big(\mathcal{L} J^{\theta}(t, X_{t}) + r_{t}\big)\xi_t^{\theta}\dd t \right] \\
= & \frac{1}{2}|| \Pi_{\theta} \big(\mathcal{L} J^{\theta}(\cdot, X_{\cdot}) + r_{\cdot} \big)  ||_{L^2}^2 =: \operatorname{MSPBE}(\theta).
\end{aligned}   \]

\end{theorem}

Recall Example \ref{eg:toy 1 continued wrong para 3} in which
the moment condition admits no solution due to the choice of the test function and hence CTD methods such as CTD(0) will not converge.
We now illustrate that CGTD does lead to a solution that is the minimizer of  MSPBE.
\setcounter{eg}{4}
\begin{eg}[Continued]
{\rm Consider the same learning problem in Example \ref{eg:toy 1}  with the parameterized value function $J^{\theta}(t,x) = x + (1-t)e^{\theta x - \frac{1}{2}\theta^2 t}[(\theta + 1)^2 + 1]$. Recall $X_t=W_t$ is a Brownian motion.
%
If we choose the test function to be the constant 1, the projection of $e^{\theta W_t - \frac{1}{2}\theta^2 t}[(\theta + 1)^2 + 1]$ onto the subspace spanned by the test function 1 is $(\theta + 1)^2 + 1$. Theorem \ref{thm:projected msbe} yields that CGTD2 or CTDC algorithms should  minimize $ \E\int_0^1 \big| (\theta + 1)^2 + 1 \big|^2 \dd t $, whose minimizer is $-1$. We implement CTD(0) and CGTD2 (along with ML)  and show the results in Figure \ref{fig:toy 1 wrong para 3}. In our implementation, the initial learning rate for all the three algorithms is 0.01 and decays as $(\sharp\text{episode})^{-0.67}$.
The results confirm our theoretical analysis.}
\end{eg}

%
%

\begin{figure}[h]
\centering
\includegraphics[width=0.5\textwidth]{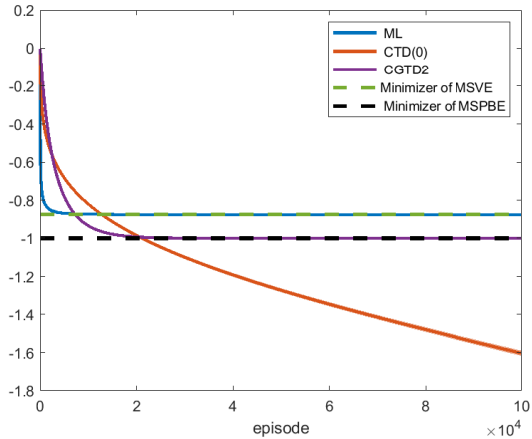}
\caption{\textbf{Minimizing  martingale loss function and CGTD methods converge to the respective minimizers, while CTD(0) diverges, for Example \ref{eg:toy 1 continued wrong para 3}.} Applying SGD to minimize the martingale loss function leads to $\theta^*_{\operatorname{ML}} \approx -0.875301$, CTD(0) does not converge since there is no solution to the moment condition, and CGTD method leads to $\theta^*_{\operatorname{MSPBE}} = -1$, which is equivalent to minimizing the mean-square projected Bellman's error. We repeat the experiment for 100 times to calculate the standard deviations, which are represented as the shaded areas. The width of each shaded area is twice the corresponding standard deviation.}
\label{fig:toy 1 wrong para 3}
\end{figure}

\section{Extensions and Applications}
\label{sec:applications}
In this section, we extend the martingale characterization to the case with an exponential discount factor, and discuss the non-Markovian setting through an example
of a fractional Brownian motion.  Then we present two applications -- a problem with option-like payoff and a linear-quadratic problem in infinite time horizon.
\subsection{Extension to discounted case}
\label{sec:discount}
In many applications the payoff functions involve discounting. We now extend our analysis for PE to such a case. 
Note that, in this case, if we let $T \rightarrow \infty$, then we will have an infinite time horizon problem.

We modify the value function of (\ref{J}) to
\begin{equation}
\label{Jd}
\begin{aligned}
J(t,x) = & \E\left[\int_t^T e^{-\rho (s-t)}r(s,X_s)\dd s + e^{-\rho (T-t)}h(X_T)\Big|X_t= x\right],
\end{aligned}
\end{equation}
where the discount rate $\rho>0$ is known and given.

The PDE (\ref{eq:pde characterization}) is revised to
\begin{equation}
\label{eq:pde characterization discount}
\left\{\begin{array}{l}
\mathcal{L}J (t,x) + r\big(t,x\big) = \rho J(t,x),\; (t,x)\in [0,T)\times \mathbb{R}^d,\\
J(T,x) = h(x),
\end{array}\right.
\end{equation}
and the FBSDE (\ref{eq:fbsde}) becomes
\begin{equation}
\label{eq:fbsded}\left\{\begin{array}{l}
\dd X_s= b(s,X_s)\dd s + \sigma(s,X_s) \dd W_s,\
s\in[t,T]; \;\;X_{t} = x,\\

\dd Y_s = -e^{-\rho (s-t)}r\big(s,X_s\big)\dd s + Z_s\dd W_s,s\in[t,T]; \ Y_T = e^{-\rho (T-t)}h(X_T),
\end{array}\right.
\end{equation}
whereas the relationship (\ref{eq:link}) is now
\begin{equation}
\label{eq:linkd}
Y_s=e^{-\rho (s-t)}J(s,X_s),\;\;Z_{s} =  e^{-\rho (s-t)}\frac{\partial J}{\partial x}(s,X_s)^\top\sigma\big(s,X_s\big),\;\;s\in[t,T].
\end{equation}
Finally
\begin{equation}
\label{eq:md} M_s: = e^{-\rho (s-t)}J(s,X_s) + \int_{t}^s e^{-\rho (s'-t)}r\big(s',X_{s'}\big)\dd s'
\equiv Y_s+\int_{t}^s e^{-\rho (s'-t)}r\big(s',X_{s'}\big)\dd s',\;\;s\in[t,T]
\end{equation}
is a martingale.

Fixing the initial time $t=0$, the above analysis suggests that the martingale loss function should be
\[
\begin{aligned}
\E \int_0^T |M_T - M_t^{\theta}|^2 \dd t  =
&\E \int_0^T \bigg|e^{-\rho T}h(X_T) - e^{-\rho t}J^{\theta}(t,X_{t})+\int_t^Te^{-\rho s}r\big(s,X_s\big)\dd s\bigg|^2 \dd t\\
\approx & \E\bigg[  \sum_{i=0}^{K-1} \bigg( e^{-\rho T}h(X_{T}) - e^{-\rho t_i}J^{\theta}(t_i,X_{t_i}) + \sum_{j=i}^{K-1} e^{-\rho t_j}r(t_{j}, X_{t_j}) \Delta t \bigg)^2 \Delta t \bigg].
\end{aligned}.
\]

On the other hand,  $\dd M^{\theta}_t = e^{-\rho t}\dd J^{\theta}_t - \rho e^{-\rho t}J^{\theta}_t \dd t + e^{-\rho t}r_t\dd t = e^{-\rho t}( \dd J^{\theta}_t + r_t\dd t- \rho J^{\theta}_t \dd t  )$; hence the martingale orthogonality condition   (\ref{eq:martingale difference}) is modified to
\begin{equation}
\label{eq:martingale difference discount}
\E\int_0^{T} \xi_t( \dd J^{\theta}_t + r_t\dd t- \rho J^{\theta}_t \dd t  )  = 0,
\end{equation}
for any test function $\xi\in L^2_{\f}([0,T])$. Note here the discount factor has been absorbed by the test function and thus omitted.

If we set $T = \infty$ in  \eqref{Jd} and assume the payoff does not depend on $t$, then the problem becomes
\[ J(x) = \E\left[\int_0^{\infty} e^{-\rho t} r(X_t) \dd t|X_0 = x\right]. \]
This type of problems occur when the time horizon is sufficiently long or indefinite. Note that in this case the value function does not depend on time explicitly. As a result, there is no longer a terminal condition; instead, it is usually replaced by a growth condition such as $\E[e^{-\rho t}J(X_t)]\to 0$ as $t\to \infty$.

As the martingale loss function requires full trajectories, it may not be directly applicable for the infinite-horizon problem where  we are obviously not able to observe a whole sample until ``infinity". However,  the martingale loss function can still be  defined by truncating at a sufficiently long time $T$ with an artificial terminal condition $e^{-\rho T} h(X_T) = 0$. Therefore, in the episodic setup with repeatedly presented  finite training sets, we can still learn the value function by minimizing the martingale loss function.  However, with a very long horizon, people are more interested in online learning with no reset by observing a single trajectory. As a result, TD learning is a better choice. All the previously discussed  TD learning methods with suitable test functions can be applied based on the conditions \eqref{eq:martingale difference discount}.

\subsection{Extension to non-Markovian setting}
\label{sec:non-markov}
A key assumption of the current paper is that the state process $X$ is Markovian. Indeed, the Markov property determines that the value function $J$, defined through (\ref{J}), is a function of the current state $x$, instead of the whole past history of $X$. However, the martingale perspective may extend beyond the Markovian setting.
While a general non-Markovian PE theory goes beyond the scope of this paper and warrants a full separate research, here we use an example to illustrate.

Recall in Example \ref{eg:toy 1}, the state process $X=W$ is  a Brownian motion, which is Markovnian. Now we consider instead a {\it fractional} Brownian motion $W^H$ with the Hurst index $H$. When $H=\frac{1}{2}$, $W^H$ reduces to a Brownian motion; but when $H\neq \frac{1}{2}$, it is well known to be a non-Markov process. For basic theory and  applications  of fractional Brownian motions, see e.g.  \cite{mandelbrot1968fractional}.

For a non-Markov process, the value function is a functional of the current time $t$ and the entire state trajectory up to $t$. For example,
\[ \E[W_T^H|\f_t] = W_t^H - \int_0^t \Psi_H(T,s|t)\dd W_s^H, \]
where
\[ \Psi_H(T,s|t) = -\frac{\sin\left( \pi(H-\frac{1}{2}) \right)}{\pi}s^{\frac{1}{2} - H}(t-s)^{\frac{1}{2} - H} \int_t^T  \frac{z^{H - \frac{1}{2}} (z-t)^{H - \frac{1}{2}}}{z - s} \dd z; \]
see \cite{sottinen2017prediction}.

For Example \ref{eg:toy 1} with $X_t=W_t$ replaced by $X_t = W^H_t$, the only modification we need is to introduce the value  function $J(t,X_{t\wedge \cdot})$ that is a functional of the  past trajectory with the terminal condition
$J(1,X_{1\wedge \cdot})=\E[X_1|X_{1\wedge \cdot}]=X_1$. 
In our experiment, we use a two-layer fully connected neural network plus an LSTM type of recurrent neural network to approximate such a path-dependent functional satisfying the terminal condition:
\[ J^{\theta}(t, X_{t\wedge \cdot}) = X_t + (1-t) \operatorname{NN}^{\theta}\left( \operatorname{LSTM}^{\theta}(t, X_{t\wedge \cdot}) \right), \]
and then minimize the martingale loss function or apply the CTD(0) algorithm to update the parameters. Here $\operatorname{LSTM}^{\theta}$ maps a sequence of time-series data $(X_{t_0},\cdots,X_{t_k},\cdots,X_{t_K})$ to a sequence of output $(Y_{t_0},\cdots,Y_{t_k},\cdots,X_{t_K})$ recursively where the time-$t_k$ output $Y_{t_k}$  depends on the past trajectory  $(X_{t_0},\cdots,X_{t_k})$. For  details of this network structure, see \citet{hochreiter1997lstm}.

We then compare the mean-square value errors between the learned functional and the ground truth solution with both algorithms. Figure \ref{fig:fbm} shows the trend of convergence  of both the
ML and CTD(0) methods, although the convergence is not as sharp as in the other Markov examples (most likely due to the non-Markovian setting and a fairly general neural network used). We also see that CTD(0) performs better than ML in this case.

\begin{figure}[h]
\centering
\includegraphics[width=0.5\textwidth]{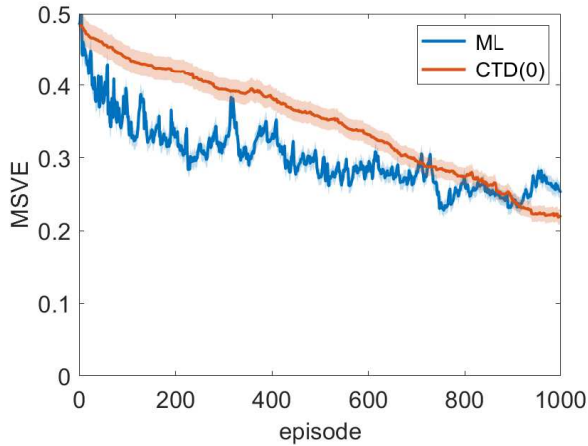}
\caption{\textbf{The mean-square value errors of the learned value functional with ML and CTD(0) algorithm for a fractional Brownian motion with Hurst index $H =0.75$}. The MSVEs are  evaluated with 1000 independent trajectories and standard deviations are computed, which are represented as the shaded areas. The width of each shaded area is twice the corresponding standard deviation.}
\label{fig:fbm}
\end{figure}

\subsection{Option-like payoff }\label{OP}
We apply the theory developed so far to evaluate
\begin{equation}\label{op}
J(t,x) = \E[e^{-r(T-t)}h(X_T)|X_t = x] ,
\end{equation}
where $X=\{X_t,\;0\leq t\leq T\}$ is the state process. This type of evaluation occurs in option pricing in which $X$ is the underlying stock price process and $h$
is the payoff function (usually known and given) at the maturity $T$.

In our simulation, we generate $X$ from
a geometric Brownian motion
$$\frac{\dd X_t}{X_t} = (r-q)\dd t + \sigma \dd W_t,$$
and take $h$ as a call option payoff
$$h(x) = (x - K)^+.$$ Moreover, we set $T = 1, K = 1$, $r=0.01,\ q = 0$, $\sigma = 0.3$. The price process is generated from $X_0 = 1$.

The value function has a theoretical (ground truth) form given by the Black--Scholes formula
\[ J(t,x) = e^{-r(T-t)}[ e^{(r-q)(T-t)}x \Phi(d_+) - K \Phi(d_-) ], \]
where
\[d_{\pm} = \frac{\log(x/K) + (r-q\pm \frac{1}{2}\sigma^2)(T-t)}{\sigma \sqrt{T-t}},  \]
and $\Phi$ is the distribution function of the standard normal.

For learning the value function in our numerical experiment,
we parameterize it by $J^{\theta}(t,x) =  (x - K)^+ + (T-t) \operatorname{NN}^{\theta}(t,x)$, where $\operatorname{NN}^{\theta}$ is a general bivariate neural network  taking both time and space as inputs. This particular form is inspired by that of the payoff function along with the fact that time to maturity, $T-t$, is critical in pricing an option.
We use both ML and online CTD(0) in our experiment, with the martingale loss function being \[ \E\int_0^T \big| e^{-r T}(X_T-K)^+ - e^{-r t}J^{\theta}(t,X_t) \big|^2 \dd t.  \]


In our implementation, we use a simple three-layer fully connected neural network with softplus activation function and 128 and 64 neurons, that is,
\[ \operatorname{NN}^{\theta}(u) = \theta_5 a\bigg( \theta_3 a\big( \theta_1 u + \theta_2\big) + \theta_4\bigg) + \theta_6,\ a(x) = \log(1 + e^x) , \]
where $\theta_1\in \mathbb{R}^{128 \times 2}, \theta_2\in \mathbb{R}^{128 \times 1}$, $\theta_3\in\mathbb{R}^{64 \times 128}, \theta_4\in \mathbb{R}^{64 \times 1}$, $\theta_5\in\mathbb{R}^{64 \times 1},\theta_6\in \mathbb{R}$.

Since now we have hundreds of parameters and the functional forms are complex, instead of comparing the learned parameters, we assess the performance of learning by the following three errors:
\[ J(0,x_0) - J^{\theta}(0,x_0),\;\E\int_0^T \big| J(t,X_t) - J^{\theta}(t,X_t) \big|^2 \dd t , \  \E\int_0^T \big| \frac{\partial J}{\partial x}(t,X_t) - \frac{\partial J^{\theta}}{\partial x}(t,X_t) \big|^2 \dd t ,\]
where $\theta$ is the vector of optimized parameters obtained, and $J$ is the ground truth value function.
In these errors, the first one is in terms of price difference at the initial time $t=0$, and the second one in terms of the averaged total price differences over time. The last one concerns the accuracy in determining $\frac{\partial J}{\partial x}$, the so-called ``Delta" of the option which is the quantity of the underlying stock needed to hedge the option risk.

The PDE (\ref{eq:pde characterization discount}) satisfied by $J$ is nothing else than the  well-known Black-Scholes PDE:
\[ \frac{\partial J}{\partial t} + (r-q)x\frac{\partial J}{\partial x} + \frac{1}{2}\sigma^2x^2\frac{\partial^2 J}{\partial x^2} - rJ = 0,\ J(T,x) =(x-K)^+. \]	
Hence, as discussed earlier, PE can also be regarded as an alternative method to solve such a PDE. This in turn presents a  benchmark in our experiment for comparison purpose, which is the deep learning method in \citet{han2018solving} called the {\it deep BSDE method} for solving PDEs. Note that their method requires the perfect knowledge about the model parameters or, equivalently, not only samples of $\{X_t,0\leq t\leq T\}$ but also samples of $\{ W_t, 0\leq t\leq T\}$ that drives the state process. When implementing the deep BSDE method, we apply a neural network with the same structure to keep the computational cost at the similar level. We use the Adam algorithm for optimization with one trajectory for each episode so that the number of training sample trajectories is also kept the same.\footnote{The Adam algorithm is proposed in \citet{kingma2014adam} and is considered to be an improvement over the vanilla SGD algorithm. We follow \citet{han2018solving} to apply the Adam algorithm. Implementation of neutral networks is through Tensorflow 2. All the  computations are conducted on an Intel(R) Core(TM) i7-1065G7 CPU @ 1.30GHz 1.50 GHz Windows laptop.}

Figure \ref{fig:call option} shows the comparison. For the first two criteria, the errors by the two PE methods developed in this paper, (offline) ML and online CTD(0), both converge to zero very quickly, while it takes some time for those with the BSDE method to be close to zero. For the last criterion, the errors by the PE methods remain close to zero and keeps almost flat from the start, while the BSDE method oscillates dramatically at the start before converging to 0. Indeed, we have shown that minimizing the martingale loss function is equivalent to approximating the value function itself in the mean--square sense, without concerning at all the derivatives of the function. In contrast, the deep BSDE method strives also to learn the derivative term (the $Z_t$ term in FBSDE \eqref{eq:fbsded}) directly and hence requires more knowledge about the system. This example shows that PE methods can be used to learn the function itself effectively but may not provide an accurate approximation to the derivative value. In particular, in terms of estimating the value function,  ML achieves the smallest error and CTD(0) is only slightly behind due to its online setting; but deep BSDE can ultimately learns the derivative. As such, PE methods provide more flexibility for users with tasks such as solving a PDE. It also suggests that for continuous-time RL one should avoid methods relying on the derivatives of the estimated value function.

Finally, we point out that the purpose of this example is to compare our methods with the deep PDE/BSDE method. Because (\ref{op}) holds in the risk-neutral world where data cannot be actually observed, our method cannot be used directly to evaluate an option price. It remains an interesting problem to price options based on physical probability and the real-world data.

\begin{figure}[h]
\centering
\begin{subfigure}{0.32\textwidth}
\centering
\includegraphics[width = 1\textwidth]{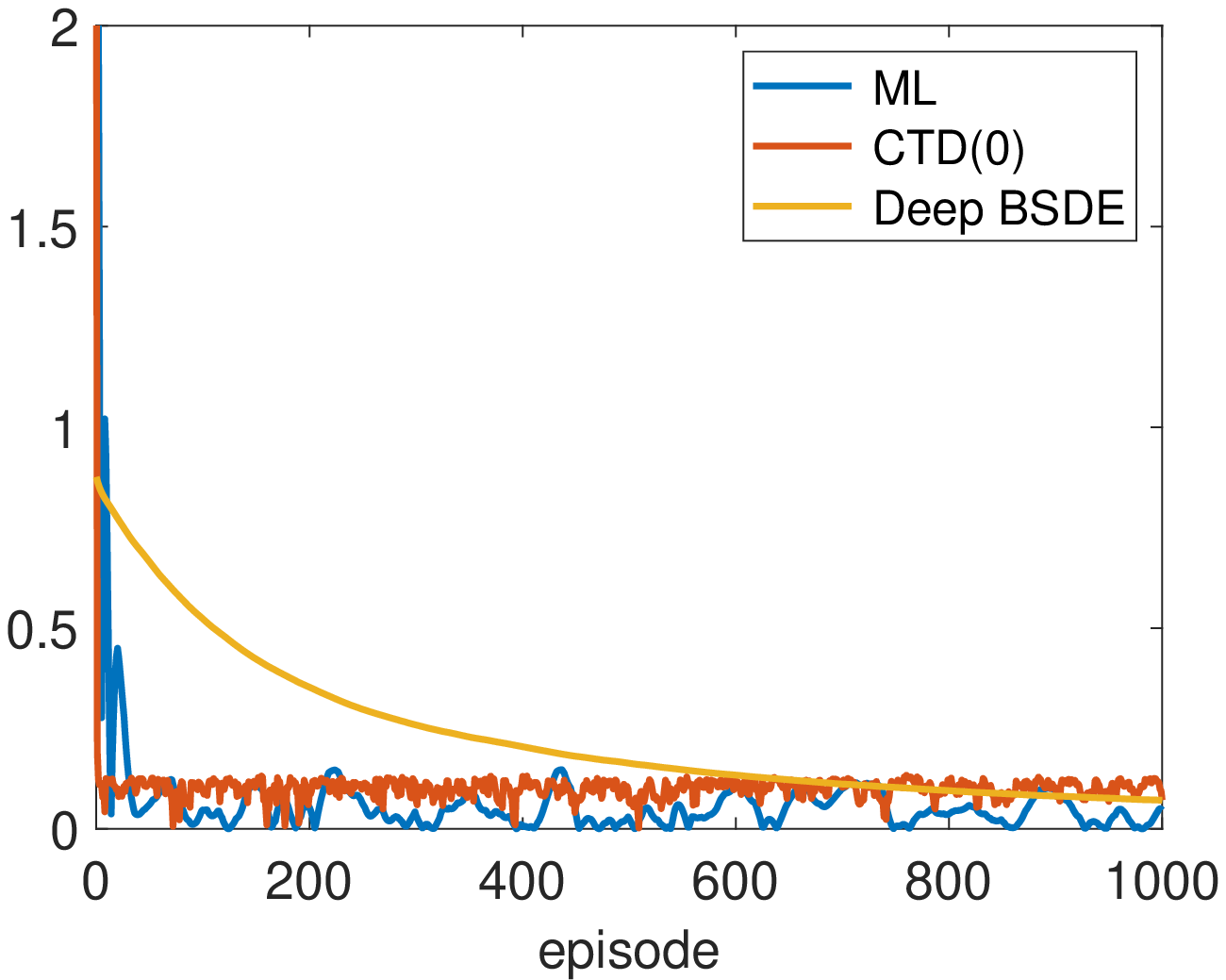}
\end{subfigure}
\begin{subfigure}{0.32\textwidth}
\centering
\includegraphics[width = 1\textwidth]{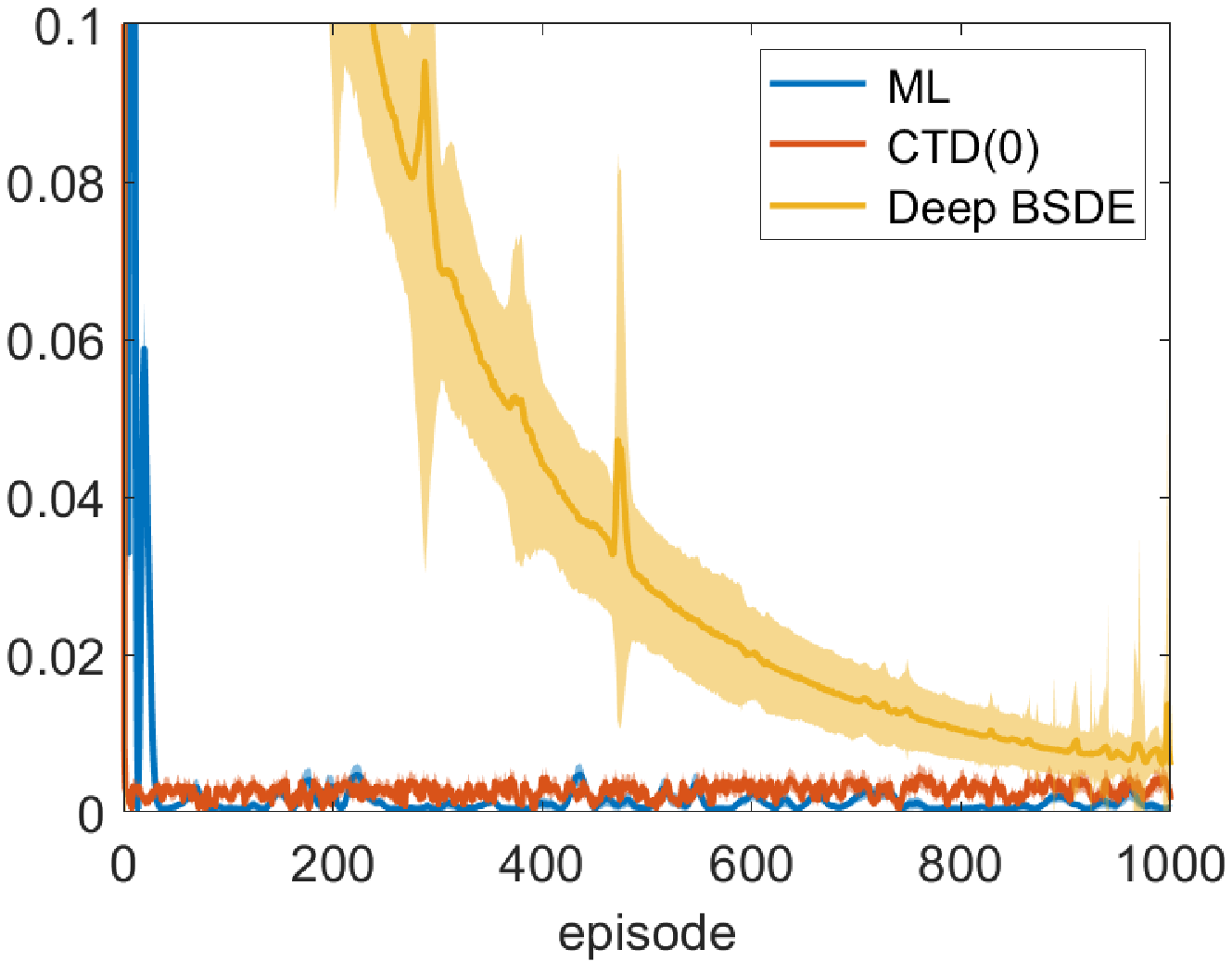}
\end{subfigure}
\begin{subfigure}{0.32\textwidth}
\centering
\includegraphics[width = 1\textwidth]{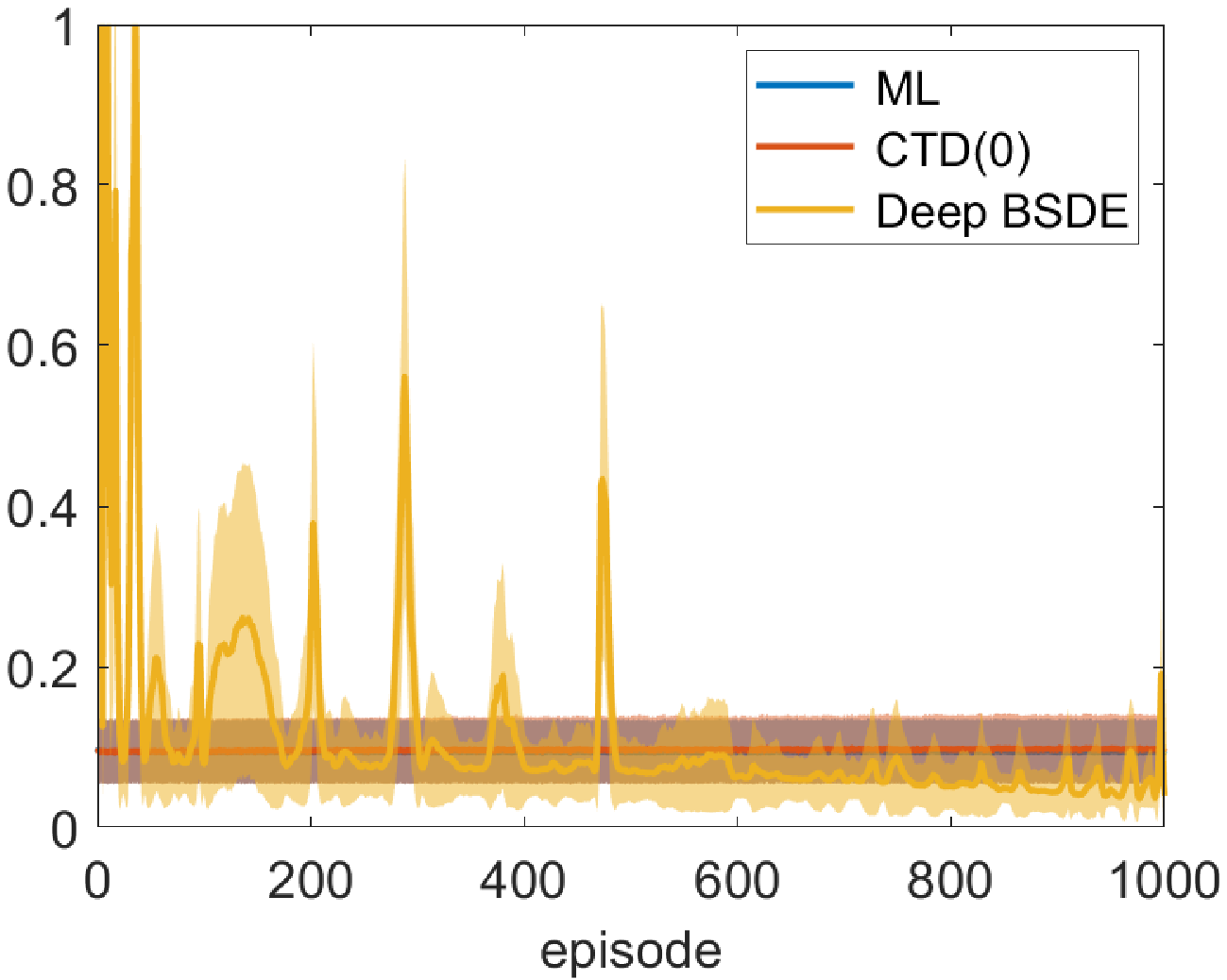}
\end{subfigure}
\caption{{\bf Comparison of learned value functions by ML, online CTD(0) and deep BSDE methods.} From left to right, we show the errors against the true solutions in terms of $|J(0,x_0) - J^{\theta}(0,x_0)|$, $\E\int_0^T \big| J(t,X_t) - J^{\theta}(t,X_t) \big|^2 \dd t$, $\E\int_0^T \big| \frac{\partial J}{\partial x}(t,X_t) - \frac{\partial J^{\theta}}{\partial x}(t,X_t) \big|^2 \dd t$ respectively. These expectations are evaluated using 5000 independent trajectories. Standard deviations are represented as the shaded areas. The width of each shaded area is twice the corresponding standard deviation.}
\label{fig:call option}
\end{figure}

\subsection{Infinite time horizon linear-quadratic problem}
\label{sec:lq}
Consider the following value function 
\[ J(x) = \E\left[\int_0^{\infty} e^{-\rho t} r(X_t) \dd t|X_0 = x\right], \]
where $X=\{X_t,\;0\leq t<\infty\}$ is the state process. In our simulation, we generate $X$ from an Ornstein–Uhlenbeck (OU) process
$$\dd X_t = a(b-X_t)\dd t + \sigma \dd W_t,$$
which is well-known to converge to its unique stationary distribution when $a>0$ and $\sigma>0$. And we take $r$ to be a quadratic function:
$$r(x) = \frac{1}{2}x^2 + q x.$$
This is a discounted linear-quadratic (LQ) control problem in infinite time horizon.

By the standard stochastic control theory  via dynamic programming \citep[Chapter 6]{YZbook} we
can compute the value function  explicitly as  $J(x) = \frac{1}{2}Ax^2+Bx + C$, where
\[ A = \frac{1}{\rho + 2a},\ B = \frac{abA + q}{\rho + a},\ C = \frac{abB + \frac{1}{2}\sigma^2A}{\rho} . \]
We set $a = 1,\ b=1,\  \sigma=0.5,\ \rho=1.5,\ q = 1$, $X_0=0$, and simulate the trajectories up to $T = 2\times 10^5$.
In our experiment, we parameterize the value function by $J^{\theta}(x) = \frac{1}{2}\theta_0 x^2 + \theta_1 x+\theta_2$.
We implement CTD(0), CLSTD(0), and CGTD2 algorithms and report the results in Figure \ref{fig:infinite lq}. Since the parametric functions lie in a linear space, CLSTD(0) explicitly solves the moment conditions, and hence converges fastest. CGTD2 converges faster than CTD(0) because the former is a true gradient-based algorithm.


%
\begin{figure}
\centering
\begin{subfigure}{0.32\textwidth}
\centering
\includegraphics[width = 1\textwidth]{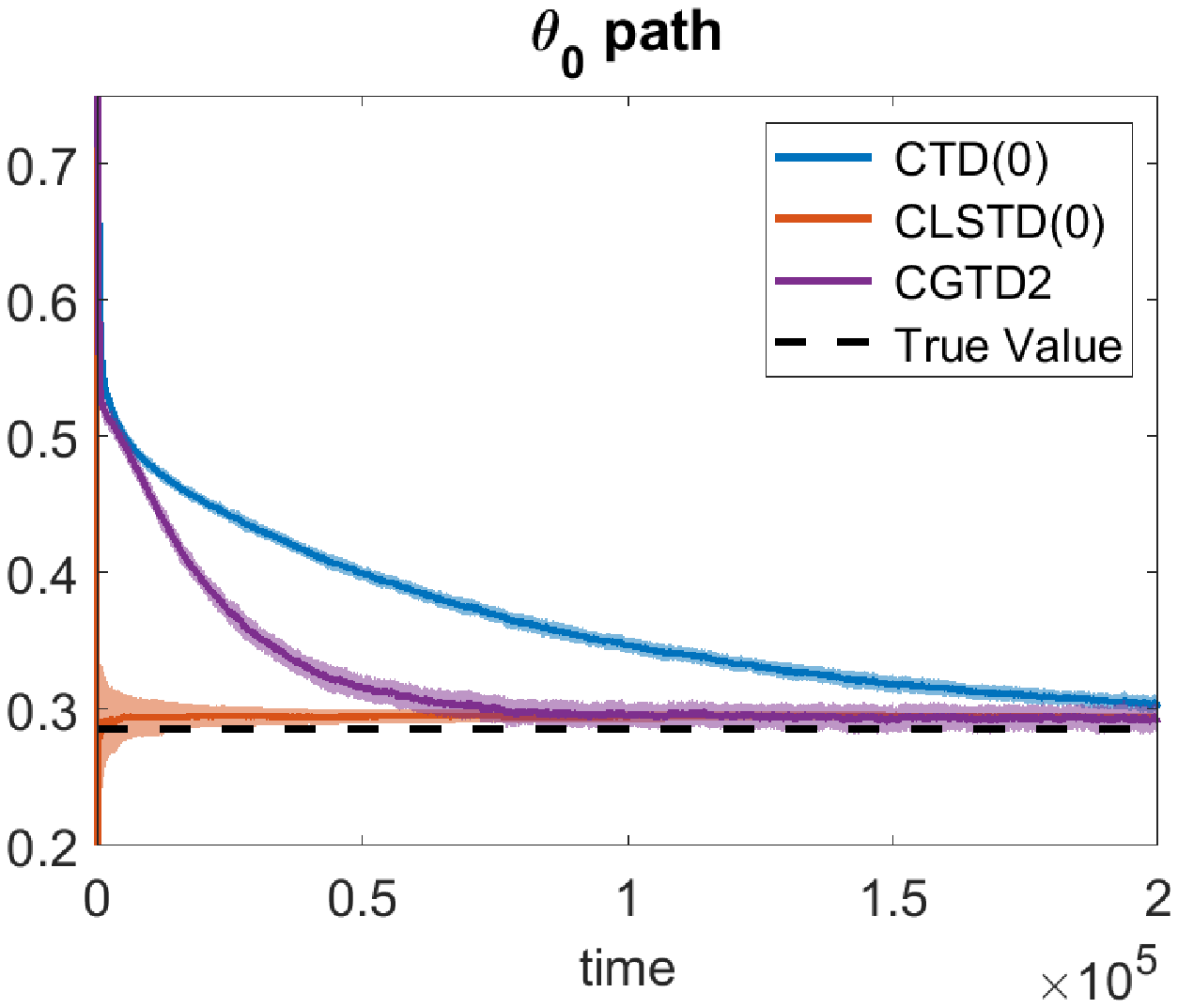}
\end{subfigure}
\begin{subfigure}{0.32\textwidth}
\centering
\includegraphics[width = 1\textwidth]{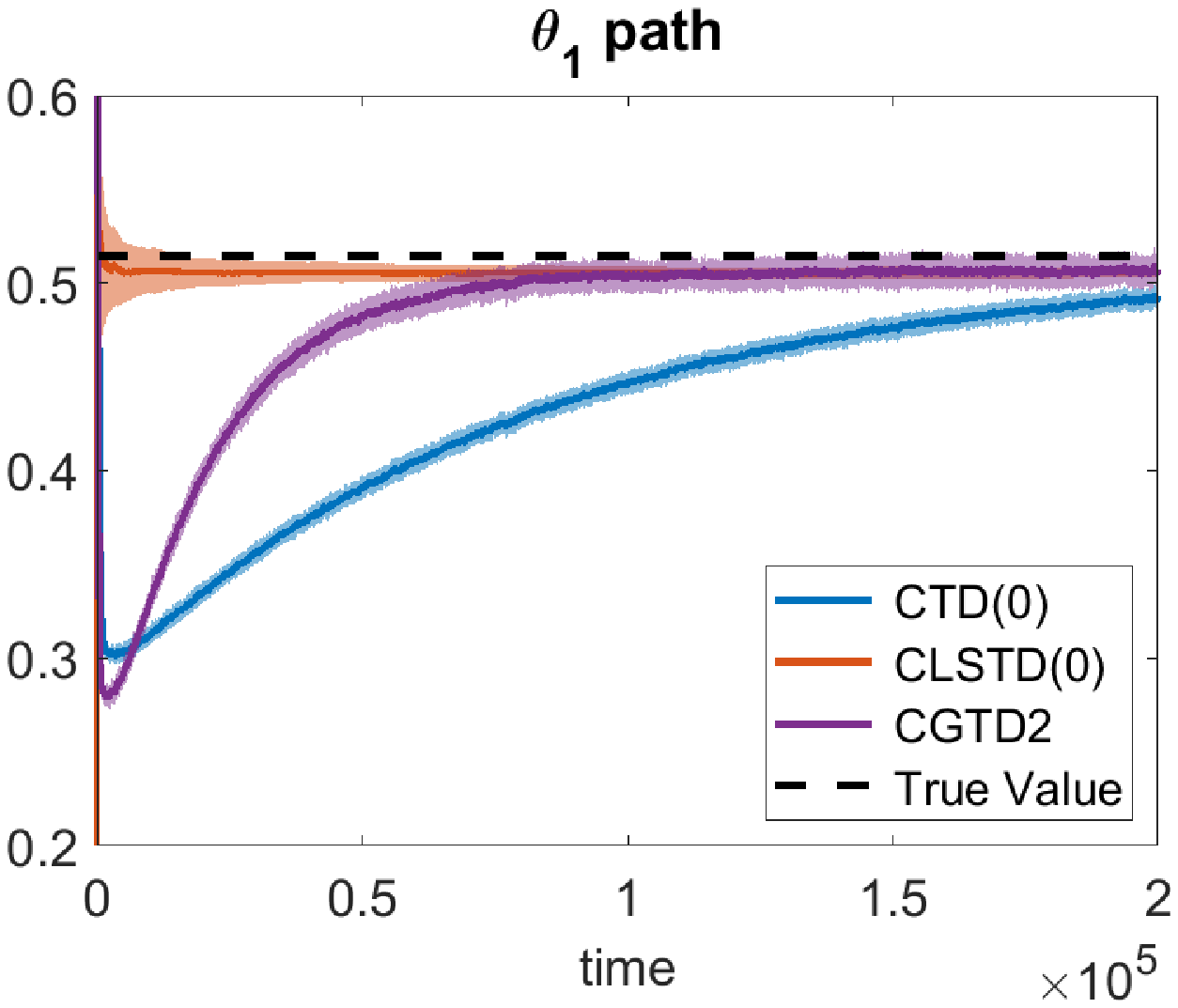}
\end{subfigure}
\begin{subfigure}{0.32\textwidth}
\centering
\includegraphics[width = 1\textwidth]{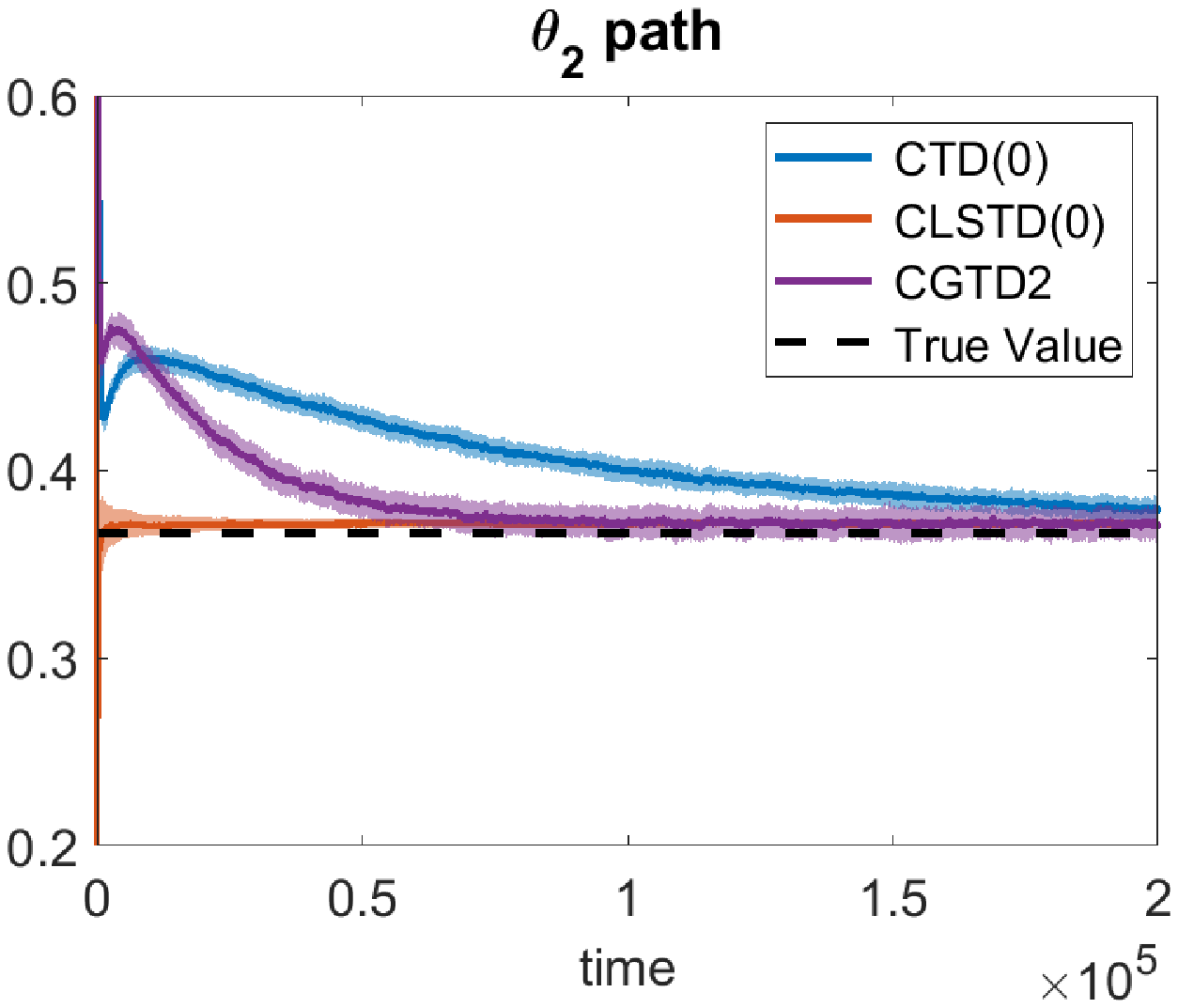}
\end{subfigure}
\caption{{\bf Comparison of learned parameters with different online TD algorithms.} All the algorithms converge to the correct value function. Among them, CLSTD(0) converges the fastest and CTD(0) the slowest. We repeat the experiment for 100 times to calculate the standard deviations, which are represented as the shaded areas. The width of each shaded area is twice the corresponding standard deviation.}
\label{fig:infinite lq}
\end{figure}

\section{Some Algorithmic Aspects}\label{algorithmic}

In this section we discuss two problems from the algorithmic perspective: the choice of test functions and the way to perform function approximation.

\subsection{Choice of test functions}
\label{sec:test function}
One of the most important implications of the martingale viewpoint is the introduction of the test functions.
In Subsection \ref{sec:orthogonality}, we show that  the choice of test functions determines in what sense the true value function is approximated and, hence, a same algorithm with different  test functions may converge differently, as illustrated in Examples \ref{eg:toy 1 continued wrong para} -- \ref{eg:toy 1 continued wrong para 3}. While this characterization remains abstract in theory and provides  little guidance on how to actually select test functions, here we present  a simple example to demonstrate  how test functions may affect the convergence from an algorithmic perspective. 

Consider the same LQ  problem in Subsection \ref{sec:lq}, with $a = 0$,  $\sigma=1$, and $q=0$. 
In this case, $X_t = X_0 + W_t$ now becomes a Brownian motion, which has no stationary distributions.\footnote{Such non-stationary processes are common in practice. Due to the presence of the discount factor, the corresponding LQ problem is still well-posed.}

As before, we parameterize the value function by $J^{\theta} = \frac{1}{2\rho}x^2 + \theta$. This parametric family contains the true value function with $\theta_{\text{true}} = \frac{1}{2\rho^2}$. The conventional choice of the test function in  TD(0) is $\xi_t = \frac{\partial J^{\theta}}{\partial \theta}(X_t) = 1$, leading to the following  updating rule on  $\theta$:
\[  \theta \leftarrow \theta + \alpha[J^{\theta}(X_{t+\Delta t}) - J^{\theta}(X_t) + r(X_t)\Delta t -\rho J^{\theta}(X_t) \Delta t ] .\]
Denote by $\theta_t$ the learned parameter value at time $t$. Then, at the continuous-time limit, $\theta_t$ satisfies an SDE (ignoring the learning rate constant $\alpha$):
\[  \dd \theta_t = \dd J^{\theta_t}(X_t) + r(X_t)\dd t - \rho J^{\theta_t}(X_t)\dd t . \]
By It\^o's lemma, $ \dd J^{\theta_t}(X_t) = \frac{X_t}{\rho}\dd W_t + \frac{1}{2\rho}\dd t$; hence
\[ \dd \theta_t = \left(\frac{1}{2\rho} - \rho \theta_t\right)\dd t + \frac{X_t}{\rho}\dd W_t . \]
Suppose the initial guess of $\theta$ is $\theta_0$.  
Then
$$\E[\theta_t] = \frac{1}{2\rho^2} (2\theta_0 \rho^2e^{-\rho t} +1- e^{-\rho t}) \to \frac{1}{2\rho^2} = \theta_{true} \;\;\mbox{ as } t\to\infty.$$
That is, asymptotically, the conventional choice of the test function indeed leads to an unbiased estimate.
Let us now calculate $\operatorname{Var}(\theta_t)$, the variance of $\theta_t$.
Set $z_t = \theta_t - \frac{1}{2\rho^2} (2\theta_0 \rho^2e^{-\rho t} +1- e^{-\rho t})$, which satisfies the SDE:
\[ \dd z_t = - \rho z_t\dd t + \frac{X_t}{\rho}\dd W_t,\ z_0=0. \]
It\^o's lemma provides
\[ \dd z_t^2 = 2z_t[- \rho z_t\dd t + \frac{X_t}{\rho}\dd W_t] + \frac{X_t^2}{\rho^2}\dd t, \ z_t = 0.  \]
Hence
$$\operatorname{Var}(\theta_t) = \E[z_t^2] = \frac{1}{4\rho^4}\left(e^{-2\rho t} - 1 + \rho t - 2\rho X_0^2 e^{-2\rho t} + 2\rho X_0^2  \right) \to \infty, \mbox{ as }t\to \infty.
$$
So, the conventionally chosen test function does not produce a consistent estimator of $\theta$ due to the blow-up in variance, which in turn is  caused by the non-stationarity of the underlying state process -- a Brownian motion in this example -- whose variance grows linearly in time.

However, this issue can be resolved  by selecting a {\it tailored} test function. Recall the CTD(0) algorithm with a general test function $\xi_t$ updates $\theta$ by
\[  \theta \leftarrow \theta + \alpha\xi_t\left[J^{\theta}(X_{t+\Delta t}) - J^{\theta}(X_t) + r(X_t)\Delta t -\rho J^{\theta}(X_t) \Delta t \right] .\]
Applying the same SDE approximation, we derive
\[ \dd \theta_t = \xi_t\left(\frac{1}{2\rho} - \rho \theta_t\right)\dd t + \xi_t\frac{X_t}{\rho}\dd W_t . \]
Intuitively, to reduce the variance of $\theta_t$, we need to choose a test function that can cancel the growing trend in variance. 
There are many choices to achieve this goal, but a simple one is to take $\xi_t = \frac{1}{|X_t| + 1}$ so that the volatility term above is bounded. We call this a ``tailored choice" of test function for this particular LQ problem. The cost of this variance reduction method is the introduction
of some bias in the mean as some correlation enters into the drift term.

Figure \ref{fig:bmlq} visualizes the result of  a simulation study that confirms our analysis. With the conventional test function $\xi_t=1$, even though the average of the learned parameter values across different experiments tends to be close to the true value, these values become more volatile as time grows larger. On the other hand, with our tailored test function $\xi_t = \frac{1}{|X_t| + 1}$, the variance is reduced dramatically, though the average is slightly off from the true value.

Overall, the study we provide in this subsection shows the promise of our martingale framework in designing more efficient algorithms with suitable choice of test functions, which may at the same time extend the existing literature on RL algorithms for MDPs.
\begin{figure}[h]
\centering
\includegraphics[width=0.5\textwidth]{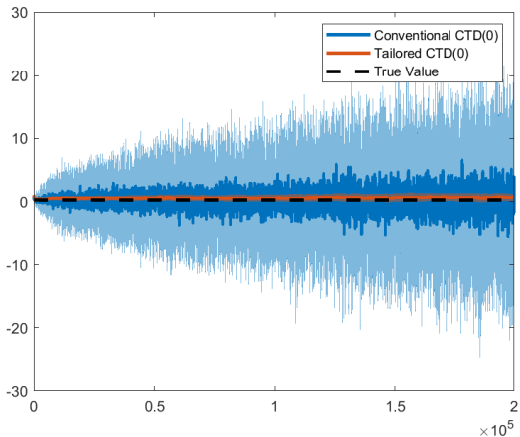}
\caption{\textbf{Comparison of the learned parameters under the conventional test function and the tailored test function}. Conventional CTD(0) refers to the algorithm using test function $\xi_t = \frac{\partial J^{\theta}}{\partial \theta}(X_t) = 1$, and Tailored CTD(0) refers to the one using test function $\xi_t = \frac{1}{|X_t| + 1}$. In the simulation, the problem primitives  are $a=0,\ \sigma=1,\ q = 0,\ \rho = 1.5$, the initial state is $X_0 = 0$ and initial guess of the parameter $\theta$ is $\theta_0 = 1$. The true parameter is $\theta_{\text{true}} = \frac{2}{9} \approx 0.22$. The learning rate is $\alpha=0.1$. We repeat the experiment for 100 times to calculate the standard deviations, which are represented as the shaded areas. The width of each shaded area is twice the corresponding standard deviation.}
\label{fig:bmlq}
\end{figure}

\subsection{Function approximation: global vs sectional}
\label{sec:function approximation}
For a finite-horizon problem, the value function $J(t,x)$ is a {\it bivariate} function of time $t$ and state $x$. Hitherto we have used
a {\it global} approximator $J^{\theta}$ in the sense that we use the same parameter $\theta$ when approximating $J(t,\cdot)$ by $J^{\theta}(t,\cdot)$. Another way of function approximation is {\it sectional}, namely, we approximate $J(t,\cdot)$ by
$J^{\theta_t}(t,\cdot)$ where the parameter $\theta_t$ may be time-varying.
More precisely, let the time discretization be fixed with the grid points $0=t_0 < t_1 <\cdots < t_K = T$, and let $J_0^{\theta_0}( x),\cdots, J_K^{\theta_K}(x)$  approximate the value function at these points, namely, $J_i^{\theta_i}( x) \approx J(t_i,x)$, $i=0,1,\cdots,K$. 

%

To compare these two methods of function approximation, the first thing to note is that the number of parameters to learn grows linearly in the number of time steps with the sectional approximation, while remains the same with the global one.
Hence, the latter has an edge in terms of computational cost when a finer time grid is used. Second, and indeed more importantly, the sectional approximation may become problematic for online learning. 
To see this, suppose we are now at $(t_i,X_{t_i})$ in the online setting.
Applying the idea of the conventional TD(0) algorithm, one can update the parameter $\theta_{i-1}$ by
\[ \theta_{i-1} \leftarrow \theta_{i-1} + \alpha \frac{\partial J_{i-1}^{\theta_{i-1}}}{\partial \theta_{i-1}}(X_{t_{i-1}})\left[ J^{\theta_i}_i(X_{t_i}) -  J^{\theta_{i-1}}_{i-1}(X_{t_{i-1}}) \right] .\]
The question is how to update $\theta_k$ for $k=i,i+1,\cdots, K$ without knowing
the {\it future} states $X_{t_k}$? It seems the best we could do is to update $\theta_k$
according to
%
%
\begin{equation}\label{updatek}
\theta_{k} \leftarrow \theta_{k} + \alpha \frac{\partial J_{k}^{\theta_{k}}}{\partial \theta_{k}}(X_{t_{i-1}})\left[ J^{\theta_i}_i(X_{t_i}) -  J^{\theta_{i-1}}_{i-1}(X_{t_{i-1}}) \right] .
\end{equation}
This form is less intuitive because we use the current and past states to update
parameters for future value functions. In contrast, the global parameterization views the value function as a whole; hence a temporal advancement  naturally leads to an update of  the whole bivariate function, including a prediction into the future as well as an updated evaluation of the past.

Finally, we run a simulation for Example \ref{eg:toy 1} to compare the learning results of the two function approximation approaches, both in offline learning (using martingale loss function) and online learning (using CTD(0) for the global approximation and (\ref{updatek}) for the sectional one). Recall that the ground truth is $J(t,x) = x$, and we have used the global approximation with $J^{\theta}(t,x) = [\theta(1-t)+1]x$. For the sectional approximation, we consider a simple form of $J_i^{\theta_i}(x) = \theta_i x$, with unknown parameters $\theta_0,\cdots,\theta_{K-1}$ while it is known that $\theta_K = 1$ based on the terminal condition.

We evaluate the performance of the different approximation approaches by MSVE as defined in \eqref{eq:value l2 error}. For the global approximation, this error is
\[ \E\int_0^T |J(t,X_t) - J^{\theta}(t,X_t)|^2\dd t = \E\int_0^T \theta^2(1-t)^2 W_t^2\dd t = \frac{1}{12}\theta^2 .   \]
For the sectional approximation, this error is calculated by
\[ \E\sum_{i=0}^{K-1}|J(t_i,X_{t_i}) - J_i^{\theta_i}(X_{t_i})|^2\Delta t  = \sum_{i=0}^{K-1}(\theta_i-1)^2 t_{i} \Delta t . \]

The results are presented in Figure \ref{fig:sliced}. For this simple example, the number of unknown parameters in the sectional approach is small so the difference in computational cost is insignificant. Otherwise, we observe that the global approximation performs similarly as the sectional one  in the offline setting (the ML method), but significantly better in the online setting (the CTD(0) method).


\begin{figure}[h]
\centering
\includegraphics[width=0.5\textwidth]{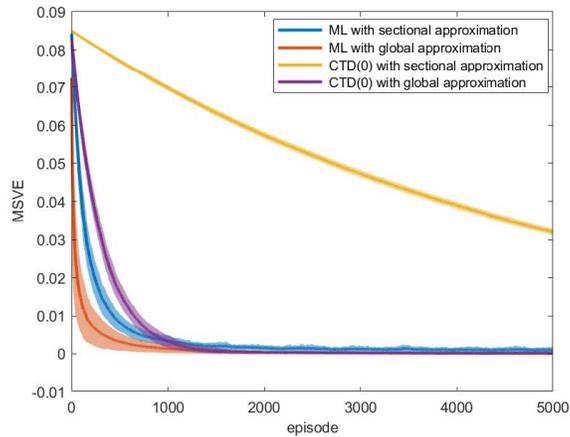}
\caption{\textbf{Comparison of the mean-square value errors of the learned value function using globel and sectional  approximation methods   with ML and CTD(0) algorithms}. The initial guess is $\theta = -1$ for global approximation and $\theta_i = t_i$, $0\leq i\leq K-1$,  for sectional approximation so that the two methods are initialized to be the same function. The learning rate is $\alpha=0.01$. We repeat the experiment for 100 times to calculate the standard deviations, which are represented as the shaded areas. The width of each shaded area is twice the corresponding standard deviation.}
\label{fig:sliced}
\end{figure}

\section{Conclusions}
\label{sec:conclusion}
In this paper, we provide a unified theoretical framework for studying PE in RL with continuous time and space. The  theory is premised upon the observation that PE is equivalent to enforcing
the martingality of a stochastic process. Many existing and popular PE algorithms, which somewhat scatter around in the MDP literature, can find a common ground through this ``martingale lens". These algorithms can be modified for solving PE in the continuous setting for actual implementation.

The martingale perspective is potentially useful for studying many important problems related to PE that have been well developed in the discrete setting but remain open in the continuous setting, including off-policy evaluation, state-action value estimation, and convergence analysis. Furthermore, it
may inspire new questions  that have not been posed by traditional RL research. For example, how to ``optimally" choose test functions and how their choices affect the convergence rate in both discrete and continuous settings.

Finally, PE is formulated for It\^o processes in this paper,  mainly because such a process has convenient and	well-studied  properties and can reasonably model many real-life dynamics. The martingale view, however, is generalizable beyond It\^o processes such as jump diffusions, non-Markov processes and semi-martingales.

\section*{Acknowledgement}
We are grateful for comments from the seminar participants at University of Southern California, Boston University, Imperial College, University of Connecticut,
the International Seminar on SDEs and Related Topics and the Joint Seminar by AIFT and Columbia University, and from the participants at the 6th Berlin Workshop for Young Researchers in Mathematical Finance.
We thank Jerome Detemple, Steven Kou, Moris Strub, Wenpin Tang, 
Renyuan Xu and Jianfeng Zhang for helpful discussions and comments on the paper. We are also indebted to the Action Editor and three anonymous referees for constructive comments which have led to an improved version of the paper.
Zhou gratefully acknowledges financial support through a start-up grant and the Nie Center for Intelligent Asset Management at Columbia University. 

\begin{appendix}

\section*{Appendix A:  A Summary of Popular PE Methods}
\label{appendix:table}
The following Table \ref{tab:summary algorithms} summarizes popular PE methods and algorithms, and the interpretations we have discovered in this paper in terms of objectives (loss/error functions to be minimized or equations to be solved) and limiting points of convergent algorithms.
\begin{table}[h]
\centering
\begin{minipage}[c]{1\textwidth}
	\centering
	\begin{tabular}{ccccc}
		\toprule
		Method & \tabincell{c}{Representative\\algorithms} & Online &  \tabincell{c}{Objective} & \tabincell{c}{Converging\\ point} \\
		\midrule
		\tabincell{c}{Monte Carlo\footnote{\cite{sutton2011reinforcement}.}} &  \tabincell{c}{gradient\\Monte Carlo} & No & \tabincell{c}{minimize\\ martingale\\ loss function} & \tabincell{c}{minimizers of\\ mean-square value\\ function error} \\
		&  &  &   &  \\
		\tabincell{c}{Residual\\ gradient\footnote{\cite{baird1995residual}.}} &  \tabincell{c}{na\"ive\\residual\\ gradient} & Yes  & \tabincell{c}{minimize\\ mean-square\\ TD error} &  \tabincell{c}{minimizers of\\ quadratic variation} \\
		&  &  &    &  \\
		\tabincell{c}{Semi-gradient\\ TD learning\footnote{ \cite{sutton1988learning,bradtke1996linear}. This terminology is taken from \citet[Chapter 9]{sutton2011reinforcement}.}} &  \tabincell{c}{TD($\lambda$)\\ LSTD($\lambda$)} & Yes  & \tabincell{c}{solve\\moment conditions}   & \tabincell{c}{zeros to\\ moment conditions} \\
		&  &  &   &  \\
		\tabincell{c}{Gradient\\ TD learning\footnote{ \cite{sutton2008convergent,sutton2009fast}.} } &  \tabincell{c}{GTD(0)\\ GTD2\\ TDC} & Yes  & \tabincell{c}{minimize\\ quadratic form\\ of moment\\ conditions} & \tabincell{c}{minimizers of\\ mean-square\\ projected\\ Bellman error} \\
		\bottomrule
	\end{tabular}%
\end{minipage}
\caption{\textbf{Summary of popular PE methods in RL literature.} The table summarizes different PE methods. The first three columns indicate the names of the methods, those of the representative algorithms, and whether applicable  online and/or offline. The last two columns reveal the objectives  and the converging  points of the corresponding  algorithms.}
\label{tab:summary algorithms}%
\end{table}%

\section*{Appendix B:  Stochastic Control Formulation of RL}
\label{appendix:control}
Let $d,m,n$ be given positive integers, $T>0$,  and $b: [0,T]\times \mathbb{R}^d\times \mathbb{R}^n \mapsto \mathbb{R}^d$ and $\sigma:
[0,T]\times \mathbb{R}^d\times \mathbb{R}^n\mapsto \mathbb{R}^{d\times m}$ be  given
functions.
A stochastic control problem is to control the {\it state} (or {\it feature}) dynamic governed by  an SDE:
\begin{equation}
\label{eq:model classical2 control}
\dd X_s = b\big( s,X_s,\bm{u}(s, X_s) \big)\dd t + \sigma\big( s,X_s,\bm{u}(s, X_s) \big) \dd W_s,\;s\in[0,T],
\end{equation}
where $\bm{u}:(t,x)\in [0,T] \times \mathbb{R}^d \mapsto \bm{u}(t,x)\in \mathcal{U}$ is a given (measurable) feedback control policy, with $\mathcal{U}\subseteq \mathbb{R}^n$ being  the {\it action space}  representing the constraints on
an agent's decisions ({\it controls} or {\it actions}).

Given a policy $\bm{u}$ and an initial time--state pair $(t,x)\in [0,T]\times \mathbb{R}^d$, let $\{X_s^{t,x,\bm{u}}, t\leq s\leq T\}$ be the solution to (\ref{eq:model classical2 control}) with $X_t=x$.
The \textit{value function} under the policy $\bm{u}$  is
\begin{equation}
\label{J control}
\begin{aligned}
	J(t,x;\bm{u}) = & \E\left[\int_t^T r(s,X_s^{t,x,\bm{u}},\bm{u}(s, X_s^{t,x,\bm{u}}))\dd s + h(X_T^{t,x,\bm{u}})\Big|X_t^{t,x,\bm{u}}= x\right],
\end{aligned}
\end{equation}
where $r: [0,T]\times \mathbb{R}^d\times \mathbb{R}^n \mapsto \mathbb{R}$ and $h:
\mathbb{R}^d\mapsto \mathbb{R}$ are given reward functions.

A policy $\bm{u}$ is called admissible if (\ref{eq:model classical2 control}) has a unique weak solution and (\ref{J control}) is finite for any $(t,x)\in [0,T]\times \mathbb{R}^d$.
A typical RL problem is to maximize (minimize) $J(t,x;\bm{u})$ over all  admissible policies $\bm{u}$. In the classical (model-based) stochastic control literature, the functional forms of $b,\sigma,r$ and $h$ are known, and there are well-developed theories to solve the problem; see, e.g., \citet{YZbook,fleming2006controlled}. However, in the RL context, these functional forms  are typically unknown, although in some applications that of $h$ may be known because it may be interpreted as a given target the agent specifies (e.g. in option pricing $h$ is the payoff function of an option, which is typically given and known; see Subsection \ref{OP}).


The PE task as a part of the general RL problem is, for a given policy $\bm{u}$, to devise a numerical procedure to find $J(t,x;\bm{u})$ as a function of $(t,x)$ using multiple sample trajectories of the process $\{ s,X_s^{t,x,\bm{u}},r\big( s,X_s^{t,x,\bm{u}},\bm{u}(t,X_s^{t,x,\bm{u}}) \big)\}_{t\leq s\leq T}$,
without the knowledge of the model parameters (the functional forms of $b,\sigma,r,h$).


If we suppress $\bm{u}$, which is fixed in PE, then we recover the formulation \eqref{eq:model classical}--\eqref{J}. 
Note that the formulation also covers the ``exploratory" setting of \citet{wang2020reinforcement} in which the admissible control policies are probability-distribution-valued, because the value function therein is of the same form as \eqref{J control} under a fixed distributional control policy.

\section*{Appendix C:  Martingale in Discrete-time Markov Reward Processes}
\label{appendix:mrp}

We show that there is also a martingale property in the classical discrete-time RL MDP formulation. To be consistent with the main setting of this paper, we consider only the finite horizon episodic tasks; the infinite horizon continuing tasks can be studied similarly.

Let $X=\{ X_t, t=0,1,\cdots,T \}$ be a discrete-time Markov process adapted to
$\{ \f_t\}_{t=0,1,\cdots,T }$ in a filtered probability space $(\Omega,\p,\f,\{ \f_t\}_{t=0,1,\cdots,T })$. One is interested in finding the value function $v$ defined by
\[ v(t,x) = \E\left[\sum_{s=t}^{T-1} r(s, X_s) + h(X_T)\Big| X_t = x\right] ,\]
where $r(t,x)$ is the expected reward at time $t$ conditioned on being at state $x$,
and $h$ is the final reward.

When the state space is finite and discrete, $X$ is  referred to as a Markov reward process (MRP) or alternatively as an MDP with a fixed policy. When the state space is infinite or typically continuous, it is usually called a semi-MRP or a semi-MDP.


Set $M_t = v(t, X_t) + \sum_{s=0}^{t-1} r(s,X_s)$ with $M_T = h(X_T) + \sum_{s=0}^{T-1} r(s,X_s)$. Then, for any $t=0,1,\cdots,T-1$, by Markov property, we obtain
\[ \begin{aligned}
\E[M_{t+1}|\f_t] = & \sum_{s=0}^{t}r(s,X_s) + \E[v(t+1,X_{t+1})|\f_t] \\
= & \sum_{s=0}^{t-1}r(s,X_s) + \E[r(t,X_t) + v(t+1,X_{t+1})|\f_t] \\
= & \sum_{s=0}^{t-1}r(s,X_s) + v(t,X_t) = M_t,
\end{aligned}\]
where the last equality is due to
\[ \E[r(t,X_t) + v(t+1,X_{t+1})|\f_t] = \E[r(t,X_t) + v(t+1,X_{t+1})|X_t] = v(t,X_t), \]
which is the well-known Bellman equation for a discrete-time MRP.

So, $M$ being a martingale is equivalent to the value function satisfying the Bellman equation, which in turn can be used to characterize PE. From this martingale perspective, we can develop parallel approaches such as the martingale loss function and the martingale orthogonality condition that will recover  various conventional PE algorithms for discete-time MRPs.

\section*{Appendix D:  Proofs of Statements}
\label{appendix:proof}
\subsection*{Proof of Proposition \ref{proposition:martingale bsde}}

\begin{proof}
To show $M$ is a martingale, observe that based on \eqref{J}, we have
\[ M_s =  \E\left[\int_s^T  r\big(s',X_{s'}\big)\dd s' + h(X_T) |X_s\right] +\int_{t}^s r\big(s',X_{s'}\big)\dd s' = \E[ M_T | \f_s] ,\]
where we have used the Markov property of the process $\{X_s,t\leq s\leq T\}$.
This establishes that $M$ is a martingale.

Conversely, if $\tilde{M}$ is a martingale, then $ \tilde{M}_s = \E[\tilde{M}_T|\f_s] $, which  is equivalent to
\[ \begin{array}{rl}
	\tilde{J}(s,X_s) = &\E\left[ \int_s^T r\big(s',X_{s'}\big)\dd s' +\tilde{J}(T,X_T) |\f_s\right]\\
	= &\E\left[\int_s^T  r\big(s',X_{s'}\big)\dd s' +h(X_T) |\f_s\right]\\
	= &{J}(s,X_s), \;\; s\in [t,T].
\end{array}
\]
Letting $s=t$, we conclude $\tilde{J}(t,x) = J(t,x)$.

%
\end{proof}

\subsection*{Proof of Theorem \ref{thm:squared td minimizer}}
We first present two lemmas that will be useful for the proof of Theorem \ref{thm:squared td minimizer} and also other theorems later.
\begin{lemma}
\label{lemma:minimizer limit}
Let $f_{h}(x) = f(x) + r_h(x)$, where $f$ is a continuous function and $r_h$ converges to 0 uniformly on any compact set 
as $h\to 0$.
\begin{enumerate}[(a)]
	\item Suppose $x_h^* \in \arg\min_{x} f_h(x) \neq \emptyset$ and $\lim_{h\to 0}x_h^* =x^*$. Then $x^*\in \arg\min_{x} f(x)$. Moreover, if there exists $\alpha>0$ such that $|r_h(x)| \leq C h^{\alpha}$ for some constant $C$, then $|f(x_h^*)  - f(x^*)| \leq 2 C h^{\alpha} $.
	\item Suppose $f_h(x_h^*) = 0$ and $\lim_{h\to 0}x_h^* =x^*$. Then $f(x^*) = 0$. Moreover, if there exists $\alpha>0$ such that $|r_h(x)| \leq C h^{\alpha}$ for some constant $C$, then $|f(x_h^*)| \leq C h^{\alpha} $.
\end{enumerate}
\end{lemma}
\begin{proof}
\begin{enumerate}[(a)]
	\item For any $y$, we have $f(x_h^*) + r_h(x_h^*) = f_h(x_h^*) \leq f_h(y)$. The sequence $\{x_h^*\}$ forms a compact set; hence $r_h(x_h^*) \to 0$ as $h\to 0$. Letting $h\to 0$, since $x_h^*\to x^*$ and $f$ is continuous, we obtain $f(x^*) \leq f(y)$. Since $y$ is arbitrary, $x^*\in \arg\min_{x} f(x)$.
	
	Moreover, we have
	\[ 0\leq f(x_h^*) - f(x^*) = f_h(x_h^*) -r_h(x_h^*)- f_h(x^*) + r_h(x^*) \leq -r_h(x_h^*) + r_h(x^*) \leq 2Ch^{\alpha}.\]
	\item Since $f(x_h^*) + r_h(x_h^*) = f_h(x_h^*) = 0$, $|f(x_h^*)| = |r_h(x_h^*)|$. The sequence $\{x_h^*\}$ forms a compact set; hence $r_h(x_h^*) \to 0$ as $h\to 0$. Letting $h\to 0$, since $x_h^*\to x^*$ and $f$ is continuous, we obtain $|f(x^*)|= 0$.
	
	The second statement is straightforward since $|f(x_h^*)| = |r_h(x_h^*)| \leq Ch^{\alpha}$.
\end{enumerate}

\end{proof}

\begin{lemma}
\label{lemma:simple process approximation}
Under Assumptions \ref{ass:sde regularity} and \ref{ass:growth in reward}, we have
\[ \E\bigg[\int_t^{t+h} \left|r(s,X_s) - r(t,X_t)\right|^2 \dd s\bigg] \leq C h^{2\mu_1 +\mu_2 +1} . \]
\end{lemma}
\begin{proof}
By Assumption \ref{ass:growth in reward}, for $s\in [t,t+h]$, we have
\[ \left|r(s,X_s) - r(t,X_t)\right|^2 \leq Ch^{2\mu_1}|X_s - X_t|^{2\mu_2} (|X_s|^{2\mu_3} + |X_t|^{2\mu_3})  . \]
When $\mu_2 > 0$, we take $p > 1$ sufficiently large such that $2\mu_2 p \geq 1$, and $q > 1$ such that $\frac{1}{p} + \frac{1}{q} = 1$.

Under Assumption \ref{ass:sde regularity}, we have the usual moment estimate of the solution to an SDE, e.g., \citet[Chapter 1, Theorem 6.1]{YZbook}. Together with  H\"older's inequality, we have
\[ \begin{aligned}
	& \E\bigg[\int_t^{t+h}|X_s - X_t|^{2\mu_2} (|X_s|^{2\mu_3} + |X_t|^{2\mu_3})  \dd s  \bigg] \\
	\leq & \left( \E\bigg[ \int_t^{t+h}|X_s - X_t|^{2\mu_2 p} \dd s  \bigg] \right)^{1/p} \left( \E\bigg[ \int_t^{t+h}(|X_s|^{2\mu_3} + |X_t|^{2\mu_3})^{q} \dd s\bigg] \right)^{1/q} \\
	\leq & C \left( \E\bigg[ \int_t^{t+h}|X_s - X_t|^{2\mu_2 p} \dd s  \bigg] \right)^{1/p} \left( h \max_{t\leq s \leq t+h} \E\left[|X_s|^{2\mu_3q}\right]  \right)^{1/q} \\
	\leq & C\left( \int_t^{t+h} (s-t)^{\mu_2 p} \dd s \right)^{1/p} h^{1/q} \\
	\leq & C h^{(\mu_2p + 1)/p} h^{1/q} = C h^{\mu_2 + 1}.
\end{aligned}\]
When $\mu_2=0$, the above inequality also holds true as $\E\bigg[\int_t^{t+h}|X_s - X_t|^{2\mu_2} (|X_s|^{2\mu_3} + |X_t|^{2\mu_3})  \dd s  \bigg] \leq Ch$.
\end{proof}

Now we are already to prove Theorem \ref{thm:squared td minimizer}.

\smallskip

\begin{proof}
By It\^o's lemma, we have
\[
\begin{aligned}
	& \sum_{i=0}^{K-1} \bigg( \frac{J^{\theta}(t_{t+1}, X_{t_{i+1}}) - J^{\theta}(t_{i}, X_{t_{i}})}{t_{i+1} - t_{i}} + r_{t_{i}} \bigg)^2 \Delta t\\
	= & \frac{1}{ \Delta t} \sum_{i=0}^{K-1} \bigg( J^{\theta}(t_{i+1}, X_{t_{i+1}}) - J^{\theta}(t_{i}, X_{t_{i}}) + \int_{t_{i}}^{t_{i+1}} r_{t_i} \dd s  \bigg)^2 \\
	= & \frac{1}{ \Delta t} \sum_{i=0}^{K-1} \bigg(  \int_{t_{i}}^{t_{i+1}} [\mathcal{L}J^{\theta}(s,X_s)+ r_{t_i}] \dd s +  \int_{t_{i}}^{t_{i+1}} \left(\frac{\partial J^{\theta}}{\partial x}\right)^\top \sigma (s,X_s) \dd W_s  \bigg)^2\\
	= & \frac{1}{ \Delta t} \sum_{i=0}^{K-1}\Bigg\{ \bigg(  \int_{t_{i}}^{t_{i+1}} [\mathcal{L}J^{\theta}(s,X_s)+ r_{t_i}] \dd s \bigg)^2 + \bigg( \int_{t_{i}}^{t_{i+1}} \left(\frac{\partial J^{\theta}}{\partial x}\right)^\top  \sigma (s,X_s) \dd W_s   \bigg)^2 \\
	& +2 \bigg(  \int_{t_{i}}^{t_{i+1}} [\mathcal{L}J^{\theta}(s,X_s)+ r_{t_i}] \dd s \bigg) \bigg( \int_{t_{i}}^{t_{i+1}} \left(\frac{\partial J^{\theta}}{\partial x}\right)^\top \sigma (s,X_s) \dd W_s\bigg) \Bigg\} .
\end{aligned}
\]
It\^o's isometry implies
\[ \E\bigg[ \bigg( \int_{t_{i}}^{t_{i+1}} \left(\frac{\partial J^{\theta}}{\partial x}\right)^\top \sigma (s,X_s) \dd W_s   \bigg)^2 \bigg] = \E \int_{t_{i}}^{t_{i+1}} \Big| \left(\frac{\partial J^{\theta}}{\partial x}\right)^\top  \sigma (s,X_s) \Big|^2\dd s   .\]
Thus,
\[
\begin{aligned}
	\operatorname{MSTDE}_{\Delta t}(\theta) = & \frac{1}{\Delta t}\E \int_{0}^{T} \big| \left(\frac{\partial J^{\theta}}{\partial x}\right)^\top  \sigma (s,X_s) \big|^2\dd s    + \frac{1}{\Delta t}\sum_{i=0}^{K-1}\E\bigg[ \bigg(  \int_{t_{i}}^{t_{i+1}} [\mathcal{L}J^{\theta}(s,X_s)+ r_{t_i}] \dd s \bigg)^2 \bigg] \\
	& + \frac{2}{\Delta t}\sum_{i=0}^{K-1}\E\bigg[ \bigg(  \int_{t_{i}}^{t_{i+1}} [\mathcal{L}J^{\theta}(s,X_s)+ r_{t_i}] \dd s \bigg) \bigg( \int_{t_{i}}^{t_{i+1}} \left(\frac{\partial J^{\theta}}{\partial x}\right)^\top  \sigma (s,X_s) \dd W_s\bigg)\bigg].
\end{aligned} \]
We write $ \operatorname{MSTDE}_{\Delta t}(\theta) \Delta t = \operatorname{QV}(\theta) + R(\theta)$, where $$\operatorname{QV}(\theta) := \E\int_{0}^{T} \big| \left(\frac{\partial J^{\theta}}{\partial x}\right)^\top  \sigma (s,X_s) \big|^2\dd s  $$ and
\[ \begin{aligned}
	R(\theta) := & \sum_{i=0}^{K-1}\E\bigg[ \bigg(  \int_{t_{i}}^{t_{i+1}} [\mathcal{L}J^{\theta}(s,X_s)+ r_{t_i}] \dd s \bigg)^2 \\
	& + 2\bigg(  \int_{t_{i}}^{t_{i+1}} [\mathcal{L}J^{\theta}(s,X_s)+ r_{t_i}] \dd s \bigg) \bigg( \int_{t_{i}}^{t_{i+1}} \left(\frac{\partial J^{\theta}}{\partial x}\right)^\top  \sigma (s,X_s) \dd W_s\bigg) \bigg] .
\end{aligned} \]
We apply Cauchy-Schwarz inequality and obtain
\[ \begin{aligned}
	|R(\theta)|\leq  & \sum_{i=0}^{K-1}\E  \int_{t_{i}}^{t_{i+1}} [\mathcal{L}J^{\theta}(s,X_s)+ r_{t_i}]^2  \dd s (\Delta t)^2  \\
	& + 2\sum_{i=0}^{K-1}\Bigg\{ \E\bigg[ \bigg(  \int_{t_{i}}^{t_{i+1}} [\mathcal{L}J^{\theta}(s,X_s)+ r_{t_i}] \dd s \bigg)^2\bigg]  \E\bigg[ \bigg( \int_{t_{i}}^{t_{i+1}} \left(\frac{\partial J^{\theta}}{\partial x}\right)^\top \sigma (s,X_s) \dd W_s\bigg)^2 \bigg] \Bigg\}^{1/2} \\
	\leq & \sum_{i=0}^{K-1}\E  \int_{t_{i}}^{t_{i+1}} [\mathcal{L}J^{\theta}(s,X_s)+ r_{t_i}]^2  \dd s (\Delta t)^2  \\
	& + 2\sum_{i=0}^{K-1}\Bigg\{ \left[\E \int_{t_{i}}^{t_{i+1}} [\mathcal{L}J^{\theta}(s,X_s)+ r_{t_i}]^2 \dd s (\Delta t)^2\right]  \left[\E \int_{t_{i}}^{t_{i+1}} \big| \left(\frac{\partial J^{\theta}}{\partial x}\right)^\top  \sigma (s,X_s) \big|^2\dd s \right]\Bigg\}^{1/2} \\
	= & (\Delta t)^2 \E  \int_{0}^{T} [\mathcal{L}J^{\theta}(s,X_s)+ \bar{r}_{s}]^2  \dd s \\
	& + 2 \Delta t \sum_{i=0}^{K-1}\Bigg\{ \left[\E  \int_{t_{i}}^{t_{i+1}} [\mathcal{L}J^{\theta}(s,X_s)+ \bar{r}_{s}]^2 \dd s\right]  \left[ \E \int_{t_{i}}^{t_{i+1}} \big| \left(\frac{\partial J^{\theta}}{\partial x}\right)^\top  \sigma (s,X_s) \big|^2\dd s \right]\Bigg\}^{1/2} \\
	\leq & (\Delta t)^2 \E \int_{0}^{T} [\mathcal{L}J^{\theta}(s,X_s)+ \bar{r}_{s}]^2  \dd s \\
	& + 2 \Delta t \Bigg\{ \sum_{i=0}^{K-1}\E  \int_{t_{i}}^{t_{i+1}} [\mathcal{L}J^{\theta}(s,X_s)+ \bar{r}_{s}]^2 \dd s  \Bigg\}^{1/2}\Bigg\{ \sum_{i=0}^{K-1}\E \int_{t_{i}}^{t_{i+1}} \big| \left(\frac{\partial J^{\theta}}{\partial x}\right)^\top  \sigma (s,X_s) \big|^2\dd s \Bigg\}^{1/2} \\
	= & (\Delta t)^2 \E \int_{0}^{T} [\mathcal{L}J^{\theta}(s,X_s)+ \bar{r}_{s}]^2  \dd s + 2\Delta t \Bigg\{ \E  \int_{0}^{T} [\mathcal{L}J^{\theta}(s,X_s)+ \bar{r}_{s}]^2  \dd s  \Bigg\}^{1/2} \sqrt{\operatorname{QV}(\theta)},
\end{aligned}  \]
where $\bar{r}_s := r_{t_i}$ for the unique $i$ such that $t_i \leq s < t_{i+1}$.

It follows from the triangle inequality that
\[\Bigg\{ \E  \int_{0}^{T} [\mathcal{L}J^{\theta}(s,X_s)+ \bar{r}_{s}]^2  \dd s \Bigg\}^{1/2} = ||\mathcal{L}J^{\theta}(\cdot,X_{\cdot})+ \bar{r}_{\cdot} ||_{L^2} \leq ||\mathcal{L}J^{\theta}(\cdot,X_{\cdot})+ r_{\cdot} ||_{L^2} + || r - \bar{r} ||_{L^2} .  \]
Hence,
\[\begin{aligned}
	|R(\theta)| \leq & (\Delta t)^2 \big( ||\mathcal{L}J^{\theta}(\cdot,X_{\cdot})+ r_{\cdot} ||_{L^2} + || r - \bar{r} ||_{L^2} \big)^2 + 2\Delta t \big( ||\mathcal{L}J^{\theta}(\cdot,X_{\cdot})+ r_{\cdot} ||_{L^2} + || r - \bar{r} ||_{L^2} \big)\sqrt{QV(\theta)}\\
	\leq & 4(\Delta t)^2 \big( 2\operatorname{MSBE}(\theta) + || r - \bar{r} ||_{L^2}^2 \big) + 2\Delta t  \sqrt{\operatorname{QV}(\theta)} \big( \sqrt{2 \operatorname{MSBE}(\theta)} +  || r - \bar{r} ||_{L^2}  \big),
\end{aligned}  \]
where $\operatorname{MSBE}(\theta) = ||\mathcal{L}J^{\theta}(\cdot,X_{\cdot})+ r_{\cdot} ||_{L^2}^2$ is the \textit{mean-square Bellman error}.

Because $\bar{r}$ is a simple process approximating $r$, we have $|| r - \bar{r} ||_{L^2} \to 0$ as $\Delta t\to 0$, which is independent of $\theta$. For an arbitrary compact set $\Gamma$, Assumption \ref{ass:regularity} yields that  $\operatorname{MSBE}(\theta)$ and $\operatorname{QV}(\theta)$ are both continuous functions of $\theta$; hence $\sup_{\theta \in \Gamma}\operatorname{MSBE}(\theta)$ and $\sup_{\theta \in \Gamma}\operatorname{QV}(\theta)$ are both finite. Consequently,
\[ \begin{aligned}
	\sup_{\theta \in \Gamma}|R(\theta)| \leq & 4(\Delta t)^2 \big( 2\sup_{\theta \in \Gamma}\operatorname{MSBE}(\theta) + || r - \bar{r} ||_{L^2}^2 \big) \\
	& + 2\Delta t  \sqrt{\sup_{\theta \in \Gamma}\operatorname{QV}(\theta)} \big( \sqrt{2 \sup_{\theta \in \Gamma}\operatorname{MSBE}(\theta)} +  || r - \bar{r} ||_{L^2}  \big) \to 0 \mbox{ as } \Delta t\to 0.
\end{aligned} \]
The desired result now follows from Lemma \ref{lemma:minimizer limit}.

Moreover, under Assumption \ref{ass:growth in reward}, it follows from Lemma \ref{lemma:simple process approximation} that
\[ ||r - \bar{r}||_{L^2}^2 = \sum_{i=0}^{K-1} \E\int_{t_i}^{t_{i+1}}\left(r(s,X_s) - r(t_i,X_{t_i})\right)^2 \dd s \leq K (\Delta t)^{2\mu_1 + \mu_2+1} = (\Delta t)^{2\mu_1 + \mu_2}. \]
Therefore, our analysis above yields
\[
\sup_{\theta \in \Gamma}|R(\theta)| \leq C_1\Delta t + C_2 (\Delta t)^{2} + C_3(\Delta t)^{\mu_1 + \mu_2/2 + 1} +C_4 (\Delta t)^{2\mu_1 + \mu_2  +2},
\]
where the leading term in the right hand side is $O(\Delta t)$. The desired result again follows from Lemma \ref{lemma:minimizer limit}.

\end{proof}

\subsection*{Proof of Theorem \ref{thm:martingale loss function}}
\begin{proof}
Since $M_t = J(t,X_t) + \int_0^t r(s,X_s)\dd s$ is a martingale, we have
\[
\begin{aligned}
	ML(\theta) = & \E \int_0^T |M_T - M^{\theta}_t|^2 \dd t\\
	= & \E \int_0^T (M_T - M_t + M_t - M^{\theta}_t)^2 \dd t \\
	= & \E\int_0^T [(M_T - M_t)^2 + (M_t - M^{\theta}_t)^2 + 2(M_T - M_t)(M_t - M^{\theta}_t) ]\dd t  \\
	= & \E\int_0^T (M_T - M_t)^2 \dd t + \E\int_0^T (M_t - M^{\theta}_t)^2\dd t + 2\int_0^T \E\left((M_t - M^{\theta}_t) \E[(M_T - M_t) | \f_t] \right) \dd t \\
	= & \E \int_0^T | J(t,X_t) - J^{\theta}(t,X_t) |^2 \dd t + \E\int_0^T |M_T - M_t|^2 \dd t.
\end{aligned}
\]
The second term does not rely on $\theta$. This proves our first statement.

Next, let us estimate the difference between the continuous-time and the discretized martingale loss functions. Denote
\[ m(t,\theta) = \E[(M_T - M^{\theta}_t)^2] = \E[\big( h(X_T) - J^{\theta}(t,X_t) + \int_t^T r_s\dd s\big)^2],\]
and
\[  \Delta \tilde{M}^{\theta}_{t_i} = h(X_T) - J^{\theta}(t_i,X_{t_i}) + \sum_{j=i}^{K-1}r(t_j,X_{t_{j}})\Delta t = h(X_T) - J^{\theta}(t,X_t) + \int_t^T \bar{r}_s\dd s ,\]
where $\bar{r}_s := r_{t_i}$ for the unique $i$ such that $t_i \leq s < t_{i+1}$.

Then
\[\begin{aligned}
	\operatorname{ML}(\theta) - \operatorname{ML}_{\Delta t}(\theta) = & \int_0^T m(t,\theta)\dd t - \sum_{i=0}^{K-1} m(t_i,\theta) \Delta t + \sum_{i=0}^{K-1} m(t_i,\theta) \Delta t - ML_{\Delta t}(\theta) \\
	= & \sum_{i=0}^{K-1} \int_{t_i}^{t_{i+1}} [m(t,\theta) - m(t_i,\theta)] \dd t + \Delta t \sum_{i=0}^{K-1}\E[(M_T - M^{\theta}_{t_i})^2 - (\Delta \tilde{M}^{\theta}_{t_i})^2] .
\end{aligned}  \]
The first term is bounded by
\[ \left| \sum_{i=0}^{K-1} \int_{t_i}^{t_{i+1}} [m(t,\theta) - m(t_i,\theta)] \dd t\right| \leq \sum_{i=0}^{K-1} \sup_{t \in [0,T]}|\frac{\partial m}{\partial t}(t,\theta)|\int_{t_i}^{t_{i+1}} (t-t_i) \dd t = \frac{T}{2}\sup_{t \in [0,T]}|\frac{\partial m}{\partial t}(t,\theta)| \Delta t . \]
To estimate the second term, recall that
\[ M_T - M^{\theta}_{t_i} - \Delta \tilde{M}^{\theta}_{t_i}  = \int_{t_i}^T (r_s - \bar{r}_s)\dd s .\]
Hence
\[
\begin{aligned}
	\big| \E[(M_T - M^{\theta}_{t_i})^2 - (\Delta \tilde{M}^{\theta}_{t_i})^2] \big| = & \left| \E[ 2\int_{t_i}^T (r_s - \bar{r}_s)\dd s (M_T - M^{\theta}_{t_i})   - \big( \int_{t_i}^T (r_s - \bar{r}_s)\dd s  \big)^2] \right| \\
	\leq & 2m(t_i,\theta)^{\frac{1}{2}}(\E[(\int_{t_i}^T (r_s - \bar{r}_s)\dd s)^2])^{\frac{1}{2}} + \E[(\int_{t_i}^T (r_s - \bar{r}_s)\dd s)^2] \\
	\leq & 2T\sup_{t\in [0,T]}|m(t,\theta)|^{\frac{1}{2}} ||r - \bar{r}||_{L^2} + T^2||r - \bar{r}||_{L^2}^2 .
\end{aligned}\]
Therefore, we have proved
\[ |\operatorname{ML}(\theta) - \operatorname{ML}_{\Delta t}(\theta)| \leq \frac{T}{2}\sup_{t \in [0,T]}|\frac{\partial m}{\partial t}(t,\theta)| \Delta t + 2T^2  \sup_{t\in [0,T]}|m(t,\theta)|^{\frac{1}{2}} ||r - \bar{r}||_{L^2} + T^3||r - \bar{r}||_{L^2}^2 . \]
For an arbitrary compact set $\Gamma$, under Assumption \ref{ass:regularity}, $\sup_{t \in [0,T],\theta\in \Gamma}|\frac{\partial m}{\partial t}(t,\theta)| + \sup_{t\in [0,T],\theta \in \Gamma}|m(t,\theta)| < \infty$, and $||r - \bar{r}||_{L^2} \to 0$. Hence as $\Delta t \to 0$,
\[ \sup_{\theta\in \Gamma}|\operatorname{ML}(\theta) - \operatorname{ML}_{\Delta t}(\theta)| \to 0. \]
By Lemma \ref{lemma:minimizer limit}, we obtain the desired conclusion.

Moreover, under Assumption \ref{ass:growth in reward}, it follows from Lemma \ref{lemma:simple process approximation} and the proof of Theorem \ref{thm:squared td minimizer} that
\[
\sup_{\theta\in \Gamma}|\operatorname{ML}(\theta) - \operatorname{ML}_{\Delta t}(\theta)|  \leq C_1\Delta t + C_2(\Delta t)^{\mu_1 + \mu_2/2} + C_3 (\Delta t)^{2\mu_1 + \mu_2},
\]
where the leading term in the right hand side is $O\left( (\Delta t)^{\min\{ 1, \mu_1 + \frac{\mu_2}{2} \}} \right)$. The desired result again follows from Lemma \ref{lemma:minimizer limit}.

\end{proof}

\subsection*{Proof of Proposition \ref{lemma:martingale characterization test function}}
\begin{proof}
The ``only if" part is evident. To prove the ``if" part, assume that $\dd M^{\theta}_t = A_t\dd t + B_t \dd W_t$. In particular, in our case, $A_t = \mathcal{L} J^{\theta}(t, X_t) + r_t$ and $B_t = (\frac{\partial J^{\theta}}{\partial x})^\top\sigma (t,X_t)$. $A,B\in L^2_{\f}([0,T])$ follows from Assumption \ref{ass:regularity}. For any $0\leq s<s' \leq T$, take $\xi_t = sgn(A_t)$ if $t\in [s,s']$ and $\xi_t=0$ otherwise. Then
\[ 0= \E\int_{s}^{s'} \xi_t \dd M^{\theta}_t =\E\int_{s}^{s'}\left( |A_t| \dd t + \xi_t B_t\dd W_t\right) = \E\int_{s}^{s'} |A_t| \dd t, \]
where the expectation of the second term vanishes because $|\xi B| \leq |B| \in L^2_{\f}([0,T])$ and hence $\E\int_0^{\cdot} \xi_t B_t\dd W_t$ is a martingale. This yields $A_t = 0$ almost surely, and thus  $ M^{\theta}$ is a martingale.
\end{proof}

\subsection*{Proof of Theorem \ref{thm:td equation}}
\begin{proof}
Based on Lemma \ref{lemma:minimizer limit}, it suffices to examine the difference
\[
\begin{aligned}
	& \big|  \E\int_0^T \xi_t\dd M^{\theta}_t - \E\sum_{i=0}^{K-1}\xi_{t_i}(M^{\theta}_{t_{i+1}} - M^{\theta}_{t_i}) \big| =\big|  \E\sum_{i=0}^{K-1}\int_{t_i}^{t_{i+1}}(\xi_t - \xi_{t_i}) \dd M^{\theta}_t \big| + \big|  \E\sum_{i=0}^{K-1}\xi_{t_i}\int_{t_i}^{t_{i+1}}(r_s - r_{t_{i}}) \dd s  \big| \\
	\leq & \E\sum_{i=0}^{K-1}\int_{t_i}^{t_{i+1}}|\xi_t - \xi_{t_i}|\cdot | \mathcal{L}J^{\theta}(t,X_t) + r_t | \dd t + \E\bigg[ \left(\sum_{i=0}^{K-1}\xi_{t_i}^2\right)^{1/2} \left(\sum_{i=0}^{K-1}\left( \int_{t_i}^{t_{i+1}}(r_s - r_{t_{i}}) \dd s \right)^2   \right)^{1/2} \bigg] \\
	\leq & \sum_{i=0}^{K-1}\big( \E\int_{t_i}^{t_{i+1}}|\xi_t - \xi_s|^2\dd t \big)^{1/2}
	\bigg( \E\int_{t_i}^{t_{i+1}}\big( \mathcal{L}J^{\theta}(t,X_t) + r_t \big)^2\dd t \bigg)^{1/2} \\
	& \E\bigg[ \left(\sum_{i=0}^{K-1}\xi_{t_i}^2\right)^{1/2} \left(\sum_{i=0}^{K-1}\left( \int_{t_i}^{t_{i+1}}(r_s - r_{t_{i}})^2 \dd s \right)\Delta t   \right)^{1/2} \bigg]  \\
	\leq & \sum_{i=0}^{K-1}\bigg( \E\int_{t_i}^{t_{i+1}}\big( \mathcal{L}J^{\theta}(t,X_t) + r_t \big)^2\dd t \bigg)^{1/2}\bigg(\int_{t_i}^{t_{i+1}}C(t-t_{i})^{\alpha}\dd t  \bigg)^{1/2} + (\Delta t)^{1/2} ||r - \bar{r}||_{L^2} \left( \sum_{i=0}^{K-1}\E[\xi_{t_i}^2] \right)^{1/2} \\
	\leq & \sum_{i=0}^{K-1}\bigg( \E\int_{t_i}^{t_{i+1}}\big( \mathcal{L}J^{\theta}(t,X_t) + r_t \big)^2\dd t \bigg)^{1/2}\sqrt{\frac{C}{1+\alpha}}(\Delta t)^{\frac{1+\alpha}{2}} +  ||r - \bar{r}||_{L^2} ||\bar{\xi}||_{L^2}\\
	\leq & \bigg(\sum_{i=0}^{K-1} \E\int_{t_i}^{t_{i+1}}\big( \mathcal{L}J^{\theta}(t,X_t) + r_t \big)^2\dd t \bigg)^{1/2} K^{1/2}\sqrt{\frac{C}{1+\alpha}}(\Delta t)^{\frac{1+\alpha}{2}} + ||\bar{\xi}||_{L^2} (\Delta t)^{\mu_1 + \mu_2/2} \\
	\leq & || \mathcal{L}J^{\theta}(\cdot,X_{\cdot}) + r_{\cdot} ||_{L^2}\sqrt{\frac{C T}{1+\alpha}}(\Delta t)^{\frac{\alpha}{2}}+ ||\bar{\xi}||_{L^2} (\Delta t)^{\mu_1 + \mu_2/2}.
\end{aligned}
\]
Hence, for an arbitrary compact set $\Gamma$, under Assumption \ref{ass:regularity}, we have
\[\begin{aligned}
	& \sup_{\theta\in \Gamma}\big|  \E\int_0^T \xi_t\dd M^{\theta}_t - \E\sum_{i=0}^{K-1}\xi_{t_i}(M^{\theta}_{t_{i+1}} - M^{\theta}_{t_i}) \big| \\
	\leq  & \sup_{\theta\in \Gamma}|| \mathcal{L}J^{\theta}(\cdot,X_{\cdot}) + r_{\cdot} ||_{L^2}\sqrt{\frac{C T}{1+\alpha}}(\Delta t)^{\frac{\alpha}{2}} + \sup_{\theta\in \Gamma}||\bar{\xi}||_{L^2} (\Delta t)^{\mu_1 + \mu_2/2}  \to 0,
\end{aligned}  \]
as $\Delta t\to 0$.

Since the leading term above is $O\left((\Delta t)^{\min\{\alpha/2,\  \mu_1+\mu_2/2\}} \right)$, we obtain the convergence rate in view of Lemma \ref{lemma:minimizer limit}.
\end{proof}

\subsection*{Proof of Theorem \ref{thm:gmm objective}}
We first prove an error estimate of the following form:
\[
\begin{aligned}
& \big| (b+\Delta b)^\top (D + \Delta D) (b+\Delta b) - b^\top D b \big| \\
= & \big| D \circ [ (b+\Delta b)(b+\Delta b)^\top - bb^\top ]  + \Delta D \circ (b+\Delta b)(b+\Delta b)^\top \big| \\
\leq & |D| |(b+\Delta b)(b+\Delta b)^\top - bb^\top | + |\Delta D| |(b+\Delta b)(b+\Delta b)^\top| \\
= & |D| |\Delta b \Delta b^\top + b\Delta b^\top + \Delta b b^\top| + |\Delta D| |b+\Delta b|^2 \\
= & |D||\Delta b|^2 + 2|D||b||\Delta b| + 2 |\Delta D| |b|^2 + 2|\Delta D||\Delta b|^2 .
\end{aligned}
\]
Based on the proof of Theorem \ref{thm:td equation}, we have that for an arbitrary compact set $\Gamma$,
\[\begin{aligned}
& \sup_{\theta\in \Gamma}\left|  \E\int_0^T \xi_t\dd M^{\theta}_t - \E\sum_{i=0}^{K-1}\xi_{t_i}(M^{\theta}_{t_{i+1}} - M^{\theta}_{t_i}) \right| \\
& \leq \sup_{\theta\in \Gamma}|| \mathcal{L}J^{\theta}(\cdot,X_{\cdot}) + r_{\cdot} ||_{L^2}\sqrt{\frac{C T}{1+\alpha}}(\Delta t)^{\frac{\alpha}{2}} + \sup_{\theta\in \Gamma}||\bar{\xi}||_{L^2} (\Delta t)^{\mu_1 + \mu_2/2}  \to 0.
\end{aligned}  \]
Given that $|A_{\Delta t} - A| \leq \tilde{C}(\theta)|\Delta t|^{\beta} $, we get
\[ \sup_{\theta \in \Gamma}|\operatorname{GMM}_{\Delta t}(\theta) - \operatorname{GMM}(\theta)| \to 0, \]
as $\Delta t \to 0$. By Lemma \ref{lemma:minimizer limit}, we obtain the desired results.

Moreover, based on the error estimate of the quadratic form, we obtain
\[\begin{aligned}
& \sup_{\theta \in \Gamma}|\operatorname{GMM}_{\Delta t}(\theta) - \operatorname{GMM}(\theta)| \\
\leq &  C\left[ O\left( (\Delta t)^{\alpha/2} \right) + O\left( (\Delta t)^{\mu_1+\mu_2/2} \right) + O\left( (\Delta t)^{\beta} \right) + o\left( (\Delta t)^{\alpha/2} + (\Delta t)^{\mu_1+\mu_2/2} + (\Delta t)^{\beta}  \right)  \right],
\end{aligned} \]
where the leading term is $O\left( (\Delta t)^{\min\{\alpha/2,\ \mu_1+\mu_2/2,\ \beta\}}\right) $.

In particular, when $A = \left[\E\int_0^T \xi_t^{\theta} (\xi_t^{\theta})^\top \dd t\right]^{-1} $ and $A_{\Delta t} = \left[\E\sum_{i=0}^{K-1} \xi_{t_i}^{\theta} (\xi_{t_i}^{\theta})^\top \Delta t\right]^{-1} $, we claim the condition $|A_{\Delta t} - A| \leq \tilde{C}(\theta)|\Delta t|^{\beta} $ holds true.
To see this, recall that
\[ |(D + \Delta D)^{-1} - D^{-1}|  = |\sum_{k=0}^{\infty}(D^{-1}\Delta D)^{k} D^{-1} - D^{-1}|\leq  \sum_{k=0}^{\infty} |D^{-1}\Delta D| |(D^{-1}\Delta D)^{k} D^{-1}| \leq \frac{|D^{-1}|^2 |\Delta D|}{1 - |D^{-1}||\Delta D|} . \]
Thus, it suffices to estimate the difference \[\begin{aligned}
& \big| \E\int_0^T \xi_t^{\theta} (\xi_t^{\theta})^\top \dd t - \E\sum_{i=0}^{K-1} \xi_{t_i}^{\theta} (\xi_{t_i}^{\theta})^\top \Delta t \big| =\big| \E\sum_{i=0}^{K-1}\int_{t_i}^{t_{i+1}} [\xi_t^{\theta} (\xi_t^{\theta})^\top - \xi_{t_i}^{\theta} (\xi_{t_i}^{\theta})^\top] \dd t  \big| \\
\leq & \sum_{i=0}^{K-1}\int_{t_i}^{t_{i+1}}\E[\xi_t^{\theta} (\xi_t^{\theta})^\top - \xi_{t_i}^{\theta} (\xi_{t_i}^{\theta})^\top] \dd t \\
\leq & \sum_{i=0}^{K-1}\int_{t_i}^{t_{i+1}}\E[ |\xi_t^{\theta} - \xi_{t_i}^{\theta}|^2 + 2 |\xi_t^{\theta} - \xi_{t_i}^{\theta}||\xi_t^{\theta}| ]\dd t \\
\leq & \sum_{i=0}^{K-1}\int_{t_i}^{t_{i+1}}C(\theta)(t - t_{i})^{\alpha}\dd t + 2\sum_{i=0}^{K-1}\int_{t_i}^{t_{i+1}} \E[|\xi_t^{\theta} - \xi_{t_i}^{\theta}||\xi_t^{\theta}|]\dd t \\
\leq & \frac{T C(\theta)}{1+\alpha}(\Delta t)^{\alpha} + 2\sum_{i=0}^{K-1} \big( \E\int_{t_i}^{t_{i+1}}|\xi_t^{\theta} - \xi_{t_i}^{\theta}|^2 \dd t\big)^{1/2}\big( \E\int_{t_i}^{t_{i+1}}|\xi_t^{\theta}|^2 \dd t\big)^{1/2} \\
\leq & \frac{T C(\theta)}{1+\alpha}(\Delta t)^{\alpha} + 2\sqrt{\frac{C(\theta)}{1+\alpha}}(\Delta t)^{\frac{1}{2} + \frac{\alpha}{2}}\sum_{i=0}^{K-1}\big(\E\int_{t_i}^{t_{i+1}}|\xi_t^{\theta}|^2 \dd t\big)^{1/2} \\
\leq & \frac{T C(\theta)}{1+\alpha}(\Delta t)^{\alpha} + 2\sqrt{\frac{TC(\theta)}{1+\alpha}}(\Delta t)^{\frac{\alpha}{2}}||\xi^{\theta}||_{L^2} .
\end{aligned}\]

\subsection*{Proof of Theorem \ref{thm:projected msbe}}

\begin{proof} Denote by $\langle \kappa, \tilde{\kappa} \rangle_{L^2}: = \E\int_0^T \kappa_t\tilde{\kappa}_t \dd t$ the  inner product in  $L^2_{\f}([0,T])$.
It follows from the property of projection that $\langle \kappa - \Pi_{\theta}\kappa, \xi^{\theta,(j)} \rangle_{L^2} = 0$ for any $\kappa \in L^2_{\f}([0,T])$ and all $j=1,\cdots,L'$.

As a stochastic process, $\mathcal{L} J^{\theta}(\cdot, X_{\cdot}) + r_{\cdot} \in L^2_{\f}([0,T])$. Write
\[ \Pi_{\theta} \big(\mathcal{L} J^{\theta}(\cdot, X_{\cdot}) + r_{\cdot} \big) = \sum_{i=1}^{L'}\alpha^{(i)}(\theta)\xi_{\cdot}^{\theta,(i)} = :\alpha(\theta)^\top \xi_{\cdot}^{\theta} . \]
Then
\[
\begin{aligned}
	& \langle \Pi_{\theta} \big(\mathcal{L} J^{\theta}(\cdot, X_{\cdot}) + r_{\cdot} \big),  \Pi_{\theta} \big(\mathcal{L} J^{\theta}(\cdot, X_{\cdot}) + r_{\cdot} \big) \rangle_{L^2} \\
	= & \sum_{1\leq i,j \leq L'}\alpha^{(i)}(\theta)\alpha^{(j)}(\theta)\langle \xi_{\cdot}^{\theta, (i)},\xi_{\cdot}^{\theta, (j)} \rangle_{L^2}= \alpha(\theta)^\top  A^\theta  \alpha(\theta) ,
\end{aligned}
\]
where the $ij$-th entry of the $L'\times L'$ matrix $A^\theta$ is $\langle \xi_{\cdot}^{\theta, (i)},\xi_{\cdot}^{\theta, (j)} \rangle_{L^2}$.

On the other hand,
\[ 
\E\left[\int_0^T \left( \Pi_{\theta}\left( \mathcal{L} J^{\theta}(\cdot, X_{\cdot}) + r_{\cdot}\right)\right)\xi_t^{\theta}\dd t \right]=\E\left[\int_0^T \big( \mathcal{L} J^{\theta}(t, X_{t}) + r_{t}\big)\xi_t^{\theta}\dd t \right]
=A^\theta\alpha(\theta).
\]

Therefore,
\[
\begin{aligned}
	& \frac{1}{2}\E\left[\int_0^T \big( \mathcal{L} J^{\theta}(t, X_{t}) + r_{t}\big)\xi_t^{\theta}\dd t \right]^\top \left[\E\int_0^T \xi_t^{\theta} (\xi_t^{\theta})^\top \dd t\right]^{-1} \E\left[\int_0^T \big( \mathcal{L} J^{\theta}(t, X_{t}) + r_{t}\big)\xi_t^{\theta}\dd t \right] \\
	=& \frac{1}{2}\alpha(\theta)^\top A^\theta(A^\theta)^{-1}A^\theta\alpha(\theta)\\
	= & \frac{1}{2}\E\int_0^T \bigg|  \Pi_{\theta} \big(\mathcal{L} J^{\theta}(\cdot, X_{\cdot}) + r_{\cdot} \big) \bigg|^2    \dd t = \frac{1}{2}|| \Pi_{\theta} \big(\mathcal{L} J^{\theta}(\cdot, X_{\cdot}) + r_{\cdot} \big)  ||_{L^2}^2= \operatorname{MSPBE}(\theta).
\end{aligned}
\]

\end{proof}

\end{appendix}

\vskip 0.2in
\bibliography{reference}

\begin{thebibliography}{}

\bibitem[Baird, 1993]{baird1993advantage}
Baird, L.~C. (1993).
\newblock Advantage updating.
\newblock Technical report, Write Lab Wright-Patterson Air Force Base, OH
  45433-7301, USA.

\bibitem[Baird, 1995]{baird1995residual}
Baird, L.~C. (1995).
\newblock Residual algorithms: Reinforcement learning with function
  approximation.
\newblock In {\em Machine Learning Proceedings 1995}, pages 30--37. Elsevier.

\bibitem[Barnard, 1993]{barnard1993temporal}
Barnard, E. (1993).
\newblock Temporal-difference methods and {M}arkov models.
\newblock {\em IEEE Transactions on Systems, Man, and Cybernetics},
  23(2):357--365.

\bibitem[Barndorff-Nielsen and Shephard, 2002]{barndorff2002estimating}
Barndorff-Nielsen, O.~E. and Shephard, N. (2002).
\newblock Estimating quadratic variation using realized variance.
\newblock {\em Journal of Applied Econometrics}, 17(5):457--477.

\bibitem[Beck et~al., 2021]{beck2021nonlinear}
Beck, C., Hutzenthaler, M., and Jentzen, A. (2021).
\newblock On nonlinear feynman--kac formulas for viscosity solutions of
  semilinear parabolic partial differential equations.
\newblock {\em Stochastics and Dynamics}, 21(08):2150048.

\bibitem[Boyan, 2002]{boyan2002technical}
Boyan, J.~A. (2002).
\newblock Technical update: Least-squares temporal difference learning.
\newblock {\em Machine Learning}, 49(2):233--246.

\bibitem[Bradtke and Barto, 1996]{bradtke1996linear}
Bradtke, S.~J. and Barto, A.~G. (1996).
\newblock Linear least-squares algorithms for temporal difference learning.
\newblock {\em Machine Learning}, 22(1):33--57.

\bibitem[Crandall et~al., 1992]{crandall1992user}
Crandall, M.~G., Ishii, H., and Lions, P.-L. (1992).
\newblock User’s guide to viscosity solutions of second order partial
  differential equations.
\newblock {\em Bulletin of the American mathematical society}, 27(1):1--67.

\bibitem[Dai et~al., 2020]{DDJ2020}
Dai, M., Dong, Y., and Jia, Y. (2020).
\newblock Learning equilibrium mean-variance strategy.
\newblock {\em SSRN preprint SSRN:3770818}.

\bibitem[Doya, 2000]{doya2000reinforcement}
Doya, K. (2000).
\newblock Reinforcement learning in continuous time and space.
\newblock {\em Neural Computation}, 12(1):219--245.

\bibitem[Du et~al., 2017]{du2017stochastic}
Du, S.~S., Chen, J., Li, L., Xiao, L., and Zhou, D. (2017).
\newblock Stochastic variance reduction methods for policy evaluation.
\newblock In {\em International Conference on Machine Learning}, pages
  1049--1058. PMLR.

\bibitem[El~Karoui et~al., 1997]{el1997backward}
El~Karoui, N., Peng, S., and Quenez, M.~C. (1997).
\newblock Backward stochastic differential equations in finance.
\newblock {\em Mathematical Finance}, 7(1):1--71.

\bibitem[Fleming and Soner, 2006]{fleming2006controlled}
Fleming, W.~H. and Soner, H.~M. (2006).
\newblock {\em Controlled Markov processes and viscosity solutions}, volume~25.
\newblock Springer Science \& Business Media.

\bibitem[Fr{\'e}maux et~al., 2013]{fremaux2013reinforcement}
Fr{\'e}maux, N., Sprekeler, H., and Gerstner, W. (2013).
\newblock Reinforcement learning using a continuous time actor-critic framework
  with spiking neurons.
\newblock {\em PLoS Computational Biology}, 9(4):e1003024.

\bibitem[Gao et~al., 2020]{GXZ2020}
Gao, X., Xu, Z.~Q., and Zhou, X.~Y. (2020).
\newblock State-dependent temperature control for {L}angevin diffusions.
\newblock {\em arXiv preprint arXiv:2011.07456, forthcoming in {\it SIAM
  Journal on Control and Optimization}}.

\bibitem[Geramifard et~al., 2006]{geramifard2006incremental}
Geramifard, A., Bowling, M., and Sutton, R.~S. (2006).
\newblock Incremental least-squares temporal difference learning.
\newblock In {\em Proceedings of the National Conference on Artificial
  Intelligence}, volume~21, page 356. Menlo Park, CA; Cambridge, MA; London;
  AAAI Press; MIT Press; 1999.

\bibitem[Guo et~al., 2020]{GuoXZ2020}
Guo, X., Xu, R., and Zariphopoulou, T. (2020).
\newblock Entropy regularization for mean field games with learning.
\newblock {\em arXiv preprint arXiv:2010.00145}.

\bibitem[Han et~al., 2018]{han2018solving}
Han, J., Jentzen, A., and E, W. (2018).
\newblock Solving high-dimensional partial differential equations using deep
  learning.
\newblock {\em Proceedings of the National Academy of Sciences},
  115(34):8505--8510.

\bibitem[Hansen, 1982]{hansen1982large}
Hansen, L.~P. (1982).
\newblock Large sample properties of generalized method of moments estimators.
\newblock {\em Econometrica}, pages 1029--1054.

\bibitem[Hansen et~al., 1996]{hansen1996finite}
Hansen, L.~P., Heaton, J., and Yaron, A. (1996).
\newblock Finite-sample properties of some alternative {GMM} estimators.
\newblock {\em Journal of Business \& Economic Statistics}, 14(3):262--280.

\bibitem[Hochreiter and Schmidhuber, 1997]{hochreiter1997lstm}
Hochreiter, S. and Schmidhuber, J. (1997).
\newblock {LSTM} can solve hard long time lag problems.
\newblock {\em Advances in neural information processing systems}, pages
  473--479.

\bibitem[Hur{\'e} et~al., 2019]{hure2019some}
Hur{\'e}, C., Pham, H., and Warin, X. (2019).
\newblock Some machine learning schemes for high-dimensional nonlinear {PDE}s.
\newblock {\em arXiv preprint arXiv:1902.01599}, 33.

\bibitem[Karatzas and Shreve, 2014]{karatzas2014brownian}
Karatzas, I. and Shreve, S. (2014).
\newblock {\em Brownian Motion and Stochastic Calculus}, volume 113.
\newblock Springer.

\bibitem[Kingma and Ba, 2014]{kingma2014adam}
Kingma, D.~P. and Ba, J. (2014).
\newblock Adam: A method for stochastic optimization.
\newblock {\em arXiv preprint arXiv:1412.6980}.

\bibitem[Kloeden and Platen, 1992]{kloeden1992stochastic}
Kloeden, P.~E. and Platen, E. (1992).
\newblock {\em Numerical Solution of Stochastic Differential Equations}.
\newblock Springer.

\bibitem[Lee and Sutton, 2021]{lee2021policy}
Lee, J. and Sutton, R.~S. (2021).
\newblock {Policy iterations for reinforcement learning problems in continuous
  time and space -- Fundamental theory and methods}.
\newblock {\em Automatica}, 126:109421.

\bibitem[Liu et~al., 2016]{liu2016proximal}
Liu, B., Liu, J., Ghavamzadeh, M., Mahadevan, S., and Petrik, M. (2016).
\newblock Proximal gradient temporal difference learning algorithms.
\newblock In {\em IJCAI}, pages 4195--4199.

\bibitem[Ljung and S{\"o}derstr{\"o}m, 1983]{ljung1983theory}
Ljung, L. and S{\"o}derstr{\"o}m, T. (1983).
\newblock {\em Theory and practice of recursive identification}.
\newblock MIT press.

\bibitem[Maei et~al., 2009]{maei2009convergent}
Maei, H.~R., Szepesvari, C., Bhatnagar, S., Precup, D., Silver, D., and Sutton,
  R.~S. (2009).
\newblock Convergent temporal-difference learning with arbitrary smooth
  function approximation.
\newblock In {\em NIPS}, pages 1204--1212.

\bibitem[Mandelbrot and Van~Ness, 1968]{mandelbrot1968fractional}
Mandelbrot, B.~B. and Van~Ness, J.~W. (1968).
\newblock Fractional brownian motions, fractional noises and applications.
\newblock {\em SIAM review}, 10(4):422--437.

\bibitem[Raissi, 2018]{raissi2018deep}
Raissi, M. (2018).
\newblock Deep hidden physics models: Deep learning of nonlinear partial
  differential equations.
\newblock {\em Journal of Machine Learning Research}, 19(1):932--955.

\bibitem[Robbins and Monro, 1951]{robbins1951stochastic}
Robbins, H. and Monro, S. (1951).
\newblock A stochastic approximation method.
\newblock {\em The Annals of Mathematical Statistics}, pages 400--407.

\bibitem[Sottinen and Viitasaari, 2017]{sottinen2017prediction}
Sottinen, T. and Viitasaari, L. (2017).
\newblock Prediction law of fractional brownian motion.
\newblock {\em Statistics \& Probability Letters}, 129:155--166.

\bibitem[Stroock and Varadhan, 1979]{SV79}
Stroock, D.~W. and Varadhan, S. R.~S. (1979).
\newblock {\em Multidimensional diffusion processes}, volume 233 of {\em
  Grundlehren der Mathematischen Wissenschaften [Fundamental Principles of
  Mathematical Sciences]}.
\newblock Springer-Verlag, Berlin-New York.

\bibitem[Sutton, 1988]{sutton1988learning}
Sutton, R.~S. (1988).
\newblock Learning to predict by the methods of temporal differences.
\newblock {\em Machine learning}, 3(1):9--44.

\bibitem[Sutton and Barto, 2018]{sutton2011reinforcement}
Sutton, R.~S. and Barto, A.~G. (2018).
\newblock {\em Reinforcement Learning: An Introduction}.
\newblock Cambridge, MA: MIT Press.

\bibitem[Sutton et~al., 2009]{sutton2009fast}
Sutton, R.~S., Maei, H.~R., Precup, D., Bhatnagar, S., Silver, D.,
  Szepesv{\'a}ri, C., and Wiewiora, E. (2009).
\newblock Fast gradient-descent methods for temporal-difference learning with
  linear function approximation.
\newblock In {\em Proceedings of the 26th Annual International Conference on
  Machine Learning}, pages 993--1000.

\bibitem[Sutton et~al., 2008]{sutton2008convergent}
Sutton, R.~S., Szepesv{\'a}ri, C., and Maei, H.~R. (2008).
\newblock A convergent o(n) temporal-difference algorithm for off-policy
  learning with linear function approximation.
\newblock In {\em NIPS}.

\bibitem[Vamvoudakis and Lewis, 2010]{vamvoudakis2010online}
Vamvoudakis, K.~G. and Lewis, F.~L. (2010).
\newblock Online actor--critic algorithm to solve the continuous-time infinite
  horizon optimal control problem.
\newblock {\em Automatica}, 46(5):878--888.

\bibitem[Wang et~al., 2020]{wang2020reinforcement}
Wang, H., Zariphopoulou, T., and Zhou, X.~Y. (2020).
\newblock Reinforcement learning in continuous time and space: A stochastic
  control approach.
\newblock {\em Journal of Machine Learning Research}, 21(198):1--34.

\bibitem[Wang and Zhou, 2020]{wang2020continuous}
Wang, H. and Zhou, X.~Y. (2020).
\newblock Continuous-time mean--variance portfolio selection: A reinforcement
  learning framework.
\newblock {\em Mathematical Finance}, 30(4):1273--1308.

\bibitem[Xu et~al., 2002]{xu2002efficient}
Xu, X., He, H.-g., and Hu, D. (2002).
\newblock Efficient reinforcement learning using recursive least-squares
  methods.
\newblock {\em Journal of Artificial Intelligence Research}, 16:259--292.

\bibitem[Yong and Zhou, 1999]{YZbook}
Yong, J. and Zhou, X.~Y. (1999).
\newblock {\em Stochastic Controls: Hamiltonian Systems and HJB Equations}.
\newblock New York, NY: Spinger.

\end{thebibliography}

\end{document}